\documentclass[Crown,sageh,times]{sagej}
\usepackage{graphicx}
\usepackage{amice_shortcuts}
\usepackage{lipsum}
\usepackage{tikz}
\usetikzlibrary{patterns}
\usepackage{bibunits}
\usepackage{svg}
\usepackage{float}
\usepackage{subcaption}
\usepackage[colorlinks,bookmarksopen,bookmarksnumbered,citecolor=blue,urlcolor=blue]{hyperref}
\usepackage{leftidx}
\usepackage{ragged2e}
\usepackage{float}
\usepackage[linesnumbered, ruled, lined, boxed, commentsnumbered]{algorithm2e}
\usepackage{footmisc}
\graphicspath{{figures/},{figures/tikz_sketches/},{figures/tikz_sketches/cone_sketch/}}

\tikzstyle{EDB}=[draw=white, opacity = 0.0,line width=2pt, preaction={clip, postaction={pattern=north west lines, pattern color=black}}]

\makeatletter
\newcommand{\printfnsymbol}[1]{%
  \textsuperscript{\@fnsymbol{#1}}%
}

\makeatother

\title{Certified Polyhedral Decompositions of Collision-Free Configuration Space}

\author{
Hongkai Dai\affilnum{*}\affilnum{2},
Alexandre Amice\affilnum{*}\affilnum{1},
Peter Werner\affilnum{1},
Annan Zhang\affilnum{1},
Russ Tedrake\affilnum{1,2}}

\affiliation{
\affilnum{1}Massachusetts Institute of Technology (MIT),
\affilnum{2}Toyota Research Institute,
\affilnum{*}equal contribution}

\date{January 2023}

\runninghead{Dai, Amice, et al.}

\begin{document}

\begin{abstract}
Understanding the geometry of collision-free configuration space (C-free) in the presence of task-space obstacles is an essential ingredient for collision-free motion planning. While it is possible to check for collisions at a point using standard algorithms, to date no practical method exists for computing C-free \emph{regions} with rigorous certificates due to the complexity of mapping task-space obstacles through the kinematics. In this work, we present the first to our knowledge rigorous method for approximately decomposing a rational parametrization of C-free into certified polyhedral regions. Our method, called \CIris (C-space Iterative Regional Inflation by Semidefinite programming), generates large, convex polytopes in a rational parameterization of the configuration space which are rigorously certified to be collision-free. Such regions have been shown to be useful for both optimization-based and randomized motion planning. Based on convex optimization, our method works in arbitrary dimensions, only makes assumptions about the convexity of the obstacles in the \emph{task} space, and is fast enough to scale to realistic problems in manipulation. We demonstrate our algorithm's ability to fill a non-trivial amount of collision-free C-space in several 2-DOF examples where the C-space can be visualized, as well as the scalability of our algorithm on a 7-DOF KUKA iiwa, a 6-DOF UR3e and 12-DOF bimanual manipulators. An implementation of our algorithm is open-sourced in {\href{https://drake.mit.edu/}{Drake}}. We furthermore provide examples of our algorithm in interactive {\href{https://deepnote.com/workspace/alexandre-amice-c018b305-0386-4703-9474-01b867e6efea/project/C-IRIS-7e82e4f5-f47a-475a-aad3-c88093ed36c6/notebook/2d_example_bilinear_alternation-14f1ee8c795e499ca7f577b6885c10e9}{Python notebooks}}.
\end{abstract}

\maketitle

\section{Introduction} \label{S: Introduction}
The notion of configuration space (C-space) has played a foundational role in robot motion planning since its proposal in the seminal work \cite[]{lozano1983spatial}. In the presence of obstacles in the Cartesian task space, a fundamental challenge is describing the collision-free C-space (C-free): the full range of configurations for which a robot is not in collision. Prior work has taken two complementary approaches to this problem. 

The first approach attempts to find an explicit description of the C-space obstacles from their task-space description and the inverse kinematics (IK). We refer to this approach as the negative approach, as C-free is described as the complement of the set of C-space obstacles. In its full generality, the problem of describing C-space obstacles is intractable \cite[]{canny1988complexity}, and so limiting assumptions on the robot are often made. For example, \cite[]{kavraki1995computation} develops a method for computing C-space obstacles based on the Fast Fourier Transform under the assumption that the robot can only translate in the workspace. In \cite[]{branicky1990computing}, explicit descriptions of C-space obstacle due to the presence of point, line, and planar task-space obstacles are presented for two and three degree of freedom (DOF) robots. A thorough review of describing C-space obstacles can be found in \cite[Chapter 3]{latombe2012robot}. There it is shown that if all the task-space obstacles are described as semi-algebraic sets (i.e. as the intersection and union of polynomial inequalities) then C-space obstacles are also semi-algebraic. This is an important result from a complexity-theoretic standpoint as it shows that describing the C-space obstacles is at least decidable, though still very hard. 

We refer to the second approach as the positive approach, as it seeks to directly describe C-free as a union of simpler sets. This description is attractive as a variety of optimization-based motion planning methods can efficiently leverage such descriptions, particularly when the simpler sets are convex \cite[]{deits2015efficient, schouwenaars2001mixed, marcucci2021shortest, marcucci2022motion}.

Rapidly-exploring Random Trees (RRT) \cite[]{lavalle1998rapidly}, Probabilistic Roadmaps (PRM) \cite[]{kavraki1996probabilistic}, and their variants can all be considered examples of this approach, describing C-free using piecewise-linear paths. Frequently, these methods provide probabilistic guarantees that the paths contain no collisions via sampling along the paths. To avoid false positive claims of non-collision, rigorous certification procedures such as \cite[]{schwarzer2004exact} can be used. Works such as \cite[]{verghese2022configuration, han2019configuration, wong2014adaptive} all seek to describe non-zero volume subsets of C-free. Similar to RRTs and PRMs, these methods have the advantage of working in arbitrary configuration spaces, make no assumptions on the C-space obstacles, and proceed via sampling. Therefore, they are typically relatively simple to implement and quite fast in low dimensions. Unfortunately, these methods only provide probabilistic guarantees of non-collision.

When the C-space obstacles are assumed to be convex, rigorous descriptions of C-free may be possible, though hardness results exist. For example, in two and three dimensions with polyhedral C-space obstacles, it is known that finding a minimal decomposition is NP-hard \cite[]{lingas1982power} to solve exactly and even APX-hard \cite[]{eidenbenz2003approximation} to approximate\footnote{A problem is said to be APX-hard if no polynomial time algorithm can achieve an approximation ratio of $1+\delta$ for some $\delta > 0$ unless $P = NP$.}. Works such as \cite[]{lien2007approximate} and \cite[]{ghosh2013fast} overcome these hardness results by finding decompositions that are unions of approximately convex sets. 

In arbitrary dimensions and under the assumption of known, convex C-space obstacles, C-free can be decomposed into convex polyhedra by using the \Iris algorithm of \cite[]{deits2015computing}. As it is based on convex programming, \Iris is relatively fast, and is also able to generate \emph{rigorous certificates} of non-collision. Unfortunately, it is often the case that obstacles are naturally described as convex sets in \emph{task space}, which are rarely convex in C-space. 

In this work, we similarly provide a method for describing C-free using convex polyhedra in a bijective, rational parametrization of C-space known as the tangent configuration space (TC-space). Our primary technical contributions are two convex (specifically Sums-of-Squares (SOS)) programs which can certify that a polyhedron in TC-space contains no collision when the obstacles are specified as convex sets in \emph{task space}. Similar to \cite[]{deits2015computing}, we then construct certified, collision-free polytopic regions by alternating between a pair of convex programs. Our method works in arbitrary dimensions and is the first to our knowledge to provide rigorous certificates for non-zero volume sets in this setting. Moreover, we provide a fast, mature implementation technique in the open-source robotics toolbox {\href{https://drake.mit.edu/}{Drake}}\footnote{\url{https://drake.mit.edu/}}.

A conference version of this paper is published in \cite[]{amice2022finding}, which assumes a robotic manipulator composed of revolute joints operating in a scene where all task-space obstacles are decomposed as a union of vertex representation (V-rep) polytopes. This journal version extends these results in many ways.

First, we demonstrate how our approach can be extended to handle other common, non-polytopic geometries such as spheres, capsules, and cylinders. Moreover, we describe how to extend our approach to handle a robot composed of any of the algebraic joints: revolute, prismatic, spherical, planar, and cylindrical. 
Our second technical contribution introduces a second method for certifying non-collision inspired by the dual of the separating hyperplane approach used in the conference paper. This approach takes the form of certifying the emptiness of a set of polynomial equations and inequalities which can also be written as an optimization program. The third technical contribution of this work is to show that feasibility of the optimization programs we use for certification is not only sufficient, but also necessary for a TC-space region to be collision free provided the degree of certain polynomials are chosen sufficiently large. Finally, we provide new examples of our algorithm deployed on various robots including 2-DOF robots to visualize the TC-space, a robot containing a prismatic joint, and a UR3e robot with collision geometries approximate by cylinders.

% This journal version includes the following extensions: 1) While the conference version only handles robots with polytopic collision geometries and revolute joints, now our approach handles non-polytopic collision geometries such as spheres, capsules and cylinders. It also extends to non-revolute joints including prismatic, planar spherical joints, etc. 2) We propose an alternative formulation to certify non-collision. In addition to the separating plane approach used in the conference version, we show that non-collision condition can also be formulated as the infeasibility certificate for a set of polynomial equations and inequalities. 3). We prove that our proposed formulations are necessary and sufficient conditions for the C-space region being collision free. 4) We demonstrate our approach with more results, including a 2-DOF robot to visualize the C-space, a robot with a prismatic joint, and UR3erobots with cylindrical collision geometries.

We begin in Section \ref{S: Problem Statement} by formally introducing our problem and our assumptions. We proceed in Section \ref{S: Background} by introducing necessary mathematical background for describing our technical approach. In Section \ref{S: Certification}, we present our most technical results: two convex programs which can certify whether a region of TC-space is collision-free. We also state the conditions under which feasibility of these programs are guaranteed when a proposed region is collision-free. We describe how to leverage the certification programs to generate convex decompositions of TC-free in Section \ref{S: Bilinear Alternation}. We conclude in Section \ref{S: Results} with examples of our algorithm deployed on various robots. We will first illustrate the algorithm on two simple 2-DOF systems where both the task and configuration spaces can be visualized and the entire configuration space can be quickly covered. We next demonstrate the ability of our algorithm to certify a wide range of postures for two realistic, 7-DOF manipulators interacting with a shelf. We conclude by showing our algorithm's ability to scale by exploring two 12-DOF, bimanual manipulators.

\noindent\textbf{Notation:} 
Throughout the paper, we will use calligraphic letters ($\calS$) to denote sets, Roman capitals ($X$) to denote matrices, and Roman lower case ($x$) to denote vectors. We use $[N] = \{1, \dots, N\}$, denote the set of all multivariate polynomials in the vector of variables $x$ as $\setR[x]$, and denote the cone of Sums-of-Squares (SOS) polynomials as $\bSigma$. Additionally, we will adopt the monogram notation of \cite[]{tedrakeManip} for rigid transforms.

\section{Problem Statement} \label{S: Problem Statement}
We consider a known, task-space environment where our robot and all obstacles have been decomposed as a union of compact, convex bodies\footnote{For technical reasons, we formally assume that the bodies are compact, convex sets expressible as a Archimedean, basic semi-algebraic sets. See Appendix \ref{A: Archimedean} for the definition of Archimedean}
for example cylinders, capsule, spheres, or vertex representation (V-rep) polytopes. Such collision geometries of our task space are readily available through standard tools such as V-HACD \cite[]{mamou2009simple} and are often a required step for simulating any given environment.

Our robot is a mechanism composed of $N + 1$ links connect via either revolute or prismatic joints \cite[]{wampler2011numerical}:
\begin{itemize}
    \item Revolute (R): a 1-DOF joint permitting revolution about an axis of symmetry. An example is a door handle.
    \item Prismatic (P): a 1-DOF joint permitting translation along an axis. An example is a linear rail.
\end{itemize}
We will assume that all revolute joints are constrained from undergoing complete rotations and all prismatic joints have bounded translation. Formally, if $\theta$ is the configuration-space variable associated to revolute joint, then:
\begin{align} \label{E: joint limit angles}
    -\pi < \theta_{l} \leq \theta \leq \theta_{u} < \pi,
\end{align}
and if $z$ is the configuration-space variable associated to a displacement then:
\begin{align} \label{E: joint limit prismatic}
    z_{l} \leq z \leq z_{u}.
\end{align}
where the bounds $\theta_{l}$, $\theta_{u}$, $z_{l}$, and $z_{u}$ are fixed constants.

Our objective is to find large, convex regions of TC-free regardless of the dimension of the configuration space. This objective is beyond the scope of current decomposition for non-convex spaces/objects such as V-HACD due to the dimensionality of the problem for interesting robots and the complexity of the non-linear kinematics.

\begin{remark}
    Our approach can handle a robot composed of any of the five algebraic joints: revolute, prismatic, planar, cylindrical, planar, and spherical \cite[]{wampler2011numerical}. We restrict ourselves to R and P joints as the other joints can be seen as a composition of these two (see appendix \ref{A: alg kin} for details).
\end{remark}

\section{Background} \label{S: Background}
% \subsection{Outline}
% \begin{itemize}
%     \item Convex combination formulation and Hyperplane condition of non-collision
%     \item Psatz
%     \item Tangent Half Angle
% \end{itemize}

This section introduces key notions from convex analysis and algebraic geometry that will be essential for our approach presented in Section \ref{S: Certification}. We begin by recalling some classic theorems pertaining to the separation of convex bodies. We next review the Positivstellensatz, a central theorem from algebraic geometry that forms the basis for many applications of the Sums-of-Squares method that we will leverage. We conclude by recalling a parameterization of a robot's forward kinematics using rational functions.

\subsection{Separating Convex Bodies} \label{S: separating convex bodies}
In this section, we review two dual ways to check whether two compact, convex sets  $\calA$ and $\calB$ intersect by using convex optimization. Our certification programs in Section \ref{S: Certification} will rely on generalizations of the programs introduced in this section. 

A well-known result from convex optimization theory is the Separating Hyperplane Theorem \cite[Section 2.5]{boyd2004convex} which states that $\calA$ and $\calB$ do not intersect, if and only if there exists a hyperplane $\calH(a,b) = \{x \mid a^{T}x + b = 0, (a, b)\neq (0, 0)\}$ which strictly separates the two bodies. The hyperplane $\calH(a,b)$ serves as a \emph{certificate} of non-intersection. Such a hyperplane is visualized in Figure \ref{F: sep hyperplane} and is described by the solution to program \eqref{E: sep hyperplane generic}. Many previous works \cite[]{brossette2017collision, lin2022reduce} have applied the Separating Hyperplane Theorem to find 
 a \textit{single} collision-free posture; in this paper we apply the theorem to find a convex set of collision-free postures. %Letting $\epsilon$ denote the margin of separation of the convex sets from the hyperplane, a separating hyperplane can be found by solving \eqref{E: sep hyperplane generic}.

% Though any separating hyperplane will suffice to prove non-intersection, we are typically interested in the hyperplane which provides the maximum margin of separation. Letting $\epsilon$ denote the margin, the maximum-margin separating hyperplane can be found by solving \eqref{E: sep hyperplane generic}.

Conversely, if $\calA$ and $\calB$ do intersect, then it is possible to certify this by finding a point in $\calA \cap \calB$. Such a point can be found by solving the convex optimization program \eqref{E: intersection via same point generic}. A certificate of the \emph{infeasibility} of \eqref{E: intersection via same point generic} proves that $\calA$ and $\calB$ do not intersect. Finding a certificate of infeasibility can be obtained by considering the dual of \eqref{E: intersection via same point generic} and is a standard notion in convex optimization \cite[Section 5.8]{boyd2004convex}.

 % Another way to prove that $\calA$ and $\calB$ do not intersect is to leverage the separating hyperplane theorem \cite[Section 2.5]{boyd2004convex} which states that $\calA$ and $\calB$ do not intersect, if and only if there exists a hyperplane $\calH(a,b) = \{x \mid a^{T}x + b = 0\}$ which strictly separates the two bodies. Though any separating hyperplane will suffice to prove non-intersection, we are typically interested in the hyperplane which provides the maximum margin of separation. Letting $\epsilon$ denote the margin, the maximum-margin separating hyperplane can be found by solving \eqref{E: sep hyperplane generic}.

A solution to program \eqref{E: sep hyperplane generic} has the advantage of being able to quantify the magnitude of separation between the two bodies. Therefore, in Section \ref{S: Bilinear Alternation} we will prefer to base our algorithm on a generalization of \eqref{E: sep hyperplane generic}. However, we will see that certain results will be easier to show by considering the \emph{infeasibility} of program \eqref{E: intersection via same point generic}.

\begin{figure*}
\begin{minipage}{0.45 \textwidth}
\begin{subequations}\label{E: sep hyperplane generic}
    \begin{gather} 
    % \bmax_{a,b, \epsilon} ~\epsilon~ \subjectto \nonumber  \\
    \textbf{Find } a,~ b \nonumber
    \\
    a^{T}x + b > 0
    , ~ \forall~x \in \calA\label{E: sep hyperplane generic A}
    \\
    a^{T}y + b < 0
    ,~ \forall~y \in \calB \label{E: sep hyperplane generic B}
    \\
    \nonumber
    % \\ \varepsilon > 0 \label{E: sep hyperplane generic positive margin}
\end{gather}
\end{subequations}
\end{minipage}
\hfill
\begin{minipage}{0.45 \textwidth}
\begin{subequations} \label{E: intersection via same point generic}
\begin{gather}
    \find x,~ y \nonumber~ \subjectto \\
    x \in \calA, ~ y \in \calB \label{E: in set constraint generic}
    \\
    x = y \label{E: same point constraint generic}
\end{gather}
\end{subequations}
\end{minipage}
\begin{minipage}[t]{0.45 \textwidth}
\centering{\begin{tikzpicture}

\draw (-1.9,1.5) -- (-1.2,0.1) -- (1,1.3) -- (-0.5,2.7)--cycle;

\draw[rotate = 30]  (1.3,-1.3) ellipse (1.5 and 1.0);

\node at (-1,1.6) {$\mathcal{A}$};
\node at (1.8,-0.5) {$\mathcal{B}$};

\node (v1) at (-1.6,-1.2) {};
\node (v2) at (3,2.3) {};
\draw  (v1) edge (v2);
\node[rotate = 38] at (-1.1,-0.4) {$a^{T}x + b = 0$};

\end{tikzpicture}}
    \subcaption{If $\calA \cap \calB = \emptyset$ then there exists a hyperplane $a^{T}x + b = 0$ which separates the two bodies.}\label{F: sep hyperplane}
\end{minipage}
\hfill 
\begin{minipage}[t]{0.45 \textwidth}
\centering{\begin{tikzpicture}
\draw (-1.9,1.5) -- (-1.2,0.1) -- (1,1.3) -- (-0.5,2.7)--cycle;
\draw[rotate = 30]  (1.5,0.6) ellipse (1.5 and 1.0);
\node at (-1,1.6) {$\mathcal{A}$};
\node at (1.6,1.7) {$\mathcal{B}$};
\node at (0.3,1.3) {$x = y$};
\node at (0.3, -1) {};

\draw [fill = black] (0.3,1.1) ellipse (0.1 and 0.1);
\end{tikzpicture}}
    \subcaption{If $\calA \cap \calB \neq \emptyset$ then there exists $x \in \calA$ and $y \in \calB$ such that $x = y$}\label{F: convex interseciton}
\end{minipage}

\caption{Program \eqref{E: sep hyperplane generic} searches for a hyperplane which separates $\calA$ and $\calB$ while program \eqref{E: intersection via same point generic} searches for a point in $\calA \cap \calB$. Both of these are convex optimization programs, and exactly one of these programs is feasible.}
\end{figure*}

We conclude by noting that programs \eqref{E: sep hyperplane generic} and \eqref{E: intersection via same point generic} are \emph{strong alternatives}; exactly one of the two programs is feasible. The key to solving either program is to find a finite parameterization of conditions \eqref{E: sep hyperplane generic A}, \eqref{E: sep hyperplane generic B},  and \eqref{E: in set constraint generic}. In Table \ref{Tab: shape conditions table}, we provide a convenient reference for some common geometries.

\begin{table}[htb]
\centering
    \begin{tabular}{p{1.1in} | p{2.3in} | p{1.75in}}
    \centering{Body} & 
    \Centering{$a^{T}x + b > 0$, 
    $\forall x \in \calA$} & 
    \Centering{$x \in \calA$}
    \\
    \hline
    \begin{tabular}{p{1in}}
    \RaggedRight{V-rep Polytope with vertices} $$\{v_{1}, \dots, v_{m}\}.$$
    \end{tabular}
    & 
    \begin{tabular}{p{2.3in}}
    \begin{gather*}
    a^{T}v_{i} + b \geq 1, ~ \forall~ i \in \{1, \dots, m\}
    \end{gather*}
    \end{tabular}
    &
    \begin{tabular}{p{1.75in}}
    \begin{gather*}
    x = \sum_{i=1}^{m} \mu_{i}v_{i}, ~
    \sum_{i=1} \mu_{i} = 1,
    \\
    \mu_{i} \geq 0
    \end{gather*}
    \end{tabular}
    \\
    \hline
    \begin{tabular}{p{1in}}
    \RaggedRight{Sphere with center $o$ and radius $r$.}
    \end{tabular}
    & 
    \begin{tabular}{p{2.3in}}
    \begin{gather*}
        a^{T} o + b \geq r\norm{a}\\
        a^{T}o + b \ge 1
    \end{gather*}
    \end{tabular}
    &
    \begin{tabular}{p{1.75in}}
    $$\norm{x - o}^{2} \leq r^{2}$$
    \end{tabular}
    \\
    \hline
    \begin{tabular}{p{1in}}
    \RaggedRight{Capsule, the convex hull of two spheres with centers $o_{1}$ and $o_{2}$ and radii $r_{1}$ and $r_{2}$.}
    \end{tabular}
    &
    \begin{tabular}{p{2.3in}}
    \begin{gather*} a^{T} o_{1} + b \geq r_{1} \norm{a} \\ a^{T} o_{2} + b \geq r_{2} \norm{a}\\
    a^To_1+b\ge 1\end{gather*}
    \end{tabular}
    &
    \begin{tabular}{p{1.75in}}
    \begin{gather*}
    o_{\mu} = \mu o_{1} + (1-\mu)o_{2}
    \\
    \norm{x - o_{\mu}} \leq \mu r_{1} + (1-\mu)r_{2}
    \\
    0 \leq \mu \leq 1
    \end{gather*}
    \end{tabular}
    \\
    \hline
    \begin{tabular}{p{1in}}
    \RaggedRight{Cylinder, the convex hull of two circles with centers $o_{1}$ and $o_{2}$, and with radii $r_{1}$ and $r_{2}$.}
    \end{tabular}
     &
     \begin{tabular}{p{2.3in}}
    \begin{gather*}
        % h = \norm{o_{1} - o_{2}} \\
        \frac{a_{z}\norm{o_{1} - o_{2}}}{2} + b \geq r_1 \norm{\begin{bmatrix} a_{x} & a_{y}\end{bmatrix}}
        \\
        \frac{-a_{z}\norm{o_{1} - o_{2}}}{2} + b \geq r_2 \norm{\begin{bmatrix} a_{x} & a_{y}\end{bmatrix}}
        \\
        a^{T}\left(\frac{o_{1}+o_2}{2}\right) + b \geq 1
    \end{gather*}
    \end{tabular}
    &
    \begin{tabular}{p{1.75in}}
    \begin{gather*}
        o_{\mu} = \mu o_{1} + (1-\mu) o_{2}
        \\
        v ^{T} (o_{1} - o_{2}) = 0
        \\
        x = o_{\mu} + v
        \\
        \norm{v} \leq \mu r_{1} + (1-\mu)r_{2}
        \\
        0 \leq \mu \leq 1
    \end{gather*}
    \end{tabular}
    \end{tabular}
    \caption{Parameterizations of conditions \eqref{E: sep hyperplane generic A} and \eqref{E: in set constraint generic} respectively for particular convex bodies. \protect\footnotemark } \label{Tab: shape conditions table}
\end{table}
\footnotetext{Strictly speaking, the formulation \eqref{E: sep hyperplane generic A} given for the sphere, capsule, and cylinder only enforce non-strict separation i.e. $a^{T}x + b \geq 0$. This can be remedied by replacing $r$ with $r+\varepsilon$ for any $\varepsilon > 0$.}

\begin{remark}
    Problem \eqref{E: sep hyperplane generic} is frequently written with non-strict inequalities \eqref{E: sep hyperplane generic A} and \eqref{E: sep hyperplane generic B} to make it compatible with modern solvers. Such a formulation requires excluding the trivial solution $(a,b) = (0,0)$ via extra constraints as well as planes which are not strictly separating. The conditions given in Table \ref{Tab: shape conditions table} accomplish both with the constraint $a^Tx + b \ge 1$ with $x = v_{i}$ for polytopic geometries and $x = o$ for sphere, cylinder, and capsules.
\end{remark}

% Note that the constraints in Table \ref{Tab: shape conditions table} also exclude the trivial solution $a = 0, b=0$ (which would satisfy the condition $a^Tx+b\ge 0 \ \forall x\in\calA, a^Ty + b \le 0 \ \forall y\in\calB$ but doesn't represent a valid separating hyperplane). First if either of the geometry is a polytope, then $a^Tv_i + b \ge 1$ excludes the trivial solution $a = 0, b=0$; moreover, if any of the geometries is not a polytope, for example a sphere, then the constraint $a^To + b \ge 1$ also excludes the trivial solution.

\subsection{Certificates of Positivity and Infeasibility} \label{S: Psatz}
In Section \ref{S: Certification}, we will show how to generalize programs \eqref{E: sep hyperplane generic} and \eqref{E: intersection via same point generic} to be able to certify non-collision for a range of robot configurations. Both generalizations will reduce to well-studied polynomial problems. Specifically, given the set
$$\calS_{g, h} = \{x \mid g_{i}(x) \geq 0, h_{j}(x) = 0, i \in [n], j \in [m]\},$$
where $g_i(x)$ and $h_{j}(x)$ are all given polynomial functions of $x$, then certifying the separating hyperplane conditions \eqref{E: sep hyperplane generic A} and \eqref{E: sep hyperplane generic B} will be akin to a certifying a polynomial implication of the form
\begin{align} \label{E: Gen Cert Prob}
     x \in \calS_{g,h}  \implies p(x) \geq 0
\end{align}
where $p(x)$ is again a polynomial.

Moreover, certifying the infeasibility of \eqref{E: intersection via same point generic} will be akin to certifying that
\begin{align} \label{E: Gen Infeasible}
    \calS_{g, h} = \emptyset.
\end{align}

Both of these polynomial problems are tractable. In particular, a class of results known as Positivstellensatz Theorems (Psatz) can be used to reduce both problems to a convex optimization program \cite[]{parrilo2000structured, blekherman2012semidefinite}. In this section, we review the Psatz results that we will use.

Our assumption \eqref{E: joint limit angles} and \eqref{E: joint limit prismatic} that our robot has joint limits implies that the subsets of TC-free we wish to certify will be Archimedean sets, a property slightly stronger than compactness formally defined in Appendix \ref{A: Archimedean}. This will enable us to use a very strong Psatz Theorem for proving implications of the form \eqref{E: Gen Cert Prob} known as Putinar's Positivstellensatz.

\begin{theorem}[Positivstellensatz {\cite[]{putinar1993positive}}] \label{T: Putinar}
Suppose $\calS_{g, h}$ is Archimedean and suppose that $p(x) > 0$ for all $x \in \calS_{g,h}$. Then there exists polynomials $\phi_{j}(x),~j = 0, \dots, m$ and SOS polynomials $\lambda_{i}(x),~ i=0,\dots,n$ such that:
\begin{align} 
    p(x) = \lambda_{0}(x) + \sum_{i=1}^n\lambda_{i}(x)g_{i}(x) + \sum_{j=1}^m \phi_{j}(x) h_{j}(x) \label{E: Putinar Psatz}.
\end{align}

Moreover, if $p(x)$ is any polynomial that can be expressed as in \eqref{E: Putinar Psatz}, then
\begin{align} \label{E: poly positive implication}
    x \in \calS_{g,h} \implies p(x) \geq 0
\end{align}

\end{theorem}

As an immediate corollary, the previous theorem can be used to prove that $\calS_{g,h}$ is empty.

\begin{theorem}[%Dual Positivstellensatz 
{\cite[]{parrilo2004sum}}] \label{T: Putinar Dual}
Suppose $\calS_{g, h}$ is Archimedean. Then $\calS_{g,h} = \emptyset$ if and only if there exists polynomials $\phi_{j}(x)$ and SOS polynomials $\lambda_{i}(x)$ such that
\begin{align}
-1 = \lambda_{0}(x) + \sum_{i} \lambda_{i}(x)g_{i}(x) + \sum_{j} \phi_{j}(x) h_{j}(x) \label{E: Putinar Dual Psatz}.
\end{align}
\end{theorem}

In both cases, the multiplier polynomials $\lambda$ and $\phi$ serve as \emph{certificates} that the conditions \eqref{E: Gen Cert Prob} or \eqref{E: Gen Infeasible} hold. These certificates can be searched for using a convex optimization technique known as Sums-of-Squares (SOS) programming, a subset of semidefinite programming (SDP) \cite[]{parrilo2000structured}. The SOS technique has been widely used in robotics, for example in stability verification \cite[]{tedrake2010lqr,majumdar2017funnel, shen2020sampling}, reachability analysis \cite[]{jarvis2003some, yin2021backward} and geometric modeling \cite[]{ahmadi2016geometry}. In this paper, we will use SOS programming to generate certificates that subsets of TC-space are contained in TC-free.

\subsection{Rational Forward Kinematics} \label{S: Rat Forward}
Our method in Section \ref{S: Certification} will rely critically on parameterizing the forward kinematics of our robot using polynomials. Many robots contain rotational joints and so their forward kinematics are naturally specified as trigonometric functions. In this section, we review a standard change of variables of our robot kinematics which will enable us to parameterize the forward kinematics as a rational function.

The forward kinematics of a rigid-body robot with $N$ joints can be written by composing rigid transforms \cite[]{craig2005introduction, tedrakeManip}. Written in homogeneous coordinates, and using the monogram notation \cite{tedrakeManip}\footnote{In monogram notation, the pose of a frame $A$ expressed in a frame $F$ is denoted as $\leftidx{^F}X^{A}$.}, the pose of a frame $A$, expressed in the reference frame $F$, as a function of the robot configuration $q$ assumes the form:
\begin{align}\label{E: gen forward kin}
    \leftidx{^F}X^{A} = \begin{bmatrix}
        \leftidx{^F}R^{A}(q) & \leftidx{^F}p^{A}(q) \\
        0_{1 \times 3} & 1 \\
    \end{bmatrix}
    =
    \prod_{i \in \calI_{F, A}} 
    \leftidx{^{P_{i}}}X^{C_i}(q_{i})\ \leftidx{^{C_i}}{X}^{P_{i+1}}
\end{align}
 In equation \eqref{E: gen forward kin}, $\calI_{F, A} =\{i_1,\hdots, i_n\}\subseteq [N]$ is the set of joints lying on the kinematic chain between $F$ and $A$. We attach two frames to each joint, with $P_i$ rigidly fixed to the parent link of the $i$\textsuperscript{th} joint, and $C_i$ rigidly fixed to the child link of the same joint. The two frames $P_i$ and $C_i$ coincide when the joint configuration $q_i=0$. The subset of configuration variables $q_{i}$ defines the degrees of freedom at the $i$\textsuperscript{th} joint, $\leftidx{^{P_i}}X^{C_i}(q_i)$ is the relative transform of the joint after the joint moves by $q_i$. The rigid transform $\leftidx{^{C_i}}X^{P_{i+1}}$ describes the physical properties of the $i$\textsuperscript{th} link such as its length. We assume that the reference frame $F$ is the $P_{i_1}$, the parent frame of the first joint $i_1$; while the frame $A$ is $C_{i_n}$, the child frame of the last joint $i_n$. \footnote{Since joint $i_n$ is the last joint on this chain $\mathcal{I}_{F, A}$, we assume $\leftidx{^{C_{i_n}}}X^{P_{i_{n+1}}} = I$.}  We choose to be explicit about the reference frame $F$ at the risk of being pedantic, as the choice of reference frame $F$ will have important consequences for the scalability of the approach described in Section \ref{S: Bilinear Alternation} (see Appendix \ref{A: Frame Selection} for a detailed discussion).

The matrices $\leftidx{^{P_{i}}}X^{C_{i}}(q_{i})$ assume the following forms \cite[]{wampler2011numerical}
\begin{align} \label{E: gen low order pair}
    \leftidx{^{P_{i}}}X^{C_i}(q_{i}) &=
    \begin{cases}
    \begin{bmatrix}
        \cos(\theta_{i}) & -\sin(\theta_{i}) & 0 & 0 \\
        \sin(\theta_{i}) &  \cos(\theta_{i}) & 0 & 0 \\
        0  & 0  & 1 & 0 \\
        0 & 0 & 0 & 1
    \end{bmatrix}
    &
    \text{if $i$\textsuperscript{th} joint is Revolute}
    \\
    \begin{bmatrix}
        1 & 0 & 0 & 0 \\
        0 & 1 & 0 & 0 \\
        0 & 0  & 1 & z_{i} \\
        0 & 0 & 0 & 1
    \end{bmatrix}
    &
    \text{if $i$\textsuperscript{th} joint is Prismatic}
    \end{cases}
\end{align}
 
Expression \eqref{E: gen forward kin} expresses the position of our robot as an \emph{multilinear trigonometric polynomial function}. Concretely, the $w$\textsuperscript{th} component (where $w\in\{x, y, z\}$) of the position of $A$ relative to $F$ and expressed in $F$ is an expression of the form:
\begin{align} \label{E: gen forward kin pos}
    \leftidx{^F}p^{A}_{w}(q) = \sum_{j} c_{jw} \prod_{i \in \calI_{F,A}} \xi_{ij,w}(q_{i})
\end{align}
with $\xi_{ij,w}(q_{i}) \in \{\cos(\theta_{i}), \sin(\theta_{i}), z_{i}\}$. The scalar constants $c_{jw}$ are determined by the robot kinematic parameters (link length, joint axis, etc). Therefore, our configuration-space variables are
\begin{align*}
    q = \bigcup_{i} \{\theta_{i}, z_{i}\}.
\end{align*}

Multilinear trigonometric functions have many fortunate algebraic properties which we exploit throughout this paper, the first of which will be a change of variables enabling us to write \eqref{E: gen forward kin pos} as a rational function.

Specifically, we will introduce the substitution:
\begin{align} \label{E: rational sub}
    t_{i} \coloneqq \tan\left( \frac{\theta_{i}}{2}\right),
\end{align}
which allows us to write
\begin{align*}
    \cos(\theta_{i}) &= \frac{1-t_{i}^{2}}{1+t_{i}^{2}}, ~~
    \sin(\theta_{i}) = \frac{2t_{i}}{1+t_{i}^{2}}.
\end{align*}
This substitution is known as the stereographic projection \cite[]{spivakCalc} and is bijective if $\theta_{i} \in (-\pi, \pi)$ which we have assumed is the case for our robotic system\footnote{An alternative approach is to write the forward kinematics $p_w(q)$ as a multilinear polynomial of indeterminates $c_i=\cos(\theta_i)$ and $s_i=\sin(\theta_i)$, with the additional constraints $c_i^2+s_i^2=1$. We don't choose this parameterization as it is hard to integrate the volume on the quotient ring $c_i^2+s_i^2=1,\;\forall i$. Also this parameterization requires introducing two variables $c_i, s_i$ for each revolute joint, rather than one variable $t_i$.}. After performing this change of variables, our forward kinematics variables are
\begin{align*}
    s =  \bigcup_{i} \{t_{i}, z_{i}\}.
\end{align*}
We refer to the configuration-space variable $s$ as the \emph{tangent-configuration-space} (TC-space) variable. 

In the TC-space variable, our forward kinematics are a \emph{rational function} with a polynomial numerator and positive, polynomial denominator. This is an expression of the form
\begin{gather}\label{E: rational forward kinematics gen}
\leftidx{^F}p^{A}_{w}(s) =  \sum_{j} c_{jw} \prod_{i \in \calI_{F,A}} \frac{\leftidx{^F}f_{ij,w}^{A}(s_{i})}{\leftidx{^F}g_{ij,w}^{A}(s_{i})}  = 
    \frac{\leftidx{^F}f^{A}_{w}(s)}{\leftidx{^F}g^{A}_{w}(s)},\; ~w\in\{x, y, z\},
\end{gather}
where $$\frac{\leftidx{^F}f_{ij,w}^{A}(s_{i})}{\leftidx{^F}g_{ij,w}^{A}(s_{i})} \in \left\{\frac{1-t_{i}^{2}}{1+t_{i}^{2}}, \frac{2t_{i}}{1+t_{i}^{2}}, \frac{z_{i}}{1}\right\}.$$
We will abbreviate the vector quantity:
\begin{align}
\leftidx{^F}p^{A}(s) 
=
\frac{\leftidx{^F}f^{A}(s)}{\leftidx{^F}g^{A}(s)}
% \begin{bmatrix}
%     \frac{\leftidx{^F}f^{A}_{x}(s)}{\leftidx{^F}g^{A}_{x}(s)} \\
%     \frac{\leftidx{^F}f^{A}_{y}(s)}{\leftidx{^F}g^{A}_{y}(s)} \\
%     \frac{\leftidx{^F}f^{A}_{z}(s)}{\leftidx{^F}g^{A}_{z}(s)}
% \end{bmatrix}
% =
% \frac{1}{\prod_{w \in \{x,y,z\}}
% \begin{bmatrix}
% \leftidx{^F}f^{A}_{x}(s)\leftidx{^F}g^{A}_{y}(s)\leftidx{^F}g^{A}_{z}(s) 
% \\
% \leftidx{^F}f^{A}_{y}(s)\leftidx{^F}g^{A}_{x}(s)\leftidx{^F}g^{A}_{z}(s) 
% \\
% \leftidx{^F}f^{A}_{z}(s)\leftidx{^F}g^{A}_{x}(s)\leftidx{^F}g^{A}_{y}(s) 
% \end{bmatrix}
\end{align}
where $\leftidx{^F}f^{A}(s)$ is a \emph{vector of polynomials} and $\leftidx{^F}g^{A}(s)$ is a single, positive polynomial. Notice that $\leftidx{^F}g^{A} (s) > 0$ since each denominator $\leftidx{^F}g_{ij, w}^{A}(s_{i}) = 1+t_i^2 \text{ or } 1$, which is strictly positive.

We emphasize again that we have assumed:
\begin{gather*}
    -\pi < \theta_{l,i} \leq \theta_{i} \leq \theta_{u,i} < \pi,
    \\
    z_{l, i} \leq z_{i} \leq z_{u, i}.
\end{gather*}
and therefore generically $s_{l} \leq s \leq s_{u}$ component-wise.

Therefore, our substitution between $q$ and $s$ is bijective and so trajectories in TC-space correspond unambiguously to trajectories in C-space. Moreover, this assumption on boundedness of our configuration space allows us to seek collision-free regions $\calP$ that are contained within $\calP_{lim}$, a polytope encoding our joint limit: $\calP\subseteq \calP_{lim} = \{s \mid s_{l} \leq s \leq s_{u}\}$. 

% This ensures that our polytope is Archimedean and therefore we will be able to use the Psatz introduced in Section \ref{S: Psatz}.

\begin{example}
    As an example, we consider the double pendulum \cite[]{underactuated}.
    
    \begin{figure}[htb]
    \centering
        \scalebox{0.90}{\begin{tikzpicture}
\draw[rotate = 30,fill, opacity = 0.3]  (-3.8926,3.9595) rectangle (-3.4926,0.3595);
\draw[fill]  (-5.1915,1.6064) ellipse (0.2 and 0.2);
\draw[rotate = 45, fill, opacity = 0.3]  (-3.6666,1.2474) rectangle (-3.2666,-2.3526);
\draw[fill]  (-3.3795,-1.5269) ellipse (0.2 and 0.2);
\draw[very thick, <->] (-5.2,-1) -- (-5.2,1.6) -- (-3.2,1.6);
\node at (-3.2,1.4) {$x$};
\node at (-5.4,-1) {$y$};
\draw[dashed] (-5.2,1.6) -- (-1.4,-5);
\draw[dashed] (-3.4,-1.5) -- (-0.3,-4.6);
\draw[dashed] (-5.2,1.6) -- (-5.2,-3.5);
\draw[->,thick] (-5.2,0.1) arc (-90:-60:1.5);
\draw[->, thick] (-2.25,-3.4919) arc (-60.0005:-45:2.3);
\node at (-4.7,-0.2) {$\theta_1$};
\node at (-1.8,-3.7) {$\theta_2$};
\node at (-3.7,0.2) {$l_1$};
\node at (-1.5,-2.5) {$l_2$};
\end{tikzpicture}}
        \caption{The forward kinematics of the double pendulum described in \cite[]{underactuated} can be described in the form \eqref{E: gen forward kin}.}
    \end{figure}
    
    The pose of the tip of the second pendulum can be written as:
    \begin{multline*}
    \left[
        \begin{array}{ c | c}
        R(\theta) &
        \begin{array}{c}
            p_{x}(\theta) \\
            p_{y}(\theta) 
            \end{array} 
            \\
            \hline
            0 & 1
        \end{array}
        \right]
        =
        \begin{bmatrix}
            \cos(\theta_{1}) & \sin(\theta_{1}) & 0 \\
            \sin(\theta_{1}) & -\cos(\theta_{1}) & 0\\
            0 & 0 & 1
        \end{bmatrix}
        \begin{bmatrix}
            1 & 0 & 0 \\
            0 & 1 & l_{1} \\
            0 & 0 & 1
        \end{bmatrix}
        *\\
        \begin{bmatrix}
            \cos(\theta_{2}) & \sin(\theta_{2}) & 0 \\
            \sin(\theta_{2}) & -\cos(\theta_{2}) & 0\\
            0 & 0 & 1
        \end{bmatrix}
        \begin{bmatrix}
            1 & 0 & 0 \\
            0 & 1 & l_{2} \\
            0 & 0 & 1
        \end{bmatrix}
    \end{multline*}
    The difference in the sign of the trigonometric part ensures that the $y$-axis is pointing down. Expanding out this product enables us to write the $x$ coordinate of the tip of the system as:
    \begin{align*}
        p_{x}(\theta_{1}, \theta_{2}) &=
        l_{2}(\sin(\theta_{2})\cos(\theta_{1}) - \sin(\theta_{1})\cos(\theta_{2})) + l_{1} \sin(\theta_{1})\\
        p_{y}(\theta_1, \theta_2) &=l_2(\sin(\theta_1)\sin(\theta_2) + \cos(\theta_1)\cos(\theta_2)) - l_1\cos(\theta_1)
    \end{align*}
    Notice that these are multilinear trigonometric polynomials, i.e. no term contains $\cos(\theta_{i})\sin(\theta_{i})$. We can perform the substitution given in \eqref{E: rational sub} to express the position as a rational function:
    \begin{align*}
        p_{x}(t_{1}, t_{2}) &=
        \frac{2l_{2}(t_{2}(1-t_{1})^{2} - t_{1}(1-t_{2})^{2}) + 2l_{1} t_{1}(1+t_{2})^{2}}
        {(1+t_{1}^{2})(1+t_{2}^{2})}\\
        p_y(t_1, t_2) &= \frac{l_2(4t_1t_2+(1-t_1)^2(1-t_2)^2) - l_1(1-t_1)^2(1+t_2)^2}{(1+t_1)^2(1+t_2)^2}
    \end{align*}
    .
\end{example}

\section{Certification of Set-Membership in TC-Free} \label{S: Certification}
In this section, we will consider the problem of certifying the non-collision of two convex bodies $\calA$ and $\calB$ whose poses in task space are a function of the configuration of our robot. While programs \eqref{E: sep hyperplane generic} and \eqref{E: intersection via same point generic} can be used to certify non-collision between $\calA$ and $\calB$ for any fixed configuration, they are insufficient to certify $\calA$ and $\calB$ do not intersect for all configurations in an entire region $\calP$ of the configuration space. Therefore, in Sections \ref{S: Hyperplane Cert} and \ref{S: infeasible non-collision cert}, we will show how to combine the ingredients of Section \ref{S: Background} to generalize programs \eqref{E: sep hyperplane generic} and \eqref{E: intersection via same point generic}. 

The presence of trigonometric functions when the forward kinematics are expressed in the variable $q$ precludes using SOS programming, our tool of choice. Therefore, we will assume that $\calA(s)$ and $\calB(s)$ are convex sets in task space with their poses expressed as \emph{rational functions} in the TC-space variable $s$. This can be achieved using the developments in Section \ref{S: Rat Forward}. Our objective will be to certify that $\calA(s)$ and $\calB(s)$ do not intersect for all $s \in \calP = \{s \mid Cs \leq d\} \subseteq \calP_{lim} = \{s \mid s_{l} \leq s \leq s_{u}\}$. 

Under these assumptions, the generalizations of \eqref{E: sep hyperplane generic} and \eqref{E: intersection via same point generic} will respectively take the form of certifying a polynomial implication and certifying the emptiness of a basic-semialgebraic set. We give a formulation of each as a SOS program. We will conclude in Section \ref{S: cert power} by proving that feasibility of our convex optimization programs is both necessary and sufficient for $\calP$ to be collision-free.

\subsection{Parametrized Hyperplane Certificates of Non-Collision} \label{S: Hyperplane Cert}
In this section, we generalize \eqref{E: sep hyperplane generic} and use SOS to search for a polynomial family of hyperplanes parametrized by the TC-space variable $s$ which will certify the non-collision of $\calA(s)$ and $\calB(s)$ for all $s \in \calP = \{s \mid Cs \leq d\}$.

We begin by remarking that even if $\calA(s)$ and $\calB(s)$ do not collide for all $s \in \calP$, there may not be a single, static hyperplane $\calH = (a, b)$ which certifies this fact. An example of this can be seen in Figure \ref{F: svm}.

\begin{figure}[htb]
    \centering
    \scalebox{1.}{
\input{figures/tikz_sketches/cone_sketch/svm.tikz}
    }
    \caption{The convex collision geometries $\calA(s)$ and $\calB(s)$ are collision-free if and only if there exists a family of hyperplanes $\calH(s)$ separating the two for each configuration $s_{0}$. The planes act as a certificate of non-collision.}
    \label{F: svm}
\end{figure}

We therefore will look for a \emph{polynomial family} of hyperplanes $\calH(s) = \{x \mid a(s)^{T}x + b(s)=0\}$ parametrized by our TC-space variable $s$. Inspection of Table \ref{Tab: shape conditions table} shows that we must generalize
\begin{gather}
   s \in \calP \implies  a^{T}(s) \ \leftidx{^F}p^{v}(s) + b(s) \geq 1,  \label{E: polynomial polytope point in set}
   \end{gather}
   for particular points $v$ specific to each of the geometries, and 
 \begin{gather}
    s \in \calP \implies  a^{T}(s) \ \leftidx{^F}p^{o}(s) + b(s) \geq r\norm{a(s)} \label{E: polynomial round in set},
\end{gather}
for center $o$ if $\calA(s)$ is either a sphere or capsule. The generalization of the conditions for the cylinder are similar to those of the sphere and capsule, and so we defer its complete derivation to Appendix \ref{A: cylinder matrix sos}.

To generalize \eqref{E: polynomial polytope point in set} and \eqref{E: polynomial round in set}, we recall that the position of any point $A \in \calA(s)$ (and similarly $\calB(s)$) can be expressed as a rational function $\leftidx{^F}p^{A}(s) = \frac{\leftidx{^F}f^{A}(s)}{\leftidx{^F}g^{A}(s)}$ where $\leftidx{^F}g^{A}(s) > 0$.

Therefore, we can express \eqref{E: polynomial polytope point in set} as:
\begin{align}
    s \in \calP \implies a^{T}(s) \ \leftidx{^F}f^{v}(s) + (b(s)-1) \ \leftidx{^F}g^{v}(s)  \geq 0 \label{E: polynomial polytope point in set 2}
\end{align}
This is an polynomial implication of the form \eqref{E: poly positive implication}. As $\calP \subseteq \calP_{lim}$ is compact polytope, $\calP$ is Archimedean \cite[Theorem 7.1.3]{marshall2008positive} and so we can use Theorem \ref{T: Putinar} to express condition \eqref{E: polynomial polytope point in set} as:
\begin{align} \label{E: polytope separation psatz condition}
    a^{T}(s) \ \leftidx{^F}f^{v}(s)  + (b(s)-1) \ \leftidx{^F}g^{v}(s) 
    =
    \lambda_{01}(s) +  \sum_{j=1}^{m} \lambda_{j1}(s)(d_{j}-c^{T}_{j}s) 
    % \underbrace{\lambda_{01}(s) +  \sum_{j=1}^{m} \lambda_{j1}(s)(d_{j}-c^{T}_{j}s)}_{\Lambda_{1}(s, C, d)},
\end{align}
where $\lambda_{j1}, j=0,\hdots, m$ are all SOS polynomials.

The condition \eqref{E: polynomial round in set}, can be expressed as a polynomial, matrix inequality using the Schur complement\footnote{
We have that $\gamma \ge r \norm{a}$ if and only if the Schur complement $\begin{bmatrix}\gamma I_3 &ra\\ ra^T & \gamma\end{bmatrix}\succeq 0$.
%, and here $\gamma = (a(s))^T\;\leftidx{^F}p^o(s)+b(s)$ and $\leftidx{^F}p^o = \frac{\leftidx{^F}f^o(s)}{\leftidx{^F}g^o(s)}$.
} 
\cite[]{boyd2004convex}
\begin{align} \label{E: shur complement implication}
    s \in \calP \implies
    \begin{bmatrix}
        \left((a(s))^T \ \leftidx{^F}f^{o}(s)+b(s) \ \leftidx{^F}g^{o}(s)\right)I_{3} & ra(s) \ \leftidx{^F}g^{o}(s) \\ r(a(s))^T \ \leftidx{^F}g^{o}(s) & (a(s))^T \ \leftidx{^F}f^{o}(s)+b(s) \ \leftidx{^F}g^{o}(s)
    \end{bmatrix}
    \succeq 0.
\end{align}
This is known as a \emph{matrix SOS} condition which can be represented as a set of semidefinite constraints \cite[]{nie2011polynomial}. Specifically, by introducing a vector auxillary variable $u$, we can write \eqref{E: shur complement implication} as:
\begin{multline}
s \in \calP, u^{T}u = 1 \implies
\\
u^{T}
    \begin{bmatrix}
        \left(a^{T}(s) \ \leftidx{^F}f^{o}(s)+b(s) \ \leftidx{^F}g^{o}(s)\right)I_{3} & ra(s) \ \leftidx{^F}g^{o}(s) \\ r(a(s))^T\ \leftidx{^F}g^{o}(s) & (a(s))^T \ \leftidx{^F}f^{o}(s)+b(s) \ \leftidx{^F}g^{o}(s)
    \end{bmatrix}
    u
    \geq 0
\end{multline}
which can be expressed as the SOS condition:
% \begin{subequations}
\begin{multline}\label{E: psatz round}
u^{T}
    \begin{bmatrix}
        \left((a(s))^T \ \leftidx{^F}f^{o}(s)+b(s) \ \leftidx{^F}g^{o}(s)\right)I_{3} & ra(s) \ \leftidx{^F}g^{o}(s) \\ r(a(s))^T\ \leftidx{^F}g^{o}(s) & (a(s))^T \ \leftidx{^F}f^{o}(s)+b(s) \ \leftidx{^F}g^{o}(s)
    \end{bmatrix}
    u = 
    \\
    \lambda_{02}(u,s) + \sum_{j=1}^{m} \lambda_{j2}(u,s)(d_{j}-c^{T}_{j}s) + \phi(u,s)(1-u^{T}u)
    % \underbrace{\lambda_{02}(u,s) + \sum_{j=1}^{m} \lambda_{j2}(s)(d_{j}-c^{T}_{j}s) + \phi(u,s)(1-u^{T}u)}_{\Phi_{2}(s,C,d)} 
    % \\
    % \lambda_{i} \in \bSigma
\end{multline}
% \end{subequations}
where $\lambda_{j2}$ are all SOS polynomials, and $\phi \in \setR[u,s]$. We introduce the additional equality $u^Tu = 1$ to make the set $\{(u, s) | s\in\calP, u^Tu=1\}$ an Archimedean set.

We are now ready to describe our convex program certifying that $\calP$ is a region of TC-space containing no collision. For each pair of bodies $\calA(s)$ and $\calB(s)$ which can collide in the scene, we search for a polynomial hyperplane via the optimization program:
\begin{subequations}\label{E: cert by hyperplane poly}
\begin{gather} 
    {\forall~ \text{pairs } \calA,\calB} ~\find a_{\calA,\calB}, b_{\calA,\calB}
    ~~\textbf{subject to}
    \\
    \forall~s \in \calP, ~ a^{T}_{\calA, \calB}(s)x + b_{\calA, \calB}(s) > 0, ~\forall x \in \calA(s) \label{E: cert by hyperplane poly A}
    \\
    \forall~s \in \calP, ~  a^{T}_{\calA, \calB}(s)y + b_{\calA, \calB}(s) < 0 ~\forall y \in \calB(s) \label{E: cert by hyperplane poly B}
    \\
    \lambda_{ij}^{\calA, \calB}(u,s),~ \mu_{ij}^{\calA, \calB}(u,s) \in \bSigma, ~ \phi^{\calA, \calB}(u,s),~ \chi^{\calA, \calB}(u,s) \in \setR[u,s] \label{E: cert by hyperplane multiplier constraint}
    % \Lambda_{\calA, \calB}(s,v_{\calA}), \Omega_{\calA, \calB}(s,v_{\calB}) \in \bSigma \label{E: multiplier PSD}
\end{gather}
\end{subequations}
where $(a_{\calA, \calB}(s),b_{\calA, \calB}(s))$ are the parameters of the polynomial hyperplane separating $\calA$ and $\calB$, the polynomials $\lambda_{ij}^{\calA, \calB}(s)$ and $\phi^{\calA, \calB}(s)$ collect all the multiplier polynomials for enforcing \eqref{E: cert by hyperplane poly A}, and $\mu_{ij}^{\calA, \calB}(s)$ and $\chi^{\calA, \calB}(s)$ collect all the multiplier polynomials for enforcing \eqref{E: cert by hyperplane poly B} by using \eqref{E: polytope separation psatz condition} and \eqref{E: psatz round} depending on the geometry of $\calA$ and $\calB$. We stress in the above program that the decision variables are the \emph{coefficients} of the polynomials $a_{\calA, \calB}$, $b_{\calA, \calB}$, and the multiplier polynomials. The symbols $u$ and $s$ are known as \emph{indeterminates} and are not explicitly searched over.

In Table \ref{Tab: poly point in set}, we summarize the conditions for enforcing \eqref{E: cert by hyperplane poly A} and \eqref{E: cert by hyperplane poly B} for common families of sets. We call a feasible solution to \eqref{E: cert by hyperplane poly} a \emph{certificate} for the polytope $\calP$ which we denote:
\begin{align}\label{E: hyperplane certificate}
    \calC_{\calP} = \bigcup_{(\calA, \calB)} \{a_{\calA, \calB}(s),~b_{\calA, \calB}(s),~\lambda_{ij}^{\calA, \calB}(u, s),~\phi^{\calA, \calB}(u, s), ~\mu_{ij}^{\calA, \calB}(u, s),~\chi^{\calA, \calB}(u, s)\}
\end{align}

\begin{table}[htb]
\centering
    \begin{tabular}{p{2.1in} | p{3.0in} }% | p{0.5in}}
    \centering{Body} & 
    \Centering{Psatz Condition for \eqref{E: cert by hyperplane poly A}} 
    \\
    \hline
    \begin{tabular}{p{2in}}
    \vspace{0.25em}
    \RaggedRight{
    V-rep Polytope with $m$ vertices $v_{i}$ at position 
    $
    \leftidx^{F}p^{v_{i}}(s) =
    \frac{\leftidx^{F}f^{v_{i}}(s)}
    {\leftidx^{F}g^{v_{i}}(s)}
    $
    }
    \vspace{0.5em}
    \end{tabular}
    & 
    \Centering{Enforce \eqref{E: polytope separation psatz condition} for each vertex $v_{i}$.}
    \\
    \hline
    \begin{tabular}{p{2in}}
    \vspace{0.25em}
    \RaggedRight{Sphere with center $o$ at position $\leftidx^{F}p^{o}(s) = \frac{\leftidx^{F}f^{o}(s)}{\leftidx^{F}g^{o}(s)}$ and radius $r$}
    \vspace{0.5em}
    \end{tabular}
    & 
    \Centering{Enforce \eqref{E: psatz round} for the center $o$ with radius $r$. Also enforce \eqref{E: polytope separation psatz condition} for the center $o$.}
    % &
    % \Centering{Yes}
    \\
    \hline
    \begin{tabular}{p{2in}}
    \vspace{0.25em}
        \RaggedRight{Capsule, the convex hull of two spheres with centers 
$o_{1}$ and $o_{2}$ at positions $\leftidx^{F}p^{o_{i}}(s) = \frac{\leftidx^{F}f^{o_{i}}(s)}{\leftidx^{F}g^{o_{i}}(s)}$ and radii $r_{1}$, $r_{2}$
        }
        \end{tabular}
        \vspace{0.5em}
        &
        \Centering{
        For $i \in \{1,2\}$ enforce \eqref{E: psatz round} for center $o_{i}$ with radius $r_{i}$. Also enforce \eqref{E: polytope separation psatz condition} for $o_{i}$.
        }
        \\
        \hline
\begin{tabular}{p{2in}}
 \vspace{0.25em}
\RaggedRight{Cylinder, the convex hull of two circles with centers $o_{1}$ and $o_{2}$, at position $\leftidx^{F}p^{o_{i}}(s) = \frac{\leftidx^{F}f^{o_{i}}(s)}{\leftidx^{F}g^{o_{i}}(s)}$, lying in the plane normal to $\leftidx^{F}p^{o_{1}}(s) - \leftidx^{F}p^{o_{2}}(s)$, and with radii $r_{1}$ and $r_{2}$.}
        \end{tabular}
        \vspace{0.5em}
        &
        \Centering{
        See Appendix \ref{A: cylinder matrix sos}.
        }
    \end{tabular}
    \caption{SOS conditions for the constraint \eqref{E: cert by hyperplane poly A} and \eqref{E: cert by hyperplane poly B} depending on the geometry of bodies $\calA$ and $\calB$.}%$\forall~s \in \calP, ~ a^{T}(s)x + b(s) \geq \epsilon, ~\forall x \in \calA(s)$}
    \label{Tab: poly point in set}
\end{table}

% \begin{remark}\label{R: polytope integral}
%     In general, the maximum margin between $\calA(s)$ and $\calB(s)$ will vary over $s \in \calP$. The choice of $\varepsilon_{\calA, \calB}$ as a scalar maximizes the minimum margin of separation between between $\calA(s)$ and $\calB(s)$ over $s \in \calP$. Another sensible choice would be to maximize the average margin. This can be achieved by making $\varepsilon_{\calA, \calB}(s)$ a polynomial function, ensuring that $s \in \calP \implies \varepsilon_{\calA, \calB}(s) \geq 0$ via another Psatz condition, and changing the objective to $\int_{s \in \calP} \varepsilon_{\calA, \calB}$. However, this integral is at least \#P-hard\footnote{\#P-hard problems are at least as hard as NP-complete problems \cite[]{provan1983complexity}} to compute \cite[]{dyer1988complexity} and therefore a more tractable surrogate should be found.
% \end{remark}

\subsection{Polynomial Infeasibility Certificates} \label{S: infeasible non-collision cert}
As we remarked in section \ref{S: separating convex bodies}, non-collision of two convex shapes $\calA$ and $\calB$ can be checked by certifying the \emph{infeasibility} of \eqref{E: intersection via same point generic}. The infeasibility of \eqref{E: intersection via same point generic} can be extended to the case when the locations of $\calA(s)$ and $\calB(s)$ are a function of $s$. 
\begin{subequations}\label{E: dual psatz poly cert abstract}
    \begin{gather} 
    \textbf{Certify that } \nexists~ s \in \calP,~ x, y \in \setR^{3} \textbf{ such that} \\
    x \in \calA(s), y \in \calB(s) \label{E: in set constraint generic poly}
    \\
    x = y \label{E: same point constraint generic poly}
\end{gather}
\end{subequations}

An equivalent, and perhaps more instructive, way of expressing \eqref{E: dual psatz poly cert abstract} is to consider the set
\begin{align} 
    \calS_{\calP, \calA, \calB} &= \{x, s \mid s \in \calP,~x \in \calA(s),~x \in \calB(s)\} \label{E: set to prove infeasible}
    \\
    &=
    \left\{x,~ s,~ u_{\calA},~ u_{\calB~}~ \middle |~
    \begin{gathered}
    Cs \leq d,
    \\
    \gamma_{i}^{\calA}(s, x, u_{\calA}) \geq 0,
    ~ h_{j}^{\calA}(s,x, u_{\calA}) = 0
    \\
    \gamma_{k}^{\calB}(s, x, u_{\calB}) \geq 0,
    h_{l}^{\calB}(s, x, u_{\calB}) = 0,
    \\
    ~i \in [n_{\calA}], ~j \in [m_{\calA}],
    ~k \in [n_{\calB}], ~l \in [m_{\calB}]
    \end{gathered}
    \right\}
    \label{E: set to prove infeasible explicit}
\end{align}
and to consider the problem
\begin{align} \label{E: opt prove infeasible abstract}
\textbf{Certify that } \calS_{\calP, \calA, \calB} = \emptyset, 
\end{align}

In \eqref{E: set to prove infeasible explicit}, $\gamma_{i}^{\calA}(s, x, u_{\calA})$ and $h_{j}^{\calA}(s, x, u_{\calA})$ are the polynomials encoding the condition that $x \in \calA(s)$ and $u_{\calA}$ collects any extra variables needed to write this condition. Similarly, $u_{\calB}$, $\gamma_{k}^{\calB}(s, x, u_{\calB})$, and $h_{l}^{\calB}(s, x, u_{\calB})$ encode that $x \in \calB(s)$. We provide explicit expressions for $\gamma_{i}^{\calA}, \gamma_{k}^{\calB}$ and $h_j^{\calA}, h_l^{\calB}$ in  Table \ref{Tab: shape conditions table poly} (given in Appendix \ref{A: Semialgebraic Set Memebership}) for a few common geometries.

\begin{example}
If $\calA$ is a polytope with $n_{\calA}$ vertices given by $v_{\calA_{i}}$, and $\calB$ is a sphere with center $o_{\calB}$ and radius $r_{\calB}$, then we can write
\begin{align*}
\calS_{\calP, \calA, \calB}
    &=
    \left\{x,~ s,~ \mu_{\calA_{i}} ~\middle |~
    \begin{gathered}
    Cs \leq d,
    \\
    %\calA eqs
    \left(\prod_{i}\leftidx{^F}g^{v_{\calA_i}}\right) \left(x-
         \sum_{i=1}^{m} \mu_{\calA_i}\left(\frac{\leftidx^{F}f^{v_{\calA_i}}(s)}{\leftidx^{F}g^{v_{\calA_i}}(s)} \right)\right) = 0 ,
         \\
    1 - \sum_{i=1}^{m} \mu_{\calA_i}    = 0,
    \\
    \mu_{\calA_i} \geq 0 ~ \forall ~ i \in [n_{\calA}],
    %end \calA eqs
    \\
    \left(\leftidx^{F}g^{o_{\calB}}(s)\right)^{2}
    \left(r_{\calB}^{2}- \norm{x - \frac{\leftidx^{F}f^{o_{\calB}}(s)}{\leftidx^{F}g^{o_{\calB}}(s)}}^{2}\right)
    \geq 0
    \end{gathered}
    \right\}
\end{align*}
\end{example}

Now, we note that $\calS_{\calP, \calA, \calB}$ is an Archimedean set. This implies that we can use Theorem \ref{T: Putinar Dual} to write \eqref{E: opt prove infeasible abstract} as an optimization problem. Denoting $u = \{u_{\calA}, u_{\calB}\}$, this can be written explicitly as

\begin{subequations}\label{E: dual psatz poly cert}
\begin{gather} 
    \find \lambda_{0},~\lambda_{j}^{\calP},~ \lambda_{j}^{\calA},~ \lambda_{j}^{\calB}, ~ \phi_{k}^{\calA},~ \phi_{k}^{\calB}
    \\
    \begin{multlined}
    -1 =
    \lambda_{0}(s,x,u) +  \sum_{j=1}^{n}\lambda_{j}^{\calP}(s,x,u)(d_{j} - c^{T}_{j}s)
    + \\
    \sum_{i = 1}^{n_{\calA}}
    \lambda_{i}^{\calA}(s,x,u)\gamma_{i}^{\calA}(s,x,u_{\calA}) +
    \sum_{j = 1}^{m_{\calA}}
    \phi_{j}^{\calA}(s,x,u)h_{j}^{\calA}(s,x,u_{\calA}) 
    + \\
    \sum_{l = 1}^{n_{\calB}}
    \lambda_{l}^{\calB}(s,x,u)\gamma_{l}^{\calB}(s,x,u_{\calB}) +
    \sum_{k = 1}^{m_{\calB}}
    \phi_{k}^{\calB}(s,x,u)h_{k}^{\calB}(s,x,u_{\calB}) 
    \end{multlined}
    \label{E: dual psatz poly cert -1 constraint}
    \\
    \lambda_{0},~\lambda_{j}^{\calP},~ \lambda_{i}^{\calA},~ \lambda_{l}^{\calB} \in \bSigma
    \label{E: dual psatz poly cert psd constraint}
    \\
    \phi_{j}^{\calA},~ \phi_{k}^{\calB} \in \setR[s,x,u]
    \label{E: dual psatz poly cert free poly constraint}
\end{gather}
\end{subequations}

 We again emphasize that in program \eqref{E: dual psatz poly cert} the decision variables are the coefficients of $\lambda_{0},~ \lambda_{j}^{\calP},~  \lambda_{i}^{\calA},~ \lambda_{l}^{\calB},~ \phi_{j}^{\calA},~ \text{and }\phi_{k}^{\calB}$, while the symbols $\{x,s,u\}$ are not decision variables but rather polynomial indeterminates. Similar to the program in \eqref{E: cert by hyperplane poly}, a certificate of non-collision can be obtained by solving \eqref{E: dual psatz poly cert} for each pair $(\calA, \calB)$ with the multipliers acting as the certificate.

\subsection{Power of the Certification Programs} \label{S: cert power}
In this section, we consider the power of both certification programs. Specifically, in Sections \ref{S: Hyperplane Cert} and \ref{S: infeasible non-collision cert} we argued that feasibility of \eqref{E: cert by hyperplane poly} and \eqref{E: dual psatz poly cert} are sufficient to prove that $\calP$ is collision-free. In this section, we present two theorems showing that the feasibility of these programs is also \emph{necessary}.

 Such a result is important given the fact that as stated, \eqref{E: cert by hyperplane poly} and \eqref{E: dual psatz poly cert} are infinite dimensional and therefore in practice must be solved by selecting a basis of finite degree for the polynomials. Other subtleties about the power of our formulation are discussed in Appendix \ref{A: necessary and sufficient}. Fortunately, we can prove that there do exist finite degrees such that both programs become feasible when $\calP$ is truly collision-free.

\begin{theorem}\label{T: hyperplane poly cert is always feasible}
Let all multiplier polynomials from \eqref{E: cert by hyperplane poly} have degree at least $\rho$ and let all of the polynomials in the parameterization of the hyperplane have degree at least $\kappa$. Suppose $\calP \subseteq \calP_{lim}$ is a subset of TC-free. 

Then there exists finite $\kappa$ and $\rho$ sufficiently large such that \eqref{E: cert by hyperplane poly} is feasible.
\end{theorem}

 A similar theorem can be stated for the program in \eqref{E: dual psatz poly cert}. 
 
\begin{theorem}\label{T: dual psatz poly cert is always feasible}
Let $\calP \subseteq \calP_{lim}$ be a compact, polytopic subset of TC-free and let all multiplier polynomials from \eqref{E: dual psatz poly cert} have degree at least $\rho$. There exists a finite $\rho$ sufficiently large such that \eqref{E: dual psatz poly cert} is feasible.
\end{theorem}
% \begin{proof}
%     Our assumptions on $\calP$ imply that it is an Archimedean set. Therefore, the feasibility of \eqref{E: dual psatz poly cert} follows immediately from an ``effective" version of Theorem \ref{T: Putinar Dual} given in \cite[]{nie_complexity_2007} which gives an explicit degree bound.  
% \end{proof}

% The proof is fairly technical and relies on appealing to the convex duality between \eqref{E: cert by hyperplane poly} and \eqref{E: dual psatz poly cert}. We therefore defer its proof to Appendix \ref{A: sufficiency of polynomial hyperplanes}.

We delay the proofs and further discussion of these results to Appendix \ref{A: necessary and sufficient}. For now, we simply remark that
Theorems \ref{T: hyperplane poly cert is always feasible} and \ref{T: dual psatz poly cert is always feasible} assert that the certification programs presented in this section are both complete in the sense that any collision-free polytope $\calP$ can be certified with our technique.

\section{Polyhedral Decomposition of TC-free} \label{S: Bilinear Alternation}
% WARNING LARGE PORTIONS OF THIS SECTION ARE PULLED VERBATIM FROM CONFERENCE VERSION
In this section, we describe our algorithm for rapidly generating certified, polyhedral decomposition of TC-free. Our algorithm can be seen as a generalization of the \Iris algorithm of \cite[]{deits2015computing} to non-convex TC-space obstacles and so we name it \CIris (Configuration-Space, Iterative Regional Inflation by Semidefinite programming). The key idea is to iteratively grow certified convex polytopes of increasing size around various important configurations in the TC-space. This is achieved by solving a series of convex optimization programs. The complete algorithm is summarized in Algorithm \ref{Alg: Bilinear Alternation}.

We begin by discussing how we will measure the size of our polytope $\calP=\{s \mid Cs\le d\}$. While it may be attractive to measure the size of a polytope by its volume, it is known that computing the volume of a half-space representation (H-Rep) polytope is \#P-hard\footnote{\#P-hard problems are at least as hard as NP-complete problems \cite[]{provan1983complexity}.}
% (see Remark \ref{R: polytope integral})
\cite[]{dyer1988complexity}
and therefore intractable as an objective. A useful surrogate for the volume of $\calP$ used in \cite[]{deits2015computing} is the volume of the maximum volume inscribed ellipse of $\calP$: the set $\calE_{\calP} = \{Qs + s_{0}\mid \norm{s}_2 \leq 1\}$ where $Q$ is a positive-semidefinite matrix describing the shape of the ellipsoid and $s_{0}$ its center. The problem of finding the maximum volume inscribed ellipsoid in a polytope is a semidefinite program described in \cite[Section 8.4.2]{boyd2004convex}.
\begin{subequations}\label{E: max inscribed ellipse in polytope}
\begin{gather}
\bmax_{Q, s_0}~ \logdet Q ~\subjectto 
    \\
    \norm{Qc_{i}}_{2} \leq d_{i} - c_{i}^Ts_{0} ~~\forall~ i\in [m]
    \label{E: ellipse in polytope}
    \\
    Q \succeq 0 \label{E: ellipse psd}
\end{gather}
\end{subequations}

As we wish our polytopes to cover diverse areas of TC-free, we will grow each polytope $\calP$ around some nominal configuration $s_{s}$ we call the seed point. New seed points are typically chosen using rejection sampling to obtain a point outside of the existing certified regions. The polytope $\calP$ is required to contain $s_{s}$ as it grows.

A maximal volume, certified polytope around $s_{s}$ can be obtained by solving the following optimization program which combines the ellipsoidal program \eqref{E: max inscribed ellipse in polytope} with the certification program \eqref{E: cert by hyperplane poly} from Section \ref{S: Hyperplane Cert}.

% Combining the ellipsoidal program, with the non-collision constraints from \eqref{E: cert by hyperplane poly} and the seed point containment yields the optimization program:
\begin{subequations} \label{E: bilinear program}
\begin{gather}
    \bmax_{
    \substack{
    Q, s_0, C, d,
    \\
    \forall (\calA, \calB)
    \\
    \lambda_{ij}^{\calA, \calB},~ \phi^{\calA, \calB},
    \\
    \mu_{ij}^{\calA, \calB},~ \chi^{\calA, \calB}
    \\
    a_{\calA, \calB},~ b_{\calA, \calB}}}~ \logdet Q ~\subjectto 
    \\
    \eqref{E: ellipse in polytope}, \eqref{E: ellipse psd} \label{E: max inscribed ellipse in polytope constraints}
    \\
    Cs_{s} \leq d\label{E: cert with ellipse contain sample}
    \\ 
    ~\norm{c_{i}}_{2}  \leq 1 ~\forall~ i \in [m]
    \label{E: polytope scaling}
    \\
    \eqref{E: cert by hyperplane poly A},
    ~\eqref{E: cert by hyperplane poly B},
    ~\eqref{E: cert by hyperplane multiplier constraint}
    \label{E: poly sep condition}
    \end{gather}
    \label{E: cert with ellipse}
\end{subequations}
The condition $\calE_{\calP}\subset\calP$ is given by the constraints \eqref{E: max inscribed ellipse in polytope constraints}. Constraint \eqref{E: cert with ellipse contain sample} enforces that $\calP$ grows around $s_{s}$. The added constraint \eqref{E: polytope scaling} prevents numerically undesirable scaling. Finally, \eqref{E: poly sep condition} enforces that we search for hyperplanes $(a_{\calA, \calB}(s), b_{\calA, \calB}(s))$ which separate each collision pair $\calA(s)$ and $\calB(s)$.

While this program is attractive as a specification, it is not convex due to bilinearity between $Q$ and $c_i$ in \eqref{E: ellipse in polytope} and the bilinearity between the multipliers and the defining equations of $\calP$ implicit in \eqref{E: poly sep condition} (see Section \ref{S: Hyperplane Cert}). This bilinearity precludes simultaneous search of the polytope $\calP$, inscribed ellipsoid $\calE_{\calP}$, and the corresponding certificate $\calC_{\calP}$. Therefore, we will approximate the solution to \eqref{E: bilinear program} by alternating between two convex programs; one of which will generate certificates of non-collision and one which will improve our polytope without violating the previous certificate.

\begin{remark}
 It is possible to replace \eqref{E: poly sep condition} with the equivalent constraints from program \eqref{E: dual psatz poly cert}. We prefer to base our algorithm on \eqref{E: cert by hyperplane poly} as it can be visualized (i.e. planes in the task space) and the polynomials contain fewer indeterminates and hence the optimization problem size is smaller. Also the separating planes approach produces separating certificates with quantifiable margins by measuring the distance from the collision geometries to the plane in task space.
\end{remark}

We begin by demonstrating how a certified polytopic region can be improved. Suppose that a convex polytope $\calP = \{s|Cs \leq d\}$ has been certified with certificate $\calC_{\calP}$ and the maximum inscribed ellipse $\calE_{\calP}$ has been computed using \eqref{E: max inscribed ellipse in polytope}. A new, larger polytope $\calP'$ can be found by solving the convex optimization program \eqref{E: polytope growth program} which pushes the faces of $\calP'$ as far away from the surface of $\calE_{\calP}$ without violating the certificate $\calC_{\calP}$. This procedure is visualized in Figure \ref{F: pushback}.

% In this section, we modify \eqref{E: bilinear program} to search for a new polytope $\calP'$ which is \emph{larger} than $\calP$, but which does not violate our certificate $\calC_{\calP}$. This is achieved by pushing the faces of our polytope $\calP$ away from $\calE_{\calP}$ as seen in Figure \ref{F: pushback}.

\begin{figure}
    \centering
    \scalebox{1.1}{\begin{tikzpicture}[scale=0.45, every node/.style={scale=0.7}]

%\draw[rotate = 30, opacity = 0.4]  (-1,1.5) ellipse (4 and 2.5);
\draw[opacity=0.2] (-2.414,-1.7495) -- (-5.1,-0.4);
\draw[opacity=0.2] (-3.714,-1.9495) -- (1.2796,0.6253);
\draw[opacity=0.2] (-3.9825,1.4686) -- (0.4,3.3);
\draw[opacity=0.2] (-3.1516,2.0658) -- (-5.1,-1.2);
\draw[opacity=0.2] (-0.7355,3.6263) -- (2.2,1.5);
\draw[opacity=0.2] (1.7,2.7) -- (-0.1456,-0.9485);
\draw[thick, opacity = 0.2] (-4.7325,-0.5675) -- (-2.8677,-1.4952) -- (0.4168,0.1389) -- (1.413,2.0447) -- (0,3.1) -- (-3.2915,1.738)--cycle;
\draw[rotate = 30, opacity = 0.2 ]  (-1,1.5) ellipse (3 and 1.5);
\draw (-5.2914,3.146) -- (1.6,4.5);
\draw (-4.7,3.5) -- (-5.7,-2.2);
\draw (0.6,4.6) -- (3.1,2.8);
\draw (3.1,3.3) -- (1.4981,-0.9552);
\draw (2.1,0) -- (-2.6,-3);
\draw (-6.2,-1.6) -- (-1.792,-3.019);
\draw[<->] (-1.1324,3.9542) -- (-0.8531,2.7303);
\draw[<->] (1.7,3.8) -- (0.7239,2.5926);

\draw[<->] (-5.2982,0.086) -- (-4.3254,-0.071);
\draw[<->] (0.9986,1.2045) -- (2.0972,0.6987);
\draw[<->] (-3.728,-1.0843) -- (-4.1632,-2.2531);
%\node at (-0.4116,3.0957) {$\delta_i$};
\begin{scope}[shift = {(-10, 0)}]
\draw[EDB] (-5.3,4.8) -- (-4.9893,4.8767) -- (-4.0159,5.042) -- (-2.7487,5.2073) -- (-1.6651,5.244) -- (-0.9305,5.2257) -- (-0.0122,4.9686) -- (1.1815,4.5462) -- (2.1733,4.4543) -- (2.6,4.5) -- (2.6,6.4) -- (-5.3,6.4)--cycle;
\draw[very thick, opacity = 0.8]  plot[smooth, tension=.7] coordinates {(-5.3,4.8) (-1.4447,5.2808) (0.9243,4.6196) (1.9529,4.436) (2.5773,4.5462)};

\draw[opacity=0.4] (-2.414,-1.7495) -- (-5.1,-0.4);
\draw[opacity=0.4] (-3.714,-1.9495) -- (1.2796,0.6253);
\draw[opacity=0.4] (-3.9825,1.4686) -- (0.4,3.3);
\draw[opacity=0.4] (-3.1516,2.0658) -- (-5.1,-1.2);
\draw[opacity=0.4] (-0.7355,3.6263) -- (2.2,1.5);
\draw[opacity=0.4] (1.7,2.7) -- (-0.1456,-0.9485);
\draw[thick] (-4.7325,-0.5675) -- (-2.8677,-1.4952) -- (0.4168,0.1389) -- (1.413,2.0447) -- (0,3.1) -- (-3.2915,1.738)--cycle;
\draw[rotate = 30 ]  (-1,1.5) ellipse (3 and 1.5);
\end{scope}

\begin{scope}[shift = {(10, 0)}]
\draw[EDB] (-5.3,4.8) -- (-4.9893,4.8767) -- (-4.0159,5.042) -- (-2.7487,5.2073) -- (-1.6651,5.244) -- (-0.9305,5.2257) -- (-0.0122,4.9686) -- (1.1815,4.5462) -- (2.1733,4.4543) -- (2.6,4.5) -- (2.6,6.4) -- (-5.3,6.4)--cycle;
\draw[very thick, opacity = 0.8]  plot[smooth, tension=.7] coordinates {(-5.3,4.8) (-1.4447,5.2808) (0.9243,4.6196) (1.9529,4.436) (2.5773,4.5462)};

\draw[opacity=0.2] (-5.2914,3.146) -- (1.9322,4.5155);
\draw[opacity=0.2] (-4.7,3.5) -- (-5.7,-2.2);
\draw[opacity=0.2] (0.8,4.5) -- (3.1,2.8);
\draw[opacity=0.2] (3.1,3.3) -- (1.4981,-0.9552);
\draw[opacity=0.2] (2.1,0) -- (-2.6,-3);
\draw[opacity=0.2] (-6.2,-1.6) -- (-1.792,-3.019);
\draw[rotate = 33 ]  (-0.7,1.5) ellipse (4.4 and 2.6);
\end{scope}

\draw[<->] (-0.8123,-0.4644) -- (-0.1977,-1.4495);
\node at (-0.8335,2.3025) {$\delta_1$};

\node at (0.2036,2.2159) {$\delta_6$};
\node at (0.355,1.3321) {$\delta_5$};
\node at (-1.1978,0.1527) {$\delta_4$};
\node at (-3.2325,-0.4453) {$\delta_3$};
\node at (-3.698,0.0924) {$\delta_2$};
\draw[EDB] (-5.3,4.8) -- (-4.9893,4.8767) -- (-4.0159,5.042) -- (-2.7487,5.2073) -- (-1.6651,5.244) -- (-0.9305,5.2257) -- (-0.0122,4.9686) -- (1.1815,4.5462) -- (2.1733,4.4543) -- (2.6,4.5) -- (2.6,6.4) -- (-5.3,6.4)--cycle;
\draw[very thick, opacity = 0.8]  plot[smooth, tension=.7] coordinates {(-5.3,4.8) (-1.4447,5.2808) (0.9243,4.6196) (1.9529,4.436) (2.5773,4.5462)};

\draw[thick] (5.2332,3.2523) -- (11.0376,4.3368) -- (12.931,2.9145) -- (11.7665,-0.2144) -- (7.6332,-2.8278) -- (4.371,-1.7965)--cycle;

\end{tikzpicture}}
    \caption{In \eqref{E: polytope growth program} we search for the maximum amount the polytopes faces can be pushed away from the current inscribed ellipse without violating the certificate found in the previous step.
    % In \eqref{E: cert with polytope} we search for the maximal pushback of the faces of the polytope away from the inscribed ellipse without violating the certificate found in the previous step.
    }
    \label{F: pushback}
\end{figure}

This can be achieved with the following optimization program:
\begin{subequations} \label{E: polytope growth program}
\begin{gather}
    \max_{
    \substack{
    C, d, \delta, 
    \\
    \forall (\calA, \calB)
    \\
    \lambda_{01}^{\calA, \calB},~ \lambda_{02}^{\calA, \calB},~  \phi^{\calA, \calB}
    \\
    \mu_{01}^{\calA, \calB},~ \mu_{02}^{\calA, \calB} ,~ \chi^{\calA, \calB}
    \\
    a_{\calA, \calB},~ b_{\calA, \calB}}
    }~ \prod_{i=1}^{m} (\delta_{i} + \varepsilon_{0}) ~\subjectto 
    \\
    \norm{Qc_{i}}_{2} \leq d_{i} - \delta_{i} - c_{i}^Ts_{0}, ~ \delta_{i} \geq 0 ~\forall~ i\in [m]
    \\
    \eqref{E: cert with ellipse contain sample}, \eqref{E: polytope scaling},
    \eqref{E: poly sep condition}~ \forall \text{pairs } (\calA(s), \calB(s))
        \label{E: growth sep constraints}
\end{gather} 
\end{subequations}
where $\varepsilon_{0} > 0$ is some positive constant ensuring that the objective is never $0$. We recall that \eqref{E: growth sep constraints} is either a constraint of 
the form \eqref{E: polytope separation psatz condition} or \eqref{E: psatz round}. We  emphasize that in \eqref{E: polytope growth program}, $\lambda_{i1}, \lambda_{i2}, \mu_{i1}, \mu_{i2}, i\ge 1$ are all fixed and it is the variables $c_{j}$ and $d_{j}$ which are searched over.

% We note that the ability of \eqref{E: polytope growth program} to improve the polytope $\calP$ is highly dependent on the quality of the certificate $\calC_{\calP}$. In general, there is no obvious metric which integrates all the information of $\calC_{\calP}$ to measure the quality of this certificate for enabling more growth in \eqref{E: polytope growth program}. The most sensible proxy for the quality of the certificate between two bodies $\calC_{\calP}(\calA, \calB)$ is therefore the margin of separation between $\calA(s)$ and $\calB(s)$ over the polytope $\calP$. This motivates the choice of \eqref{E: cert by hyperplane poly} to generate certificates $\calC_{\calP}$ rather than \eqref{E: dual psatz poly cert}, as \eqref{E: cert by hyperplane poly} finds certificates which ensures a maximum, minimum margin of separation.

\begin{algorithm}
\caption{
Given an initial polytopic region $\calP_{0}$ and seed point $s_{s} \in \calP_{0}$ for which \eqref{E: cert with ellipse} is feasible, return a new polytopic region $\calP_{i}$ with a maximal inscribed ellipse $\calE_{\calP_{i}}$ with larger volume than $\calE_{\calP_{0}}$ and a collision-free certificate $\calC_{\calP_{i}}$. 
% Alternate programs \eqref{E: cert with ellipse} and \eqref{E: cert with polytope} until the change in the volume of the ellipse $\calE_{\calP}$ is less than some $tolerance$. Each optimization program is always feasible since $\calP_{0}$ is assumed collision-free.
% Notice that since $\calP_{0}$ is assumed collision-free, programs \eqref{E: cert with ellipse} and \eqref{E: cert with polytope} are always feasible. \eqref{E: cert with ellipse} and \eqref{E: cert with polytope} are alternated until the change in the volume of the ellipse $\calE_{\calP}$ is less than some $tolerance$.
}\label{Alg: Bilinear Alternation}
\SetAlgoLined
 \LinesNumbered
  \SetKwRepeat{Do}{do}{while}
 $i \gets 0$
 \\
 \Do{$\left(\textbf{vol}(\calE_{\calP_{i}}) - \textbf{vol}(\calE_{\calP_{i-1}})\right) / \textbf{vol}(\calE_{\calP_{i-1}})  \geq$ tolerance}{
 $\calC_{\calP_{i}} \gets$ Solution of \eqref{E: cert by hyperplane poly} with data $\calP_{i}$
 \\
 $\calE_{\calP_{i}} \gets$ Solution of \eqref{E: max inscribed ellipse in polytope} with data $\calP_{i}$
 \\
 $(\calP_{i+1}, \calC_{\calP_{i+1}}) \gets$ Solution of \eqref{E: polytope growth program} with data $(\calE_{\calP_{i}}, \calC_{\calP_{i}})$
 \\
 $i \gets i+1$
 }
 \Return $(\calP_{i}, \calC_{\calP_{i}})$
\end{algorithm}

Our complete algorithm proceeds in three steps. First, an initial, collision-free polytope $\calP_{0}$ containing a seed point $s_{s}$ is certified using \eqref{E: hyperplane certificate} to obtain $\calC_{\calP_{0}}$. Next, the maximum inscribed ellipsoid $\calE_{\calP_{0}}$ is computed using \eqref{E: max inscribed ellipse in polytope}. Finally, $\calP_{0}$ is improved using \eqref{E: polytope growth program} to obtain a new polytope $\calP_{1}$. This polytope $\calP_{1}$ has the same number of defining inequalities as $\calP_{0}$. We iterate this process until the volume of $\calE_{\calP}$ stops improving. This algorithm is formalized in Algorithm \ref{Alg: Bilinear Alternation}. Every step of this process involves solving an \emph{convex program} for which very fast, commercial solvers exist \cite[]{mosek, andersen2000mosek}.

\begin{remark}
Some practical considerations for improving the runtime of Algorithm \ref{Alg: Bilinear Alternation} are discussed in the appendices. Specifically, in Appendix \ref{A: Practical Aspects} we expand on design choices which substantially impact the size of the optimization programs as well as which part of Algorithm \ref{Alg: Bilinear Alternation} can be parallelized. Additionally, in Appendix \ref{A: Seeding} we discuss a heuristic strategy for proposing a large, initial regions $\calP_{0}$.
\end{remark}

\section{Results} \label{S: Results}
We demonstrate the use of Algorithm \ref{Alg: Bilinear Alternation} on systems of varying complexity. We begin with very simple robots where both the task and configuration space can be visualized and demonstrate that our algorithm can find very large portions of TC-space and achieve near-complete coverage for simple systems in reasonable time.

We then demonstrate the use of Algorithm \ref{Alg: Bilinear Alternation} on various robots commonly found in industry. These include a KUKA iiwa reaching into a shelf, a bimanual KUKA iiwa, and similar setups for the Franka UR3. Our objective is show the scalability of our algorithm in realistic settings as well as demonstrate the diversity of shapes our approach can handle.

A mature implementation of our algorithm is available in the open-source robotics toolbox {\href{https://drake.mit.edu/}{Drake}} \cite[]{drake}. We furthermore provide examples of our algorithm in interactive {\href{https://deepnote.com/workspace/alexandre-amice-c018b305-0386-4703-9474-01b867e6efea/project/C-IRIS-7e82e4f5-f47a-475a-aad3-c88093ed36c6/notebook/2d_example_bilinear_alternation-14f1ee8c795e499ca7f577b6885c10e9}{Python notebooks}}. Animations of various figures in this section can also be found on this project's \href{https://alexandreamice.github.io/project/c-iris}{website}.

The implementation details of all experiments in this section, such as the choice of reference frame for each plane, the degree of the polynomials parametrizing the hyperplanes, and the degree of the multipliers polynomials in each program are expounded on in Appendix \ref{A: Practical Aspects}.

\subsection{Simple Robots} \label{S: Simple Robots}
In this section, we consider two simple robots each containing only two degrees of freedom. This enables us to visualize both the task space, as well as the configuration space. Though containing few degrees of freedom, each environment maintains rich, realistic collision geometries.

\subsubsection{Pendulum on a Rail} \label{S: Pend on Rail}

\begin{figure*}[htb]
 \centering
\begin{minipage}[c]{0.49 \textwidth}
    \centering{\includegraphics[width = 0.98\textwidth]{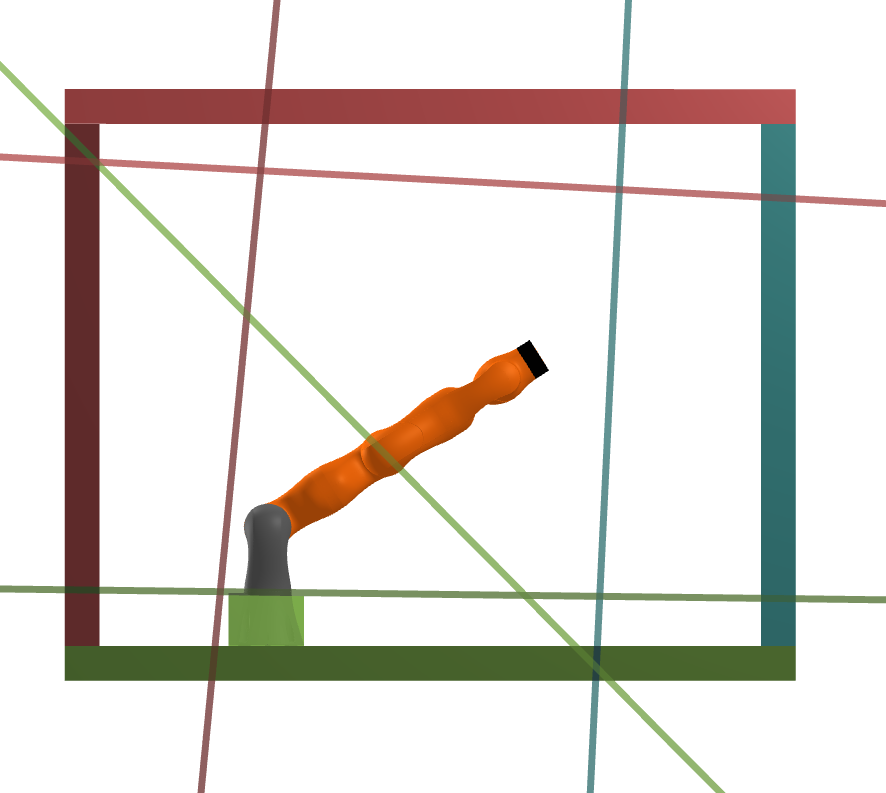}}
    \vspace{11mm}
    \subcaption{The pendulum on a rail robot. Each hyperplane is a function of the TC-variable $s$ and separates the collision body of the same color from the tip of the robot highlighted in black.}
    \label{F: pend on rail task space}
\end{minipage}
\hfill 
\begin{minipage}[c]{0.45 \textwidth}
\centering{
\begin{tikzpicture}
        \node [anchor = south west] (image) at (0,0) {\includegraphics[width=0.98\textwidth] {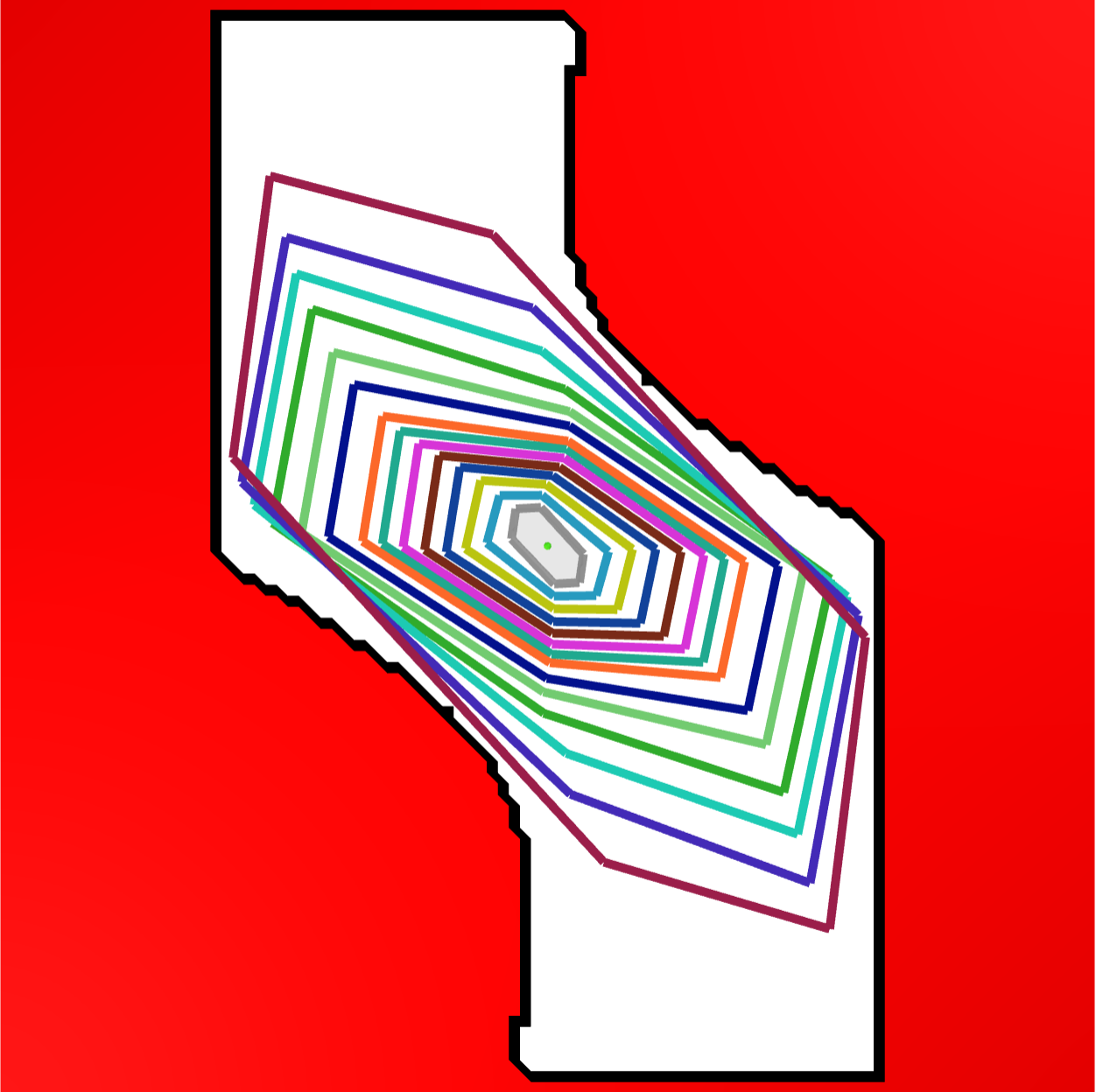}};
        \draw[stealth-, very thick,black] (image.south east) -- (image.south west) node[midway, below]  {Prismatic Joint Position};
        \draw[stealth-, very thick,black] (image.north west) -- (image.south west) node[midway, rotate=90, above]   {Revolute Joint Angle};
        \end{tikzpicture}
}
    \subcaption{The tangent configuration space of the pendulum on a rail robot. The tangent-configuration-space obstacle is in red. A sample of the polytopes obtained running Algorithm \ref{Alg: Bilinear Alternation} around the configuration $(0,0)$ are shown.}
        \label{F: pend on rail cspace}
\end{minipage}
    % \begin{subfigure}[c]{0.49\textwidth}
    % \centering
    %     \includegraphics[width = 0.98\textwidth]{figures/iiwa_on_rail/pend_on_rail_w_planes.png}
    %     \subcaption{Visualization of the robot in task space. Each hyperplane separates the collision body of the same color from the tip of the robot highlighted in black.}
    %     \label{F: pend on rail task space}
    % \end{subfigure}
    % \hfill
    % \begin{subfigure}[c]{0.49\textwidth}
    % \centering
    %     \begin{tikzpicture}
    %     \node [anchor = south west] (image) at (0,0) {\includegraphics[width=0.98\textwidth] {figures/iiwa_on_rail/pend_on_rail_cspace.png}};
    %     \draw[stealth-, very thick,black] (image.south east) -- (image.south west) node[midway, below]  {Prismatic Joint Position};
    %     \draw[stealth-, very thick,black] (image.north west) -- (image.south west) node[midway, rotate=90, above]   {Revolute Joint Angle};
    %     \end{tikzpicture}
    %     \subcaption{The configuration space of the pendulum on rail robot. The configuration space obstacle is in red. A sample of the polytopes obtained running Algorithm \ref{Alg: Bilinear Alternation} around the configuration $(0,0)$ are shown.}
    %     \label{F: pend on rail cspace}
    % \end{subfigure}
    \caption{A 2-DOF robot consisting of a revolute joint at the base of the orange link and a prismatic joint between the base and the box.}
    \label{F: pend on rail}
\end{figure*}

Our first robot shown in Figure \ref{F: pend on rail task space} consists of a single arm, shown in orange, connected to a base via a revolute joint and placed within a box. The base of the robot is connected to the box via a prismatic joint. The collision geometries of the robot and box are approximated using polytopic boxes. A total of $42$ pairs of geometries can collide in this scene (i.e. certifying non-collision requires solving $42$ instances of either \eqref{E: cert by hyperplane poly} or \eqref{E: dual psatz poly cert}). In Figure \ref{F: pend on rail cspace}, we visualize the two dimensional tangent configuration space of our robot with the TC-space obstacle shown in red. We emphasize the highly non-convex shape of TC-free.

We run Algorithm \ref{Alg: Bilinear Alternation} starting with a regular octagon of side length $0.01$ centered at the configuration $(0,0)$, a configuration with the arm fully extended upwards and centered in the box. We obtain a sequence of certified polytopes of increasing size in the TC-space which are plotted in varying colors in Figure \ref{F: pend on rail cspace}. 

The algorithm terminates after 86 iterations of the while loop from Algorithm \ref{Alg: Bilinear Alternation} taking a total of 314 seconds of wall time. During the course of the algorithm, the volume of the maximum inscribed ellipsoid improves by a factor of $83$, from a starting value of $0.021$ to $1.746$. The improvement in the volume of the inscribed ellipsoid, as well as the average time to solve both the certification program \eqref{E: cert by hyperplane poly} and \eqref{E: polytope growth program} are reported in Figure \ref{F: pend on rail volume improvement} and Table \ref{F: pend on rail stats} respectively.

After completion, we select a single random configuration within our final certified region. In Figure \ref{F: pend on rail}, we highlight the tip of the pendulum in black. Additionally, we color each collision body for which the tip can collide in a separate color and plot the separating plane certificate between the tip and the body in the same color.

\begin{figure*}[htb]
\begin{subfigure}[c]{0.65\textwidth}
    \centering
        % \includegraphics[width = 0.98\textwidth]{figures/iiwa_on_rail/pend_on_rail_w_planes.png}
        % \scalebox{0.9}{
        \includegraphics[width = \textwidth]{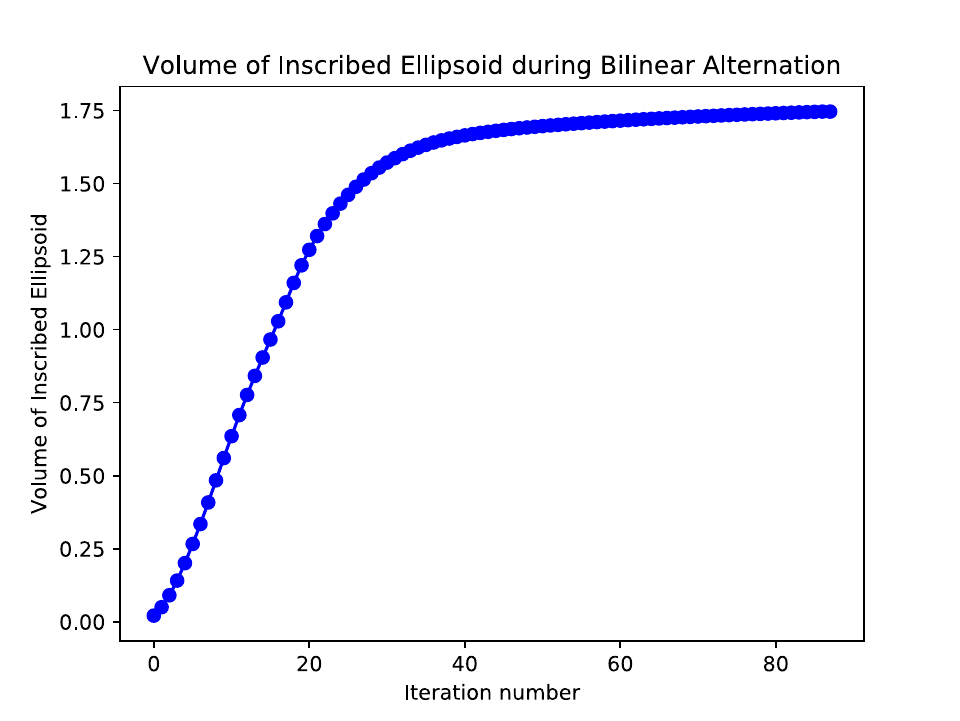}
        \caption{The volume of the maximum inscribed ellipsoids of the TC-free regions shown in Figure \ref{F: pend on rail cspace} is plotted over iterations of Algorithm \ref{Alg: Bilinear Alternation}. This volume grows by a factor of $83$ over the course of 86 iterations Algorithm \ref{Alg: Bilinear Alternation}.}
        \label{F: pend on rail volume improvement}
        % }
    \end{subfigure}
    \hfill
    \begin{subfigure}[c]{0.33\textwidth}
    \vspace{29.5mm}
    \centering
        % \begin{table}
        \begin{tabular}{p{1in} | c}
        % &  \\
        \hline
        Number of collision pairs & 42 \\
        \hline
        Size of the largest PSD variable & 2 \\
        \hline
        Average time to solve \eqref{E: cert by hyperplane poly} & 0.191s \\
        \hline
        Average time to solve \eqref{E: polytope growth program} & 0.423s \\
        % \hline
        % Number of regions to cover & 2 \\
        \hline
        Wall time to grow single region & 314s \\
        \hline
        \end{tabular}
    \vspace{13mm}
    \caption{Statistics dominating the run time of Algorithm \ref{Alg: Bilinear Alternation} for the pendulum on a rail system. The complexity scales with the number of collision geometries as well as the size of the largest PSD matrix variable for enforcing the Psatz conditions in Programs \eqref{E: cert by hyperplane poly} and \eqref{E: polytope growth program}.} 
    \label{F: pend on rail stats}
        % \end{table}
    \end{subfigure}
    % \begin{subfigure}[t]{0.32\textwidth}
    % \centering
    %     \includegraphics[width = 0.98\textwidth]{figures/placeholder.png}
    % \end{subfigure}
    \caption{The progress of Algorithm \ref{Alg: Bilinear Alternation} on the pendulum on a rail system for a single polytopic region is plotted. Statistics dominating the run time of the algorithm are also reported.}
\end{figure*}

\subsubsection{Pinball Flipper} \label{S: Pinball}
\begin{figure*}
\begin{subfigure}[c]{0.49\textwidth}
    \vspace{12mm}
    \centering
        \includegraphics[width = 0.98\textwidth]{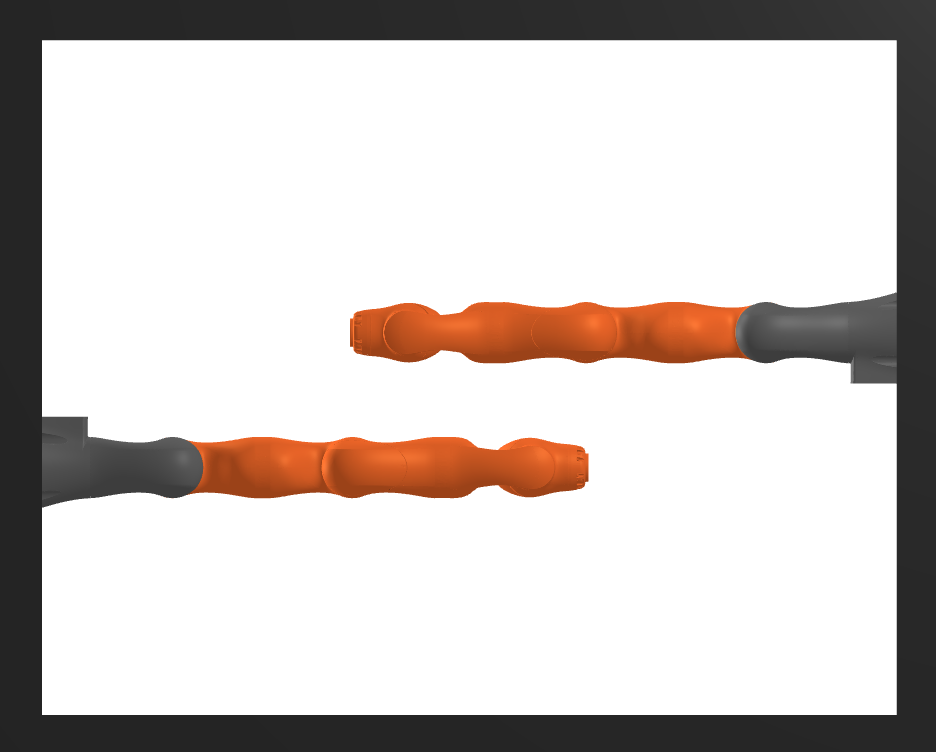}
        \vspace{25.5mm}
        \caption{The pinball flipper system consists of pendulums each with a revolute joint between the orange link and the gray base. All collision geometries in the scene are approximate using boxes.}
        \label{F: pinball task space}
    \end{subfigure}
    \hfill
    \begin{subfigure}[c]{0.49\textwidth}
    \centering
        \begin{tikzpicture}
        \node [anchor = south west] (image) at (0,0) {\includegraphics[width=0.98\textwidth] {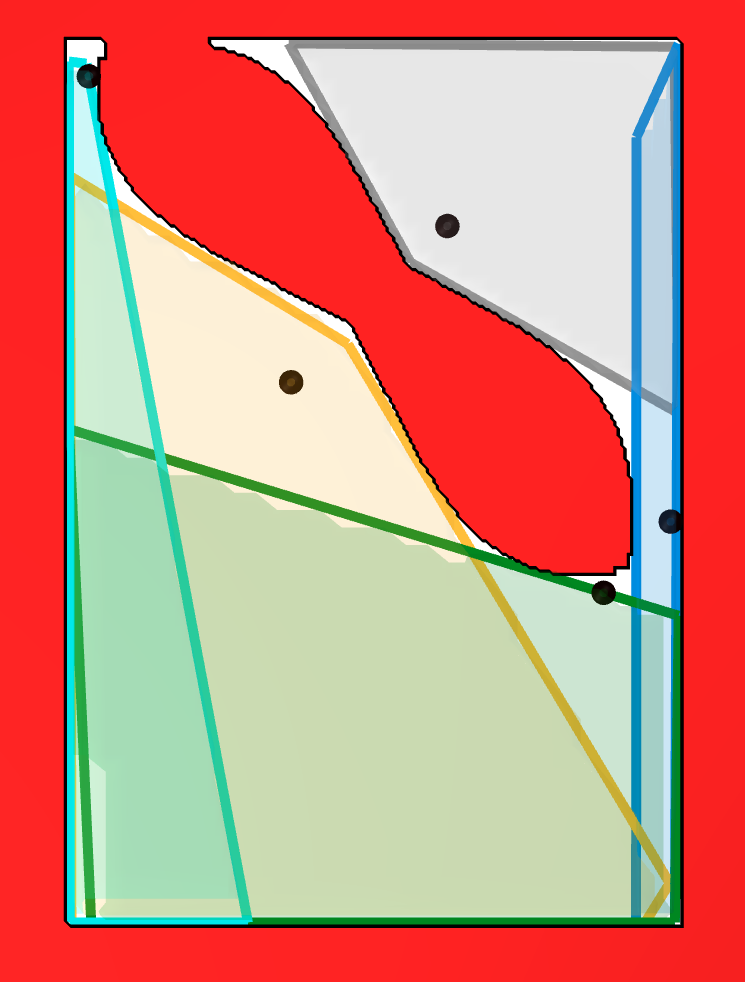}};
        \draw[stealth-, very thick,black] (image.south east) -- (image.south west) node[midway, below]  {Right Flipper Joint Angle};
        \draw[stealth-, very thick,black] (image.north west) -- (image.south west) node[midway, rotate=90, above]   {Left Flipper Joint Angle};
        \end{tikzpicture}
        \caption{The TC-space of the 2DOF pendulum flipper system. The TC-space obstacle is shown in red. Algorithm \ref{Alg: Bilinear Alternation} is run for five different polytopes each initially centered around the black dots. The polytopes output by the algorithm are plotted in various colors. These polytopes almost fully cover TC-free and are guaranteed to be collision-free by construction.}
        \label{F: pinball cspace}
    \end{subfigure}
\caption{The pinball flipper system and its TC-space. Algorithm \ref{Alg: Bilinear Alternation} is successfully able to cover TC-free with polytopic regions. An animation of the regions growing to cover this space is available \href{https://alexandreamice.github.io/project/c-iris/pinball_growth.html}{here}} 
\end{figure*}

We refer to our second system shown in Figure \ref{F: pinball task space} as the pinball flipper. Each orange arm is connected to its gray base via a revolute joint. Each collision geometry in the scene is approximated with a box and a total of $130$ collision pairs exist. We similarly plot the TC-space in Figure \ref{F: pinball cspace} with the TC-space obstacle highlighted in red. In this experiment, we attempt to almost completely cover TC-free with polytopic regions in order to enable a motion plan where the flippers exchange positions. Overall, this scene exhibits a much more complicated TC-space obstacle as well as substantially more collision pairs when compared to the system from Section \ref{S: Pend on Rail}.

We run Algorithm \ref{Alg: Bilinear Alternation} seeded with octagonal regions of side length $0.01$, each centered at one of $5$ different configurations shown as the black dots in Figure \ref{F: pinball cspace}. The resulting regions are also plotted in Figure \ref{F: pinball cspace} and almost completely cover the space. Though each region was initially seeded with a polytope of the same shape, our algorithm successfully adapts the shape of each polytope to fill the space. Our algorithm also is not conservative; it successfully finding regions which are tight to the TC-space obstacle in all cases. 

The change in volume of the maximum inscribed ellipsoid of each region is shown in Figure \ref{F: pinball vol growth}. We remark that the volume of each region exhibits a diverse set of behaviors over the iterations. Each region was grown sequentially, with a total wall time to cover the space of $1439$s. This wall time could easily be improved by growing each region in parallel.

In Figure \ref{F: pinball trajectory}, we demonstrate the behavior of our certificates for various poses of our robot. In the top panel, we highlight in black the two tips of each flipper. The current configuration is highlighted as the green dot in the bottom panel. For each configuration, we also plot the hyperplane that proves the separation between the two black tips. Notice that in Figures \ref{F: pinball cspace1}, \ref{F: pinball cspace2}, and \ref{F: pinball cspace3}, the current configuration is contained in multiple regions at once. Therefore, each hyperplane in Figure \ref{F: pinball traj0}
- \ref{F: pinball traj4} is drawn in the same color as its associated TC-space region in Figures \ref{F: pinball cspace0}
- \ref{F: pinball cspace4}. 

We draw attention to the fact that at every configuration $s_{0}$ in TC-free, many different separating hyperplanes exist. The hyperplane obtained by evaluating the output of our certifier at $s_{0}$ is highly dependent on the region which is being certified. For example, in Figure \ref{F: pinball cspace2}, the blue region corresponds largely to a change in the position of the left flipper, while the green region corresponds largely to a change in the right flipper. We see in Figure \ref{F: pinball traj2}, that the algorithm finds different separating planes for the blue and the green region, even for the same configuration, so as to accommodate the different range of robot motion in each region. For the blue region, which includes a large rotation of the left flipper, the blue plane would continue to separate the left flipper from the right flipper as the left flipper moves. Similarly, the green plane would continue to separate the right flipper from the left as the right flipper moves.

\begin{figure}
\begin{subfigure}[c]{0.66\textwidth}
    \centering
        \includegraphics[width = 0.98\textwidth]{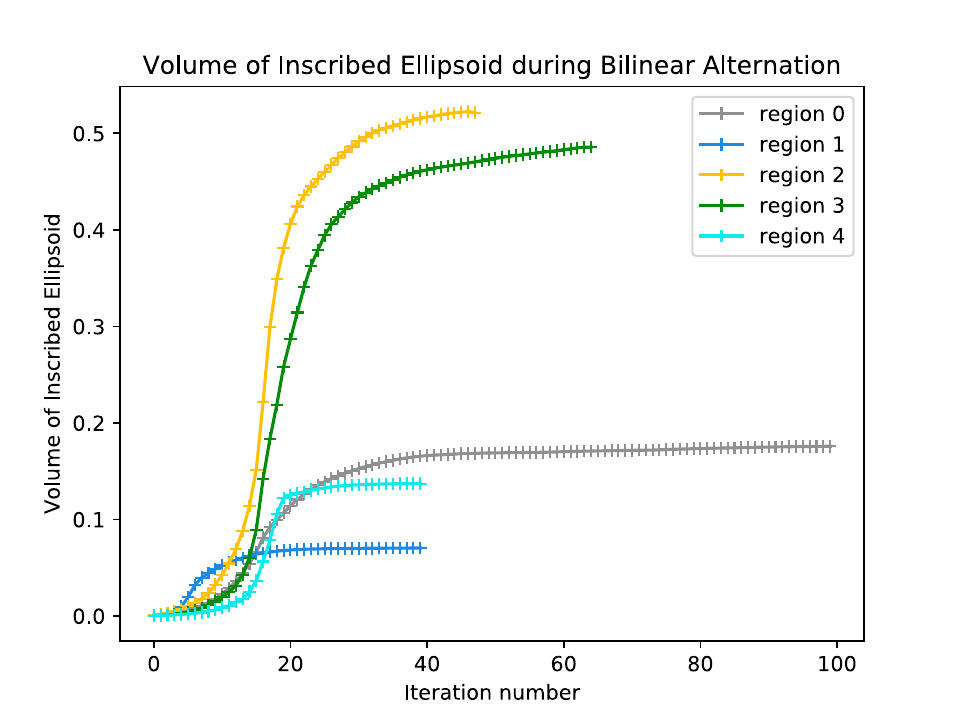}
        \caption{The volume of the maximum inscribed ellipsoid as the polytope is grown around various seedpoints is improved during Algorithm \ref{Alg: Bilinear Alternation}. The final polytopes associate to each color are shown in Figure \ref{F: pinball cspace}.}
        \label{F: pinball vol growth}
    \end{subfigure}
    \hfill
    \begin{subfigure}[c]{0.33\textwidth}
    \vspace{29mm}
    \centering
        % \begin{table}
        \begin{tabular}{p{1in} | c}
        \hline
        \Centering{Number of collision pairs} & 130 \\
        \hline
        \Centering{Size of the largest PSD variable} & 2 \\
        \hline
        \Centering{Average time to solve} \eqref{E: cert by hyperplane poly} & 0.638s \\
        \hline
        \Centering{Average time to solve \eqref{E: polytope growth program}} & 1.319s \\
        % \hline
        % \Centering{Number of regions to cover} & 5 \\
        \hline
        \Centering{Wall time to grow cover} & 1439s \\
        \hline
        \end{tabular}
        \vspace{13mm}
        % \end{table}
        \caption{Statistics dominating the run time of Algorithm \ref{Alg: Bilinear Alternation} for the pinball flipper system. The complexity scales with the number of collision geometries as well as the size of the largest PSD matrix variable for enforcing the Psatz conditions in Programs \eqref{E: cert by hyperplane poly} and \eqref{E: polytope growth program}.} \label{F: pinball stats}
    \end{subfigure}
    % \begin{subfigure}[t]{0.32\textwidth}
    % \centering
    %     \includegraphics[width = 0.98\textwidth]{figures/placeholder.png}
    % \end{subfigure}
    \caption{The progress of Algorithm \ref{Alg: Bilinear Alternation} on the pinball flipper system for each polytopic region is plotted. Statistics dominating the run time of the algorithm are also reported.}
    \label{F: pinball stats main}
\end{figure}

\begin{figure*}
 \captionsetup[subfigure]{justification=centering}
    \begin{subfigure}[c]{0.19\textwidth}
    \centering
        \includegraphics[width = \textwidth]{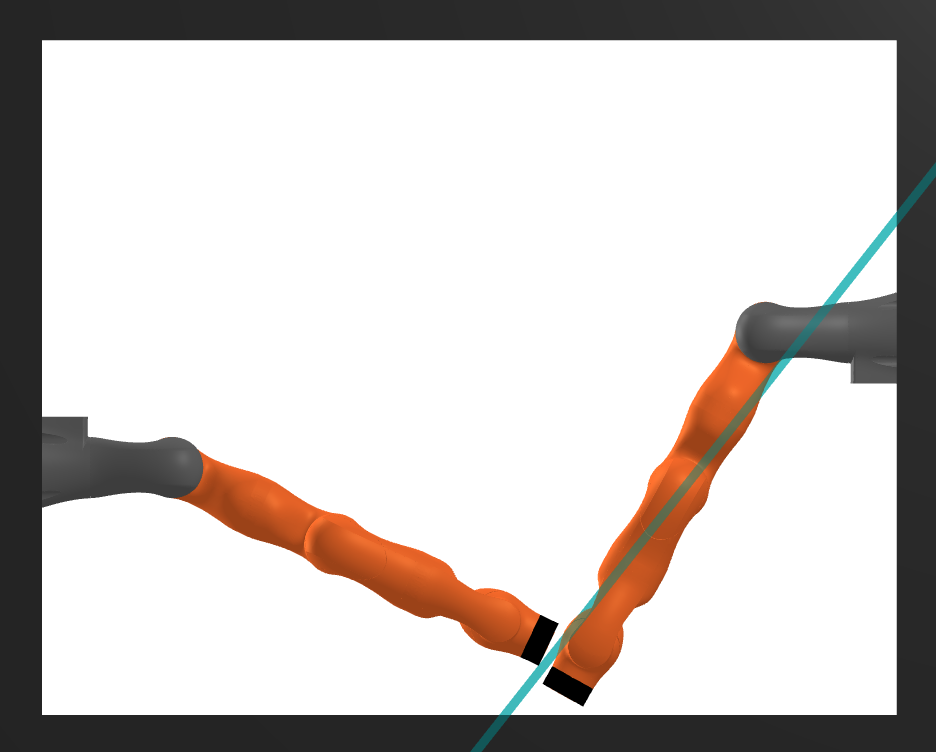}
        \caption{}\label{F: pinball traj0}
    \end{subfigure}
    \hfill
    \begin{subfigure}[c]{0.19\textwidth}
    \centering
        \includegraphics[width = \textwidth]{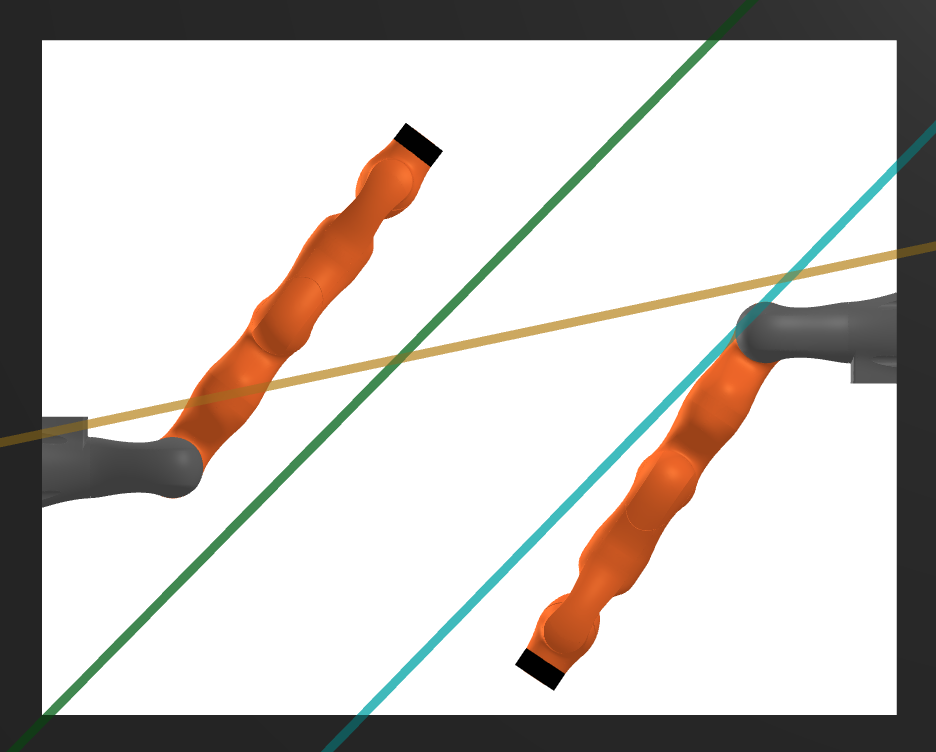}
        \caption{}\label{F: pinball traj1}
    \end{subfigure}
    \hfill
    \begin{subfigure}[c]{0.19\textwidth}
    \centering
        \includegraphics[width = \textwidth]{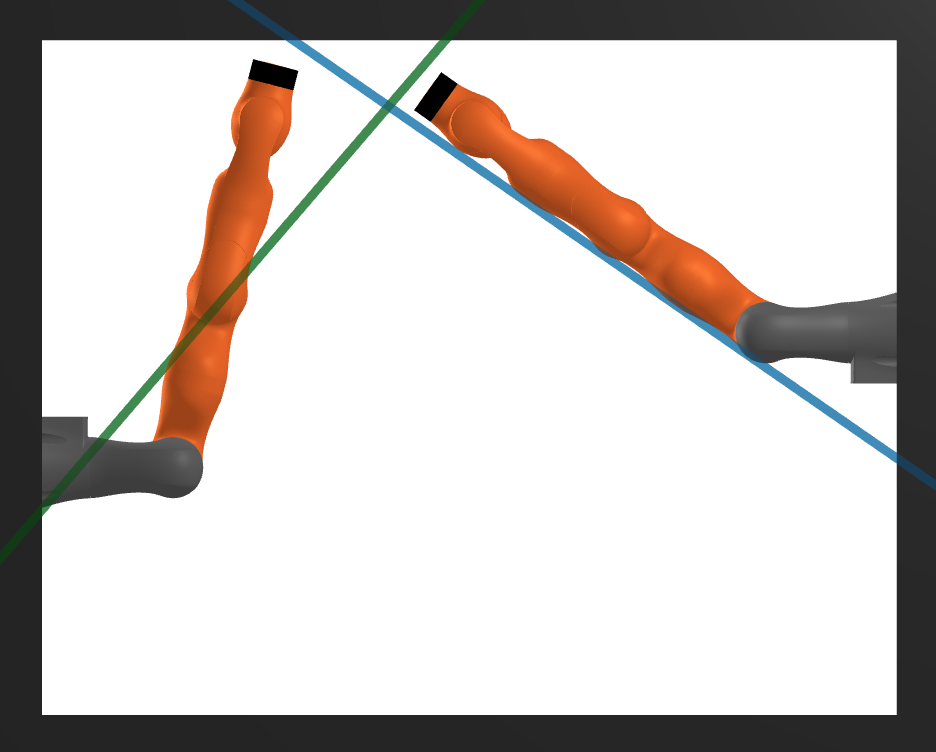}
        \caption{}\label{F: pinball traj2}
    \end{subfigure}
    \hfill
    \begin{subfigure}[c]{0.19\textwidth}
    \centering
        \includegraphics[width = \textwidth]{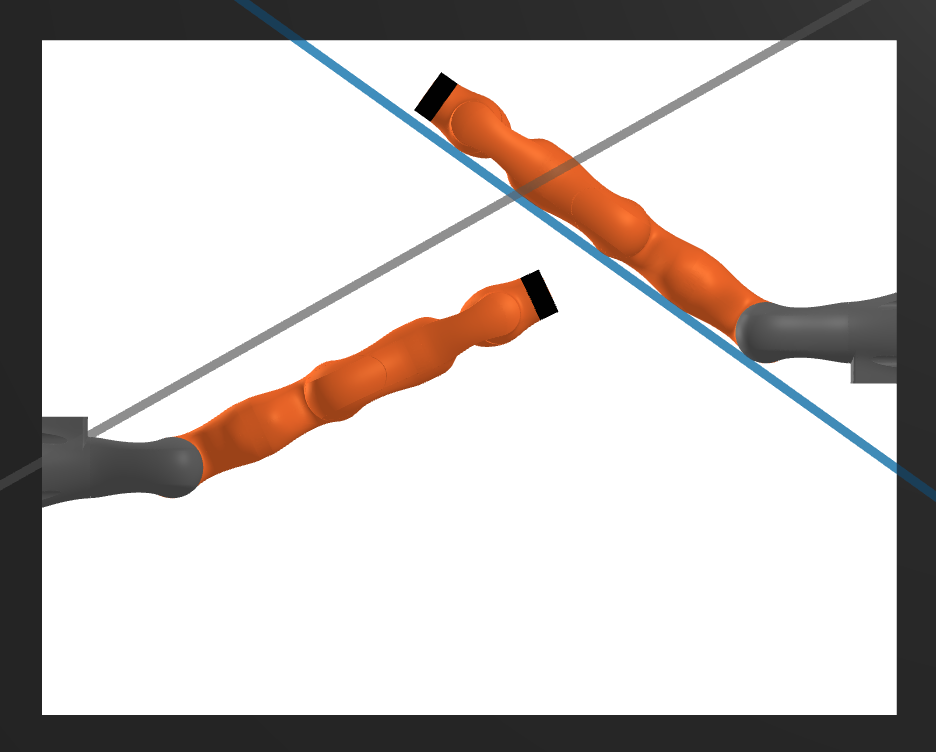}
        \caption{}\label{F: pinball traj3}
    \end{subfigure}
    \hfill
    \begin{subfigure}[c]{0.19\textwidth}
    \centering
        \includegraphics[width = \textwidth]{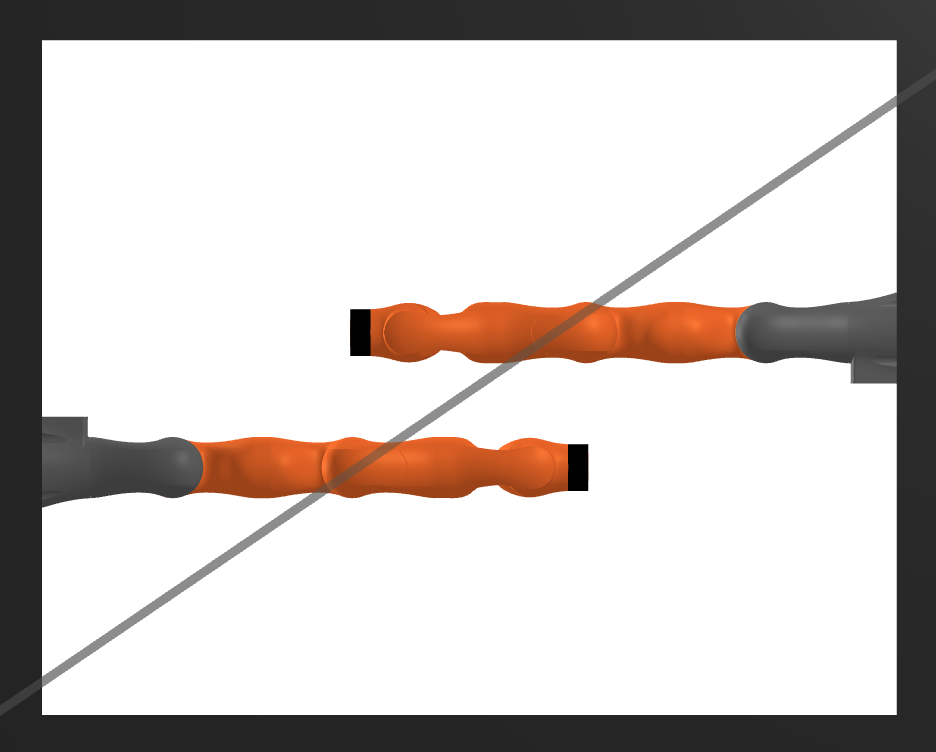}
        \subcaption{}\label{F: pinball traj4}
    \end{subfigure}
    %%%%%%%%%%%%%%%%%%%%%%%%%%%%%%%%
    \begin{subfigure}[c]{0.19\textwidth}
    \centering
        \includegraphics[width = 0.98\textwidth]{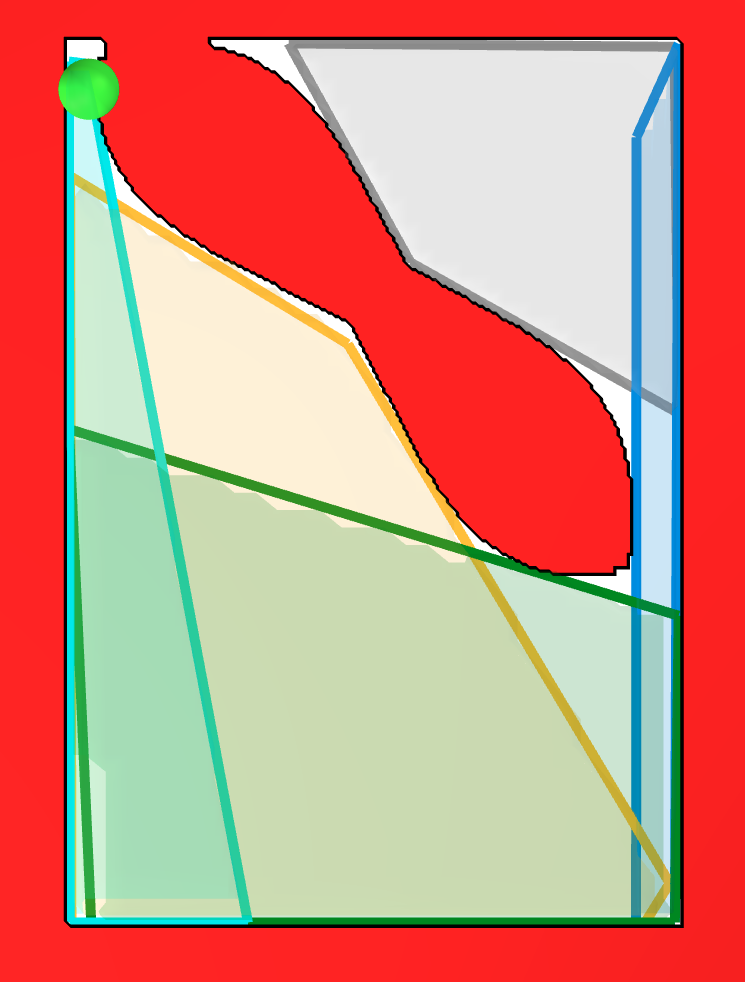}
        \caption{}\label{F: pinball cspace0}
    \end{subfigure}
    \hfill
    \begin{subfigure}[c]{0.19\textwidth}
    \centering
        \includegraphics[width = 0.98\textwidth]{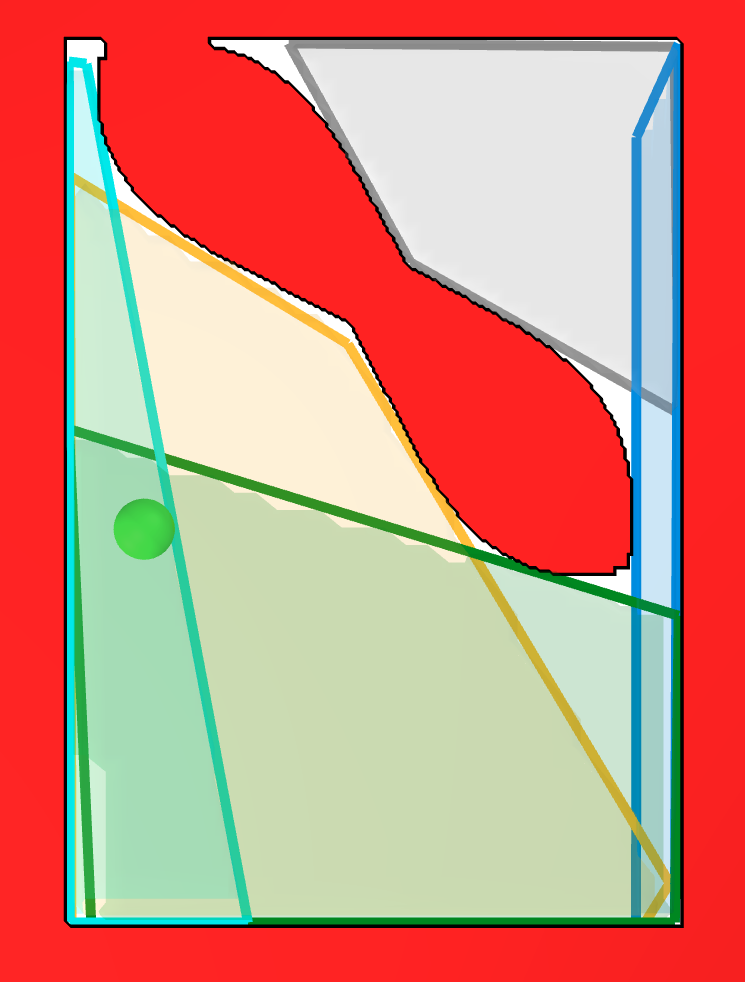}
        \caption{}\label{F: pinball cspace1}
    \end{subfigure}
    \hfill
    \begin{subfigure}[c]{0.19\textwidth}
    \centering
        \includegraphics[width = 0.98\textwidth]{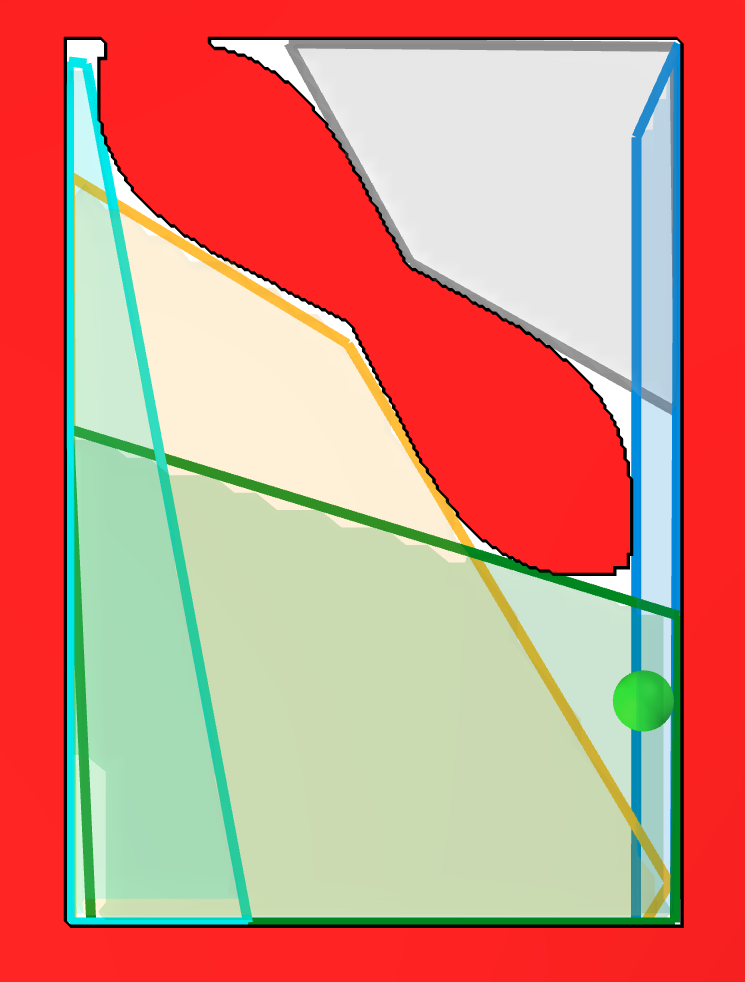}
        \caption{}\label{F: pinball cspace2}
    \end{subfigure}
    \hfill
    \begin{subfigure}[c]{0.19\textwidth}
    \centering
        \includegraphics[width = 0.98\textwidth]{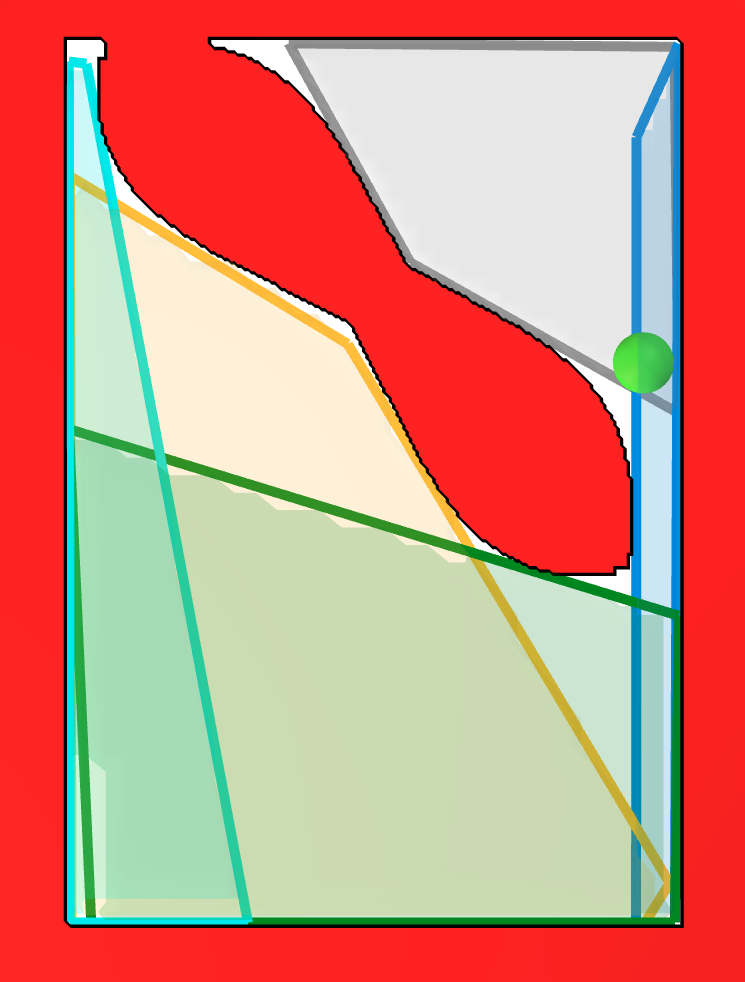}
        \caption{}\label{F: pinball cspace3}
    \end{subfigure}
    \hfill
    \begin{subfigure}[c]{0.19\textwidth}
    \centering
        \includegraphics[width = 0.98\textwidth]{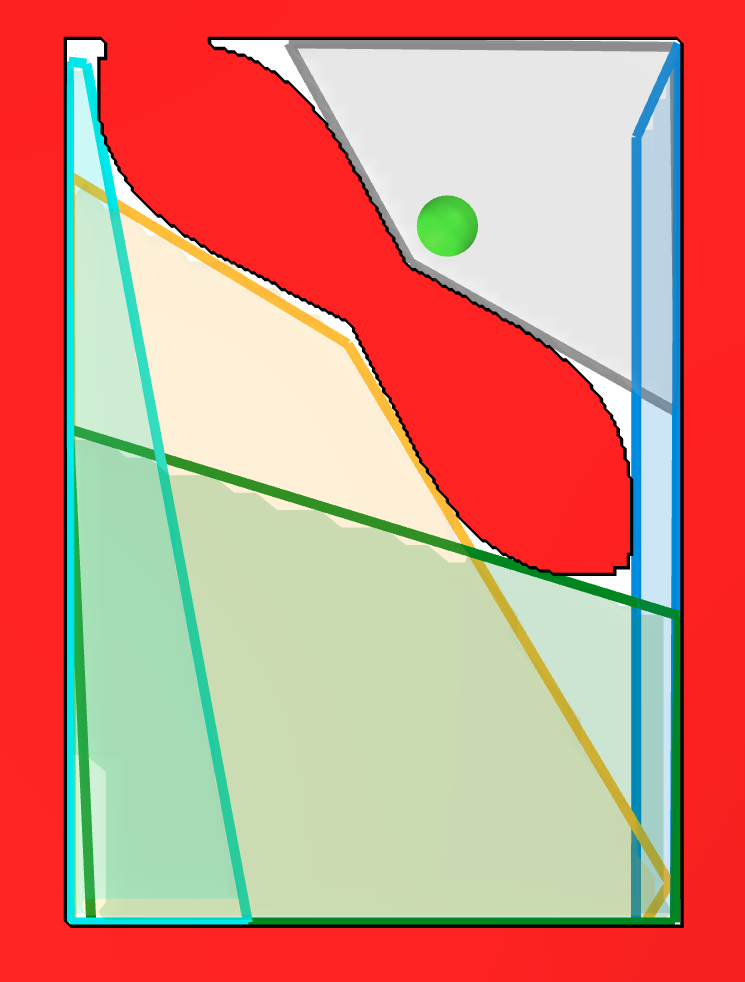}
        \caption{}\label{F: pinball cspace4}
    \end{subfigure}
    % \begin{subfigure}[t]{0.32\textwidth}
    % \centering
    %     \includegraphics[width = 0.98\textwidth]{figures/placeholder.png}
    % \end{subfigure}
    \caption{We approximate almost the entirety of TC-free for the robot flipper system using 5 polytopic regions. The top panel shows the hyperplanes certifying that the two black tips of the system do not collide. The bottom panel shows the configuration of the robot as a green dot. An example of this system undergoing a trajectory is available \href{https://alexandreamice.github.io/project/c-iris/pinball_trajectory.html}{here}.}
    \label{F: pinball trajectory}
\end{figure*}

    % \begin{itemize}
    %     \item 2-DOF revolute example visualization
    %     \item 3-DOF arm-on rail (2 revolute, 1 prismatic joints
    %     \item 7-DOF iiwa
    %     \item Bimanual iiwa?
    %     \item Examples with other convex shapes.
    % \end{itemize}
\subsection{KUKA IIWA robot} \label{S: iiwa}
In this section we demonstrate our algorithm deployed on the KUKA iiwa arm in two scenes relevant to robot manipulation. The collision geometry of the iiwa is approximated as a union of convex polytopes as are all obstacles in the scene. We begin by considering a single iiwa to demonstrate the practicality of our algorithm before considering a bimanual manipulator to demonstrate the scalability of our approach.

\subsubsection{7-DOF IIWA With a Shelf} \label{S: iiwa and shelf}
\begin{figure*}
    \centering
    \begin{subfigure}[t]{0.32\textwidth}
    \centering
        \includegraphics[width = 1.0\textwidth, trim={18cm 8cm 22cm 12cm},clip]{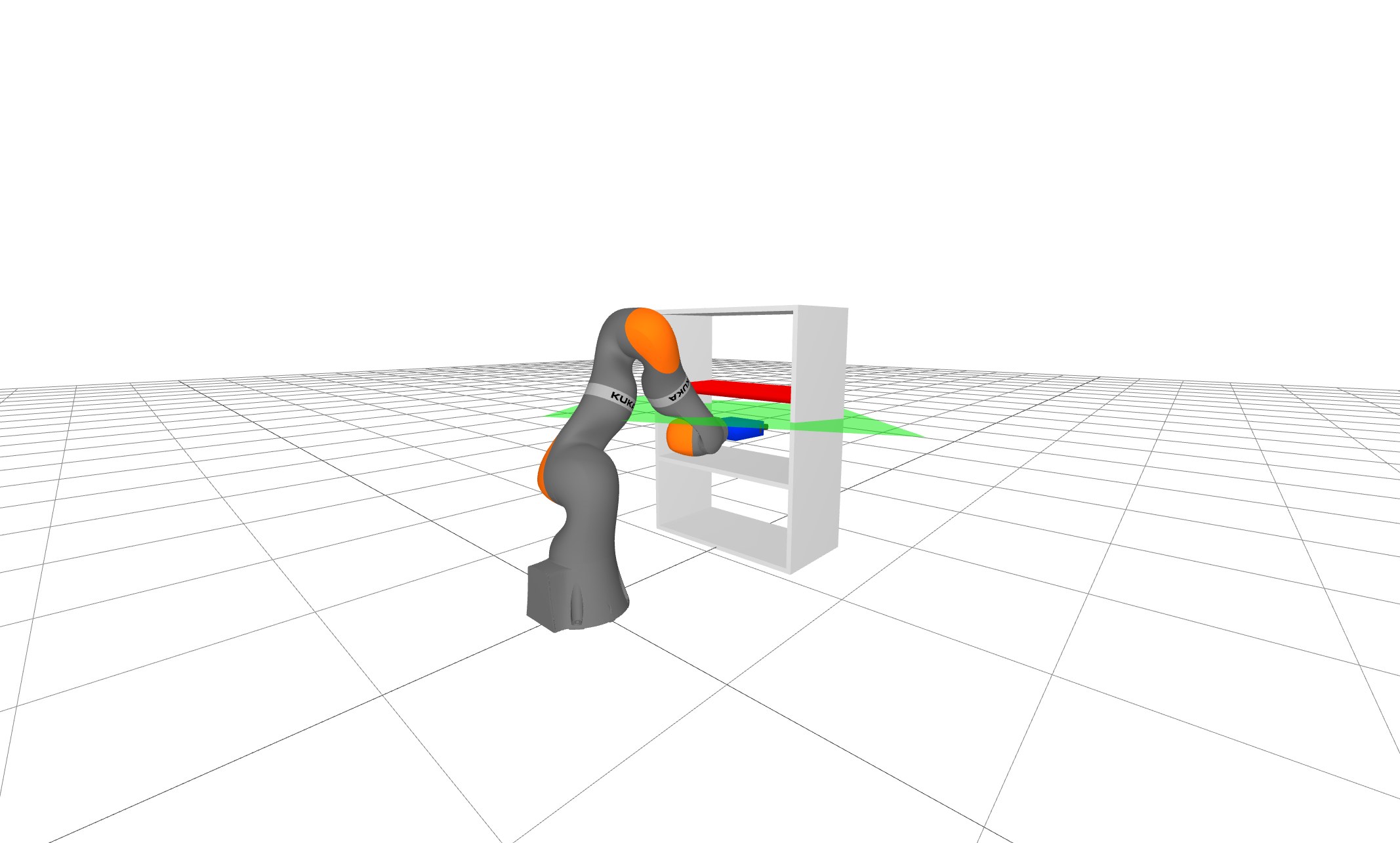}
        %\caption{Start of Trajectory} \label{F: cert cfg2}
    \end{subfigure}
    \begin{subfigure}[t]{0.32\textwidth}
    \centering
        \includegraphics[width = 1.0\textwidth, trim={18cm 8cm 22cm 12cm},clip]{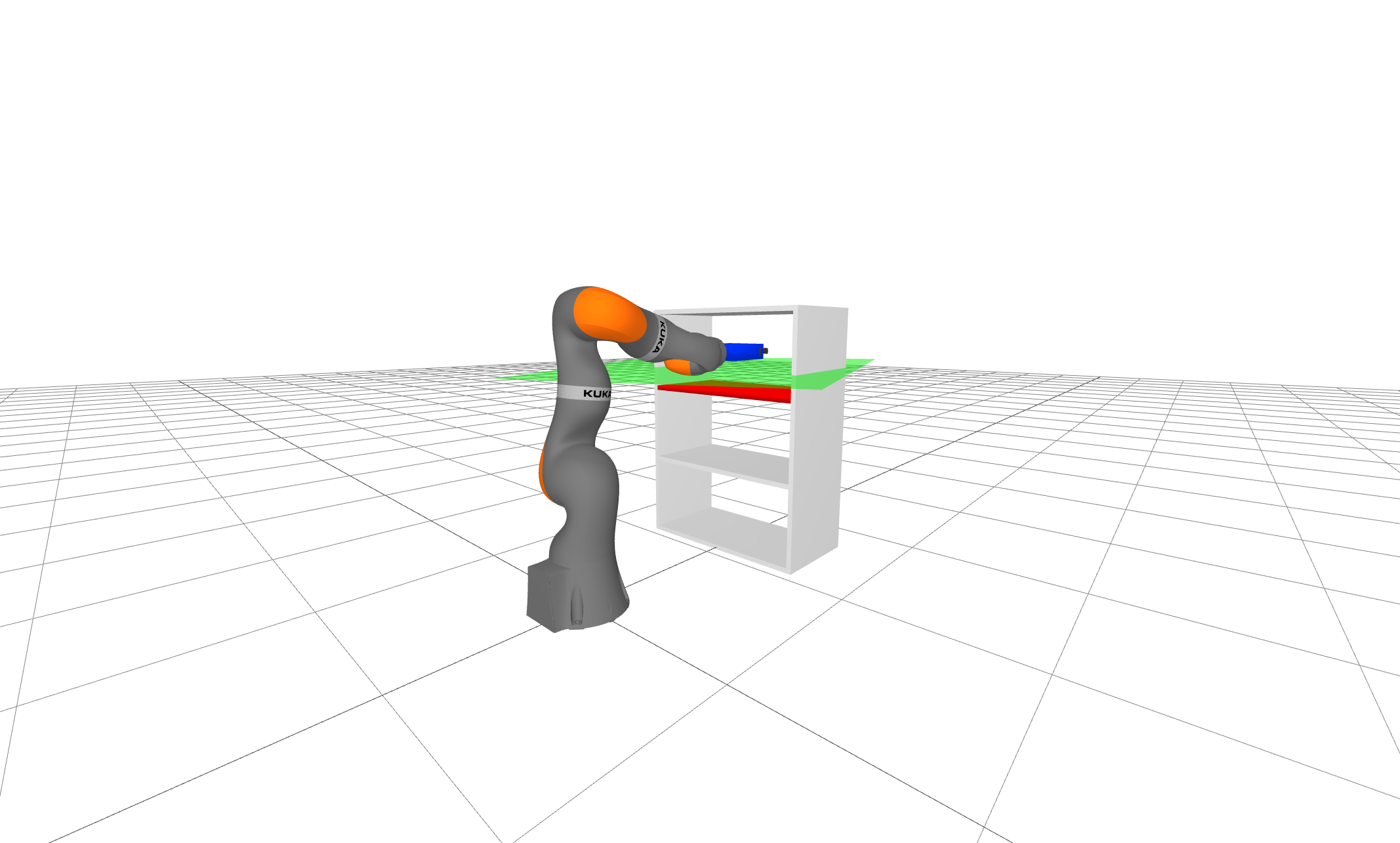}
        %\caption{Middle of Trajectory} \label{F: cert cfg1}
        \vspace{ 0.3 cm}
    \end{subfigure}
    \begin{subfigure}[t]{0.32\textwidth}
    \centering
        \includegraphics[width = 1.0\textwidth, trim={18cm 8cm 22cm 12cm},clip]{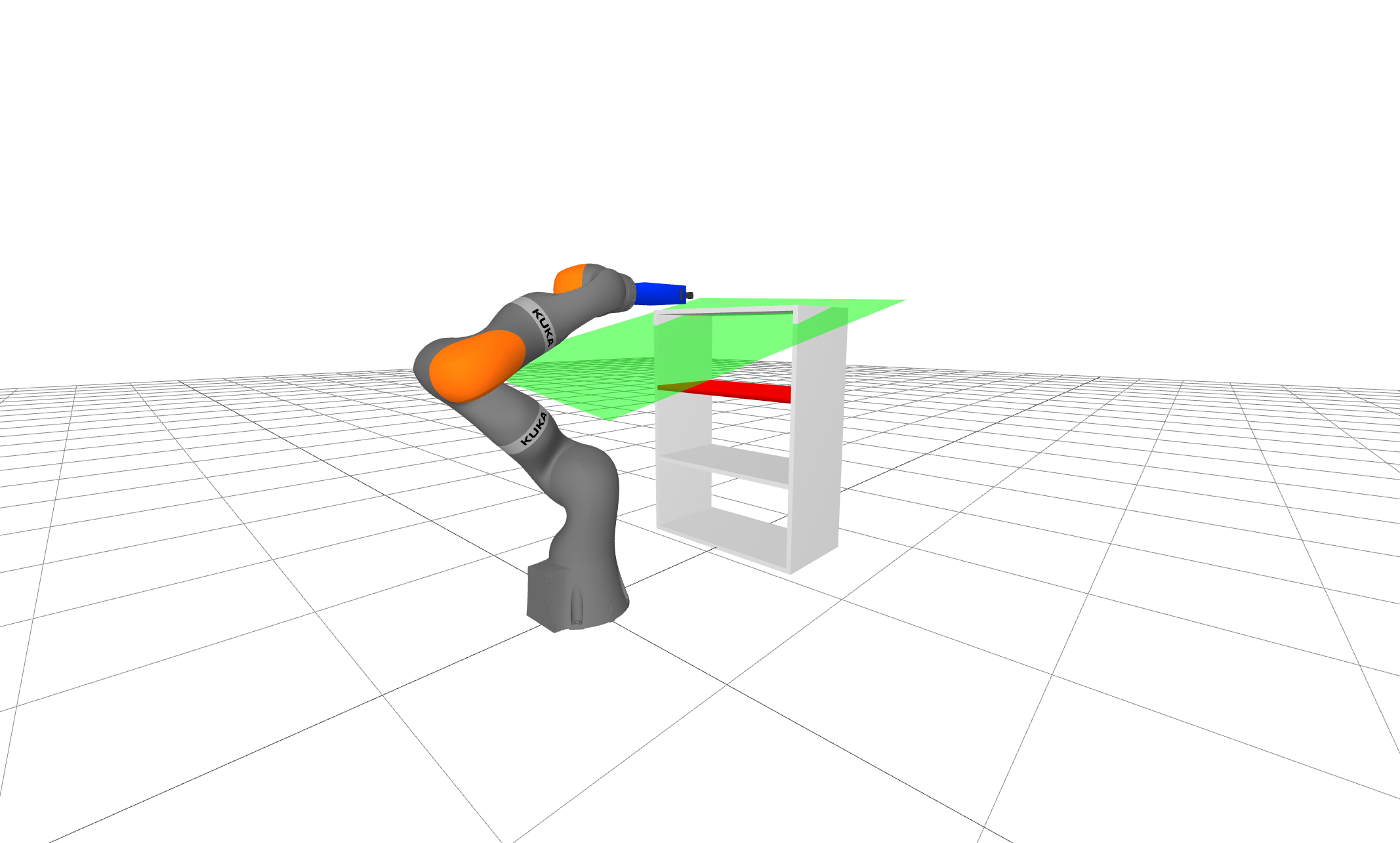}
        %\caption{Start of Trajectory} \label{F: cert cfg2}
    \end{subfigure}
    \caption{\small{7-DOF iiwa  example. We highlight one pair of collision geometries (blue on robot gripper and red on the shelf), together with their separating plane (green).}}
    \label{F: collision constraint}
    \vspace{-13.5pt}
\end{figure*}

We apply Algorithm \ref{Alg: Bilinear Alternation} to the scene shown in Figure \ref{F: collision constraint}: a 7-DOF KUKA iiwa arm reaching into a shelf. Our approach successfully finds many collision-free configurations, and we plot in green the separating hyperplane certificate between the end-effector, highlighted in blue, and the top shelf highlighted in red. 

The run time of Algorithm \ref{Alg: Bilinear Alternation} is dominated by the certification of non-collision between the pairs with the longest kinematic chain, as this leads to the highest degree polynomials and hence semidefinite variables in programs \eqref{E: cert by hyperplane poly} and \eqref{E: polytope growth program}. For this program, the largest positive semidefinite matrix variable has $16$ rows. Overall, the largest certification program \eqref{E: cert by hyperplane poly} takes 54s to solve, while the program \eqref{E: polytope growth program} takes on average 8s to solve.

% We also show the growth of volume for one certified C-free region in Fig \ref{fig:iiwa_shelf_volume}. The volume increases by a factor of 10,000 in 11 iterations of Algorithm \ref{Alg: Bilinear Alternation}, and covers a range of configurations (Fig \ref{fig:iiwa_shelf_volume} left). In step 3 of Algorithm \ref{Alg: Bilinear Alternation}, the largest SOS take 54s to solve, and in step 4 of Algorithm \ref{Alg: Bilinear Alternation}, the SOS takes 8s to solve. In both steps, the largest positive semidefinite matrices have 16 rows.

In Figure \ref{fig:iiwa_shelf_volume}, we demonstrate the behavior of one certified region. In Figure \ref{F: iiwa multi posture}, we show that the configurations of one of our certified polytopic region of TC-space (with 24 faces in the polytope) corresponds to many task-space end-effector positions. The configurations from Figure \ref{F: iiwa multi posture} are drawn from a region which grows by a factor of $10,000$ using $11$ iterations of Algorithm \ref{Alg: Bilinear Alternation}. This improvement in volume is reported in Figure \ref{F: iiwa vol growth}, where we also compare the volume of the maximum volume inscribed ellipsoid against the volume of the polytopic region.

\begin{figure}
    \centering
    \begin{subfigure}[c]{0.35\textwidth}
    \vspace{18mm}
    \centering
    \includegraphics[width=1.\textwidth]{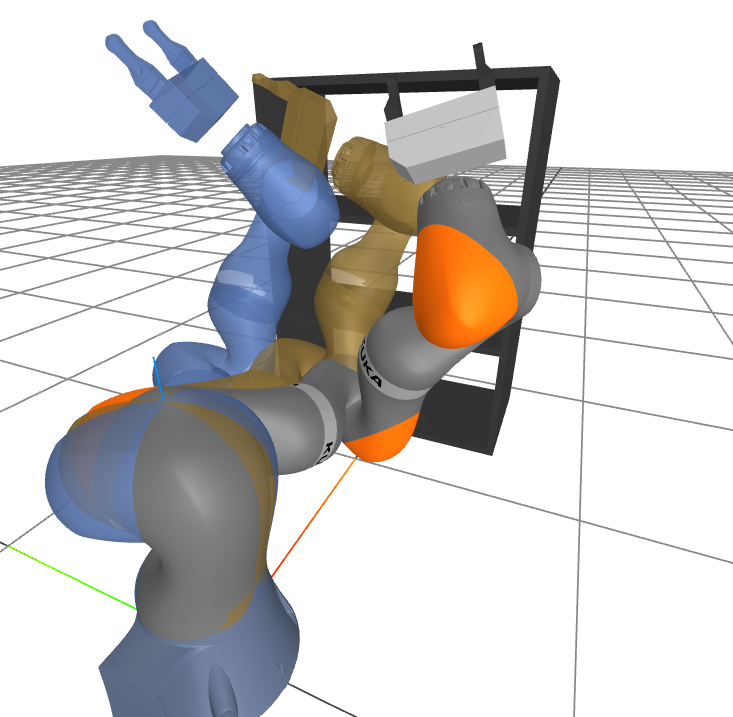}
    \vspace{5mm}
    \caption{The configurations in our certified regions correspond to a wide range of task-space positions. We sample three configurations from the same certified region and plot the corresponding task-space position in different colors.} \label{F: iiwa multi posture}
    \end{subfigure}
    \hfill
    \begin{subfigure}[c]{0.64\textwidth}
    \centering
    \includegraphics[width=1.\textwidth]{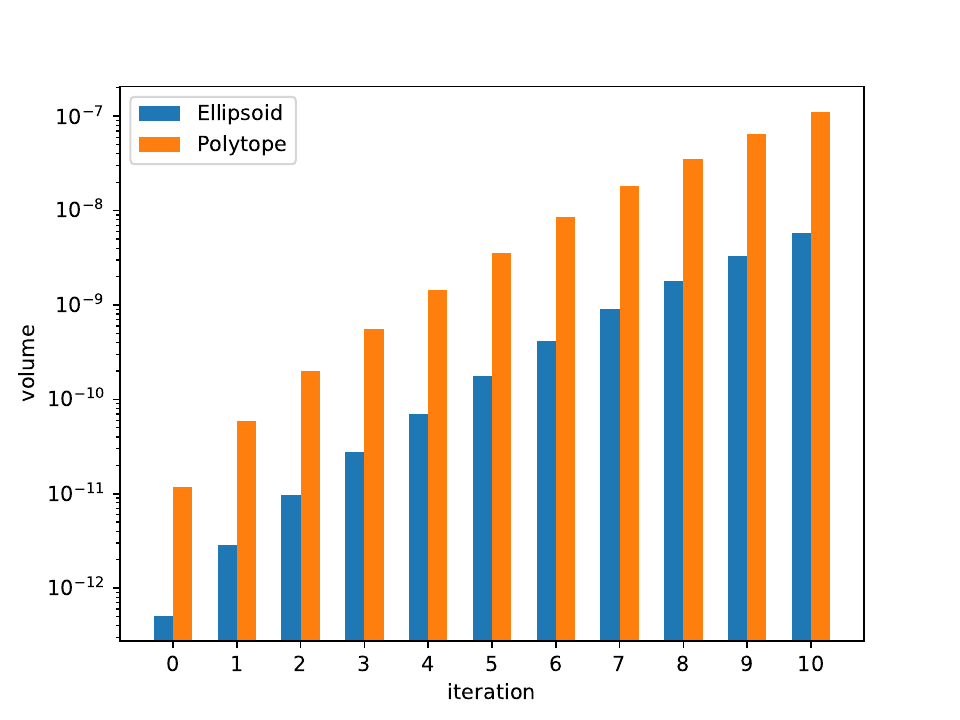}
    \caption{A single region for the 7-DOF KUKA iiwa is grown over the course of 11 iterations of Algorithm \ref{Alg: Bilinear Alternation}. We compare the volume of the maximum volume inscribed ellipsoid to the volume of the polytopic region at each iteration and show that the volume improves by a factor of 10,000.} \label{F: iiwa vol growth}
    \end{subfigure}
    \caption{Algorithm \ref{Alg: Bilinear Alternation} grows certified regions which contain configurations reaching a large portion of the task space. We show that our algorithm is capable of growing the volume of a certified region by a factor 10,000 over the course of just 11 iterations.}
    \label{fig:iiwa_shelf_volume}
    \vspace{-13.5pt}
\end{figure}

\subsubsection{12-DOF Bimanual KUKA IIWA Example} \label{S: bimanual iiwa}
We next consider designing regions to avoid self-collision for a robot consisting of two KUKA iiwa arms with the final joint welded (rotation of the final joint does not change the configuration of any geometry for this robot). This robot contains 12-DOF. This system tests the scalability of our algorithm due to the degree of the polynomials involved in the forward kinematics, as well as the complexity of the collision geometries.

Solving the largest certification program in \eqref{E: cert by hyperplane poly} takes 105 minutes, while the program in \eqref{E: polytope growth program} takes 4 minutes. The increase in solve times compared to the single iiwa environment from Section \ref{S: iiwa and shelf} is best attributed to the increase in the size of the semidefinite variables due to the larger DOF. The largest semidefinite matrix in both programs have $64$ rows and correspond to certifying that the two tips of the iiwas do not collide. 

Nonetheless, our algorithm again finds certified, 30-face polytopic regions of TC-space which correspond to a wide range of task-space positions as seen in Figure \ref{F: bimanual iiwa position}. Moreover, the same region is quite tight to the TC-space obstacle; one sampled configuration in the certified region, shown in Figure \ref{F: bimanual close}, corresponds to just $7.3$mm of separation between the two arms. 

\begin{figure*}[htb]
\centering
\begin{subfigure}[c]{0.45\textwidth}
\centering
\includegraphics[width=0.9\textwidth]{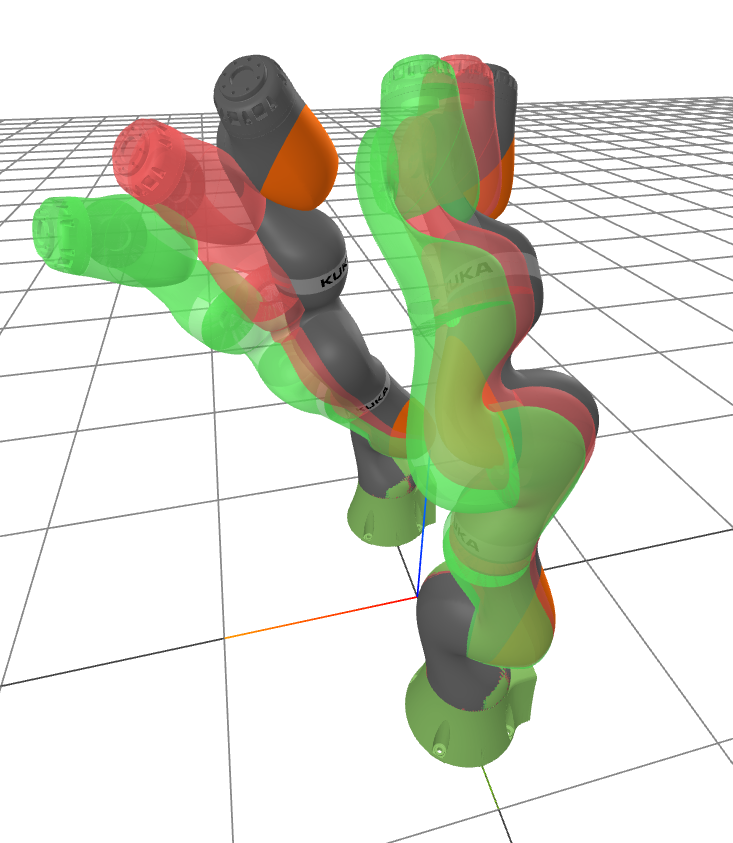}
\subcaption{Multiple configurations of the 12-DOF, bimanual iiwa manipulator sampled from a single certified region of TC-free. Each configuration is shown in a separate color.} \label{F: bimanual iiwa position}
\end{subfigure}
\hfill
\begin{subfigure}[c]{0.45\textwidth}
\vspace{9.5mm}
\centering
\includegraphics[width=0.9\textwidth]{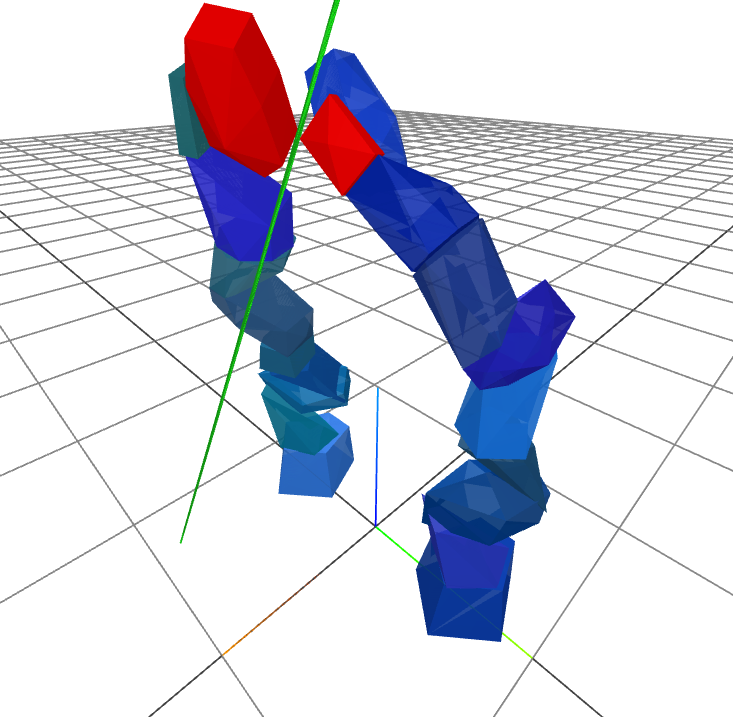}
\vspace{4mm}
\subcaption{The geometries of the bimanual iiwa from Figure \ref{F: bimanual iiwa position} are tightly approximated using polytopes. At one position in the certified TC-free region, the two geometries highlighted in red are separated by just $7.3$mm.}
\label{F: bimanual close}
\end{subfigure}
\caption{Algorithm \ref{Alg: Bilinear Alternation} finds certified polytopic regions of TC-free even for high DOF systems in reasonable times. The algorithm is also not conservative. It finds large regions which correspond to a broad range of task-space positions. Moreover, the regions are very tight to the TC-space obstacle, finding configurations which lead to very small separation between the task-space objects.}
\label{Fig: dual_iiwa}
\end{figure*}

\subsection{UR3e Robot}
In this section, we test our algorithm on a UR3e robot with a gripper mounted at the wrist. The robot's links are approximated by cylinders and we weld the gripper's prismatic joints so that each UR3e has a total of 6 DOFs. This section differs from the KUKA iiwa experiment in Section \ref{S: iiwa} due to the introduction of non-polytopic collision geometries into the scene. Similar to Section \ref{S: iiwa}, we test our approach for a scene where the robot is reaching into a shelf, as well as a bimanual set up.

% We also test our algorithm on a UR3e robot model (whose collision geometries are approximated by cylinders) . The grippers collision geometries are approximated by boxes. The gripper's prismatic joint is welded, hence each UR3erobot + gripper has 6 DOFs.

% We first demonstrate our result between the robot and a shelf, with another box-shaped object on the shelf stack, as shown in Fig. \ref{fig:ur_shelf}. The different postures sampled within one certified C-space region showcase that our C-space region is big enough to contain a variety of postures. The largest SOS problem solved in Algorithm \ref{Alg: Bilinear Alternation} takes about 56s. The largest positive semidefinite matrices in the SOS problem have 16 rows.
\subsubsection{6-DOF UR3e With a Shelf}
\begin{figure}[htb]
\captionsetup[subfigure]{justification=centering}
\centering
\begin{subfigure}{0.45\textwidth}
\includegraphics[width=0.95\textwidth]{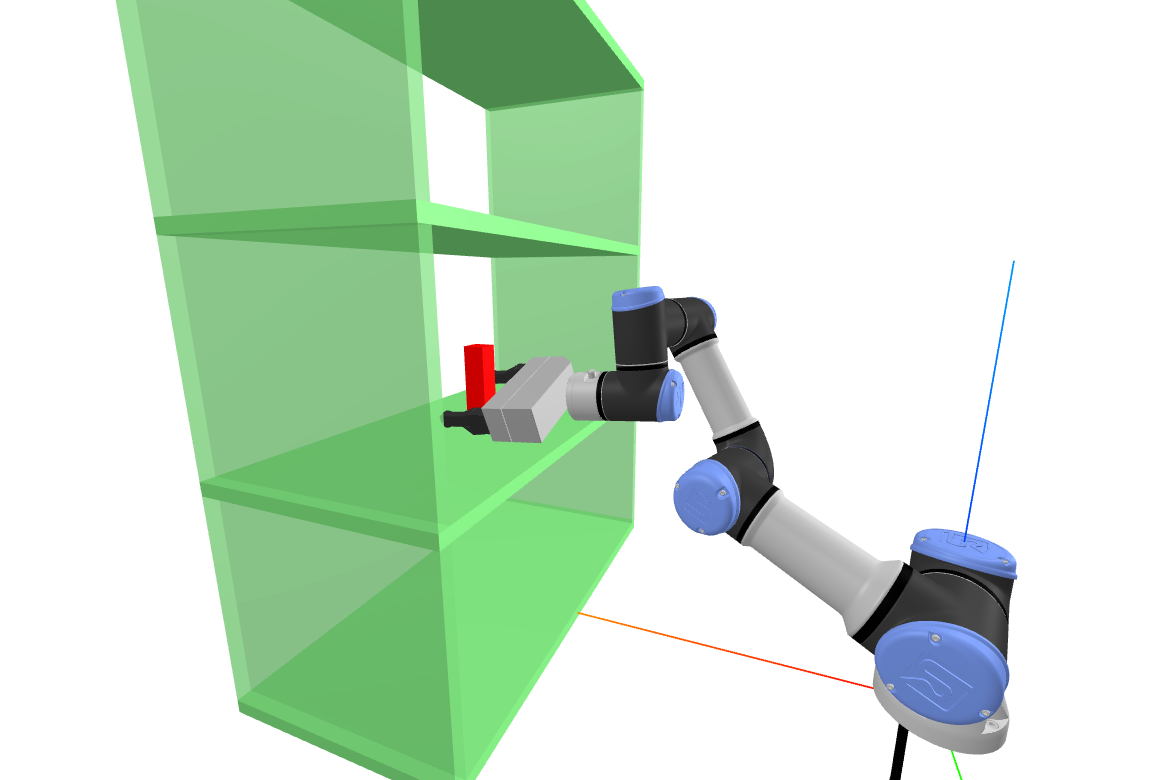}
\subcaption{}
\label{figure:ur_shelf1}
\end{subfigure}
\hfill
\begin{subfigure}{0.45\textwidth}
\includegraphics[width=0.95\textwidth]{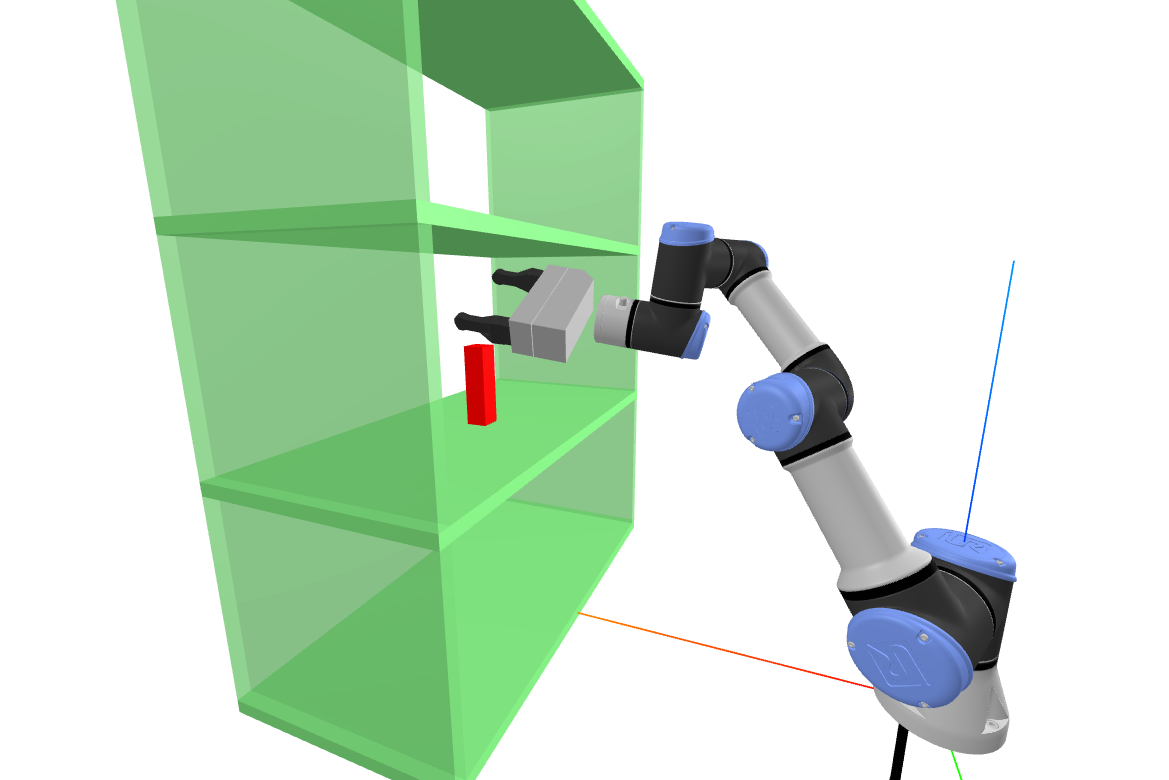}
\subcaption{}
\label{figure:ur_shelf2}
\end{subfigure}
\par \bigskip
\begin{subfigure}{0.45\textwidth}
\includegraphics[width=0.95\textwidth]{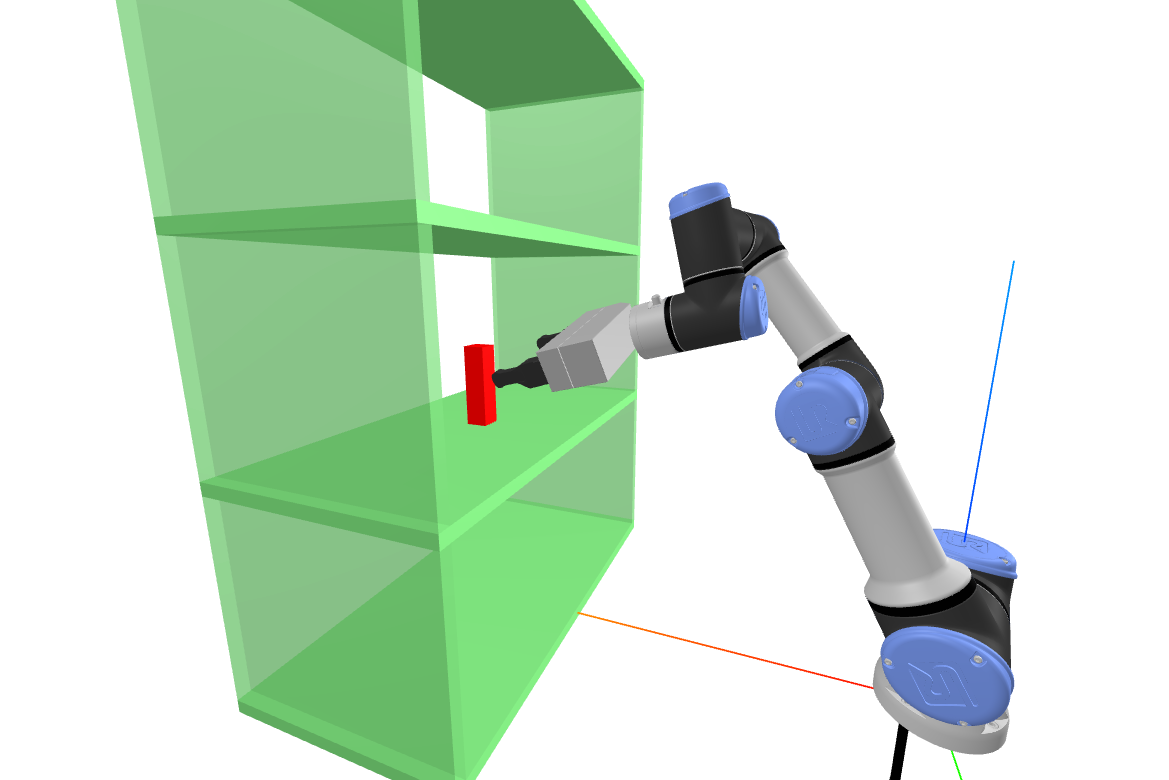}
\subcaption{}
\label{figure:ur_shelf3}
\end{subfigure}
\begin{subfigure}{0.45\textwidth}
\includegraphics[width=0.95\textwidth]{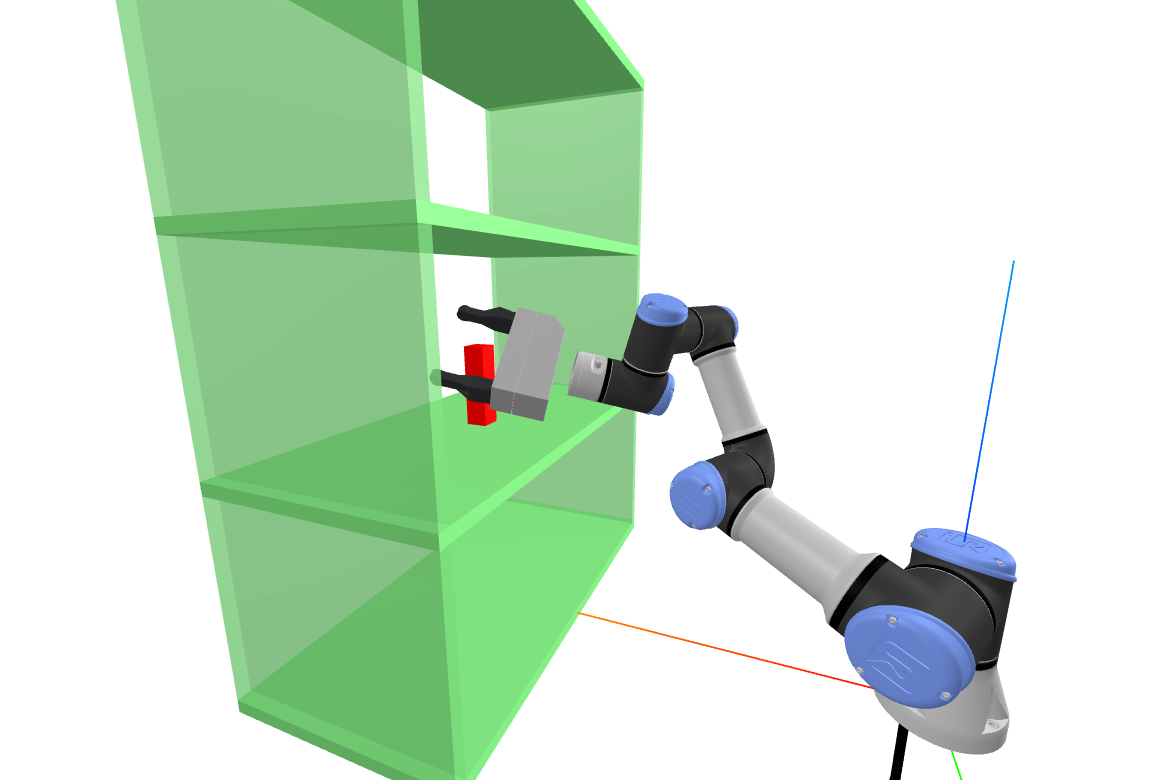}
\subcaption{}
\label{figure:ur_shelf4}
\end{subfigure}
\caption{Different postures sampled within one certified TC-space region for a UR3e robot with gripper. The certified-region include both the gripper reaching the red box in the center of the shelf (Fig.\ref{figure:ur_shelf1}), retracting from the shelf (Fig.\ref{figure:ur_shelf3}), and reaching different regions within the shelf while avoiding the red box (Fig.\ref{figure:ur_shelf2} and \ref{figure:ur_shelf4}). An animation of the range of configurations attainable in this region are available \href{https://alexandreamice.github.io/project/c-iris/ur_single.html}{here}.}
\label{fig:ur_shelf}
\end{figure}

In Figure \ref{fig:ur_shelf}, we consider a UR3e robot reaching into a shelf to grasp a small box shaped object. To simulate a situation where the robot is attempting to pick up the red object, we use Algorithm \ref{Alg: Bilinear Alternation} to grow a certified, TC-free polytope (with 12 faces) near the object. Figure \ref{fig:ur_shelf} shows a variety of postures sampled from the final TC-free polytope and demonstrates that within a single region, our robot is able to reach into the shelf to grasp the object, retract away from the shelf, and maneuver within the shelf while avoiding the object.

Similar to Section \ref{S: iiwa and shelf}, the largest semidefinite variables in programs \eqref{E: cert by hyperplane poly} and \eqref{E: polytope growth program} has $16$ rows with program \eqref{E: cert by hyperplane poly} taking about $56$s to solve.

\subsubsection{12-DOF Bimanual UR3e}

\begin{figure}[htb]
\captionsetup[subfigure]{justification=centering}
\centering
\begin{subfigure}{0.48\textwidth}
\includegraphics[width=0.98\textwidth]{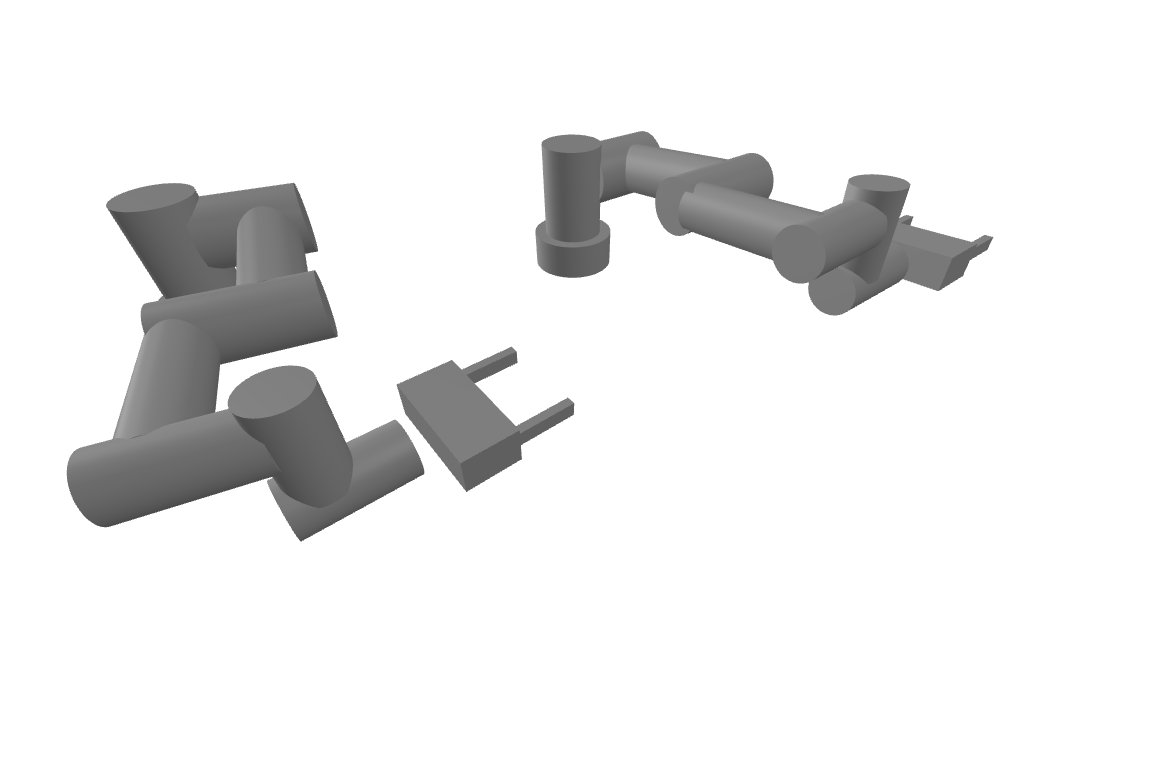}
\caption{}\label{fig:dual_ur_posture1}
\end{subfigure}
\begin{subfigure}{0.48\textwidth}
\includegraphics[width=0.98\textwidth]{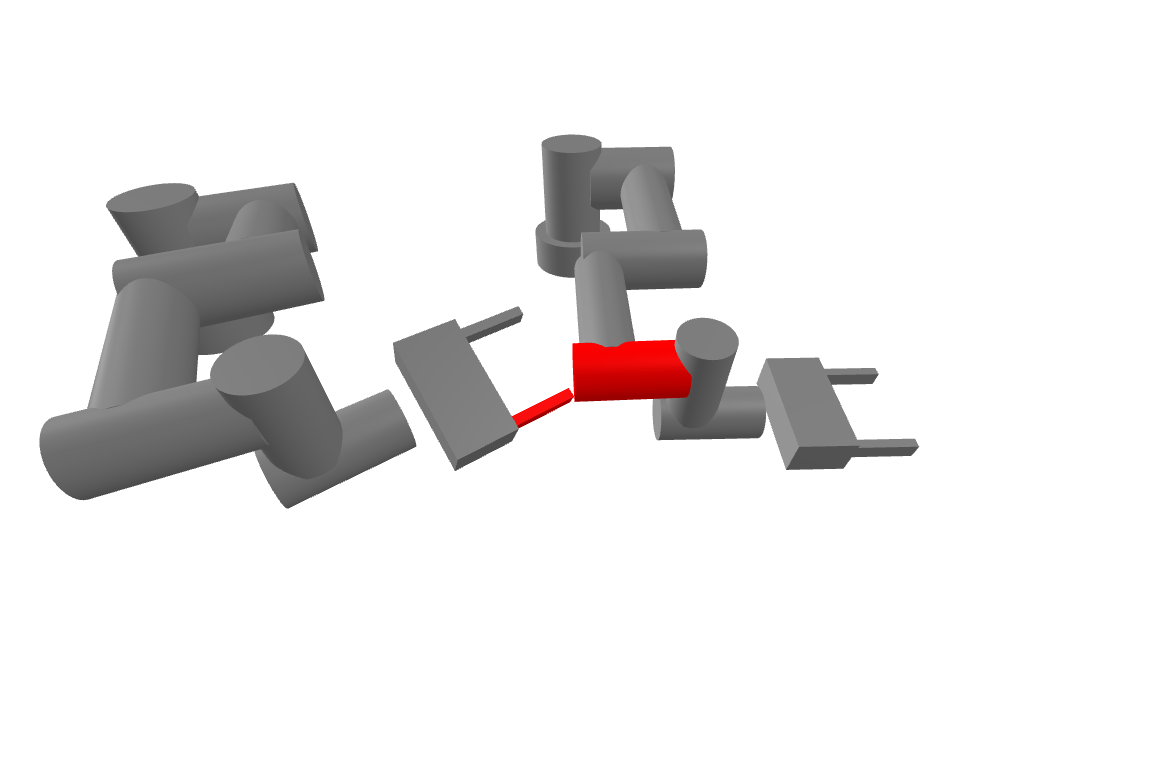}
\caption{}
\label{fig:dual_ur_posture2}
\end{subfigure}
\caption{Top down view of two postures sampled within one certified TC-space region on the dual UR3e platform. In the right figure we highlight the two collision geometries that are separated by only 0.3mm. A dynamic visualization of the range of attainable postures is available by running this \href{https://deepnote.com/workspace/alexandre-amice-c018b305-0386-4703-9474-01b867e6efea/project/C-IRIS-7e82e4f5-f47a-475a-aad3-c88093ed36c6/notebook/dual_ur-8fc84da71e494588bbc82350826b417a}{notebook}}
\label{fig:dual_ur}
\end{figure}

Finally, we demonstrate our algorithm on a dual UR3e platform shown in Figure \ref{fig:dual_ur}. Again, we emphasize that we are able to find large regions of TC-configuration space which correspond to diverse positions in task space with the postures in Figure \ref{fig:dual_ur_posture1} and \ref{fig:dual_ur_posture2} being drawn from the same certified region (a 13-face polytope). Moreover, these regions are very tight to the TC-configuration space obstacle with the two bodies highlighted in red in Figure \ref{fig:dual_ur_posture2} being just $0.3$mm apart. For this example, the largest positive semidefinite matrices in \eqref{E: cert by hyperplane poly} and \eqref{E: polytope growth program} has 128 rows with the largest program taking about $35$ minutes to solve. This program solves faster than the analogous program for the bimanual iiwa from  Section \ref{S: bimanual iiwa} because we require fewer polynomial positivity conditions to certify that the UR3e's cylindrical geometries are on a given side of a plane compared to the polytopic approximation used for the iiwa.

\section{Conclusion} \label{S: Conclusion}

Understanding the complicated geometry of C-free is an essential step to designing safe, collision-free motion plans. In this work, we presented an approach for describing a rational parametrization of C-free, known as TC-free, using a union of polytopes. Our primary contributions are two Sums-of-Squares program \eqref{E: cert by hyperplane poly} and \eqref{E: dual psatz poly cert} which can certify that a polytopic region of TC-space is collision-free, as well as another program \eqref{E: polytope growth program} which finds a local improvement that increases the size of a TC-free polytope. We prove that feasibility of our certification programs \eqref{E: cert by hyperplane poly} and \eqref{E: dual psatz poly cert} are both necessary and sufficient for proving that a polytopic region of TC-space is collision-free and we combine programs \eqref{E: cert by hyperplane poly} and \eqref{E: polytope growth program} into a practical algorithm for describing TC-free as a union of certified, collision-free polytopes in the TC-space. We deployed our algorithm on both simple and realistic environments and demonstrate that Algorithm \ref{Alg: Bilinear Alternation} finds large TC-space regions which correspond to diverse positions in task space. We demonstrate that these regions are not conservative and very tight to the TC-space obstacle even for 12-DOF systems by showing postures with just millimeters of separation.

The presented method works for TC-spaces of arbitrary dimensions, makes only very mild assumptions on the kinematics of our robot, and makes no assumptions about the shape of the TC-space obstacles. Moreover, it only relies on the mild assumption that obstacles in the task space are described as unions of convex sets, an assumption that is frequently satisfied whenever a given environment is simulated. 

Such certified descriptions of TC-free find practical application in both randomized and optimization-based collision-free motion planning algorithms, providing a means to certify safety of an entire trajectory by checking membership in a set rather than by finite sampling which can be prone to false assertions of safety. Moreover, the convexity of the generated regions is particularly attractive to optimization-based methods such as the GCS framework of \cite[]{marcucci2022motion}. Future work intends to further explore these applications as well as practical algorithms for seeding Algorithm \ref{Alg: Bilinear Alternation} to obtain good coverage of TC-free with few regions.

\section{Acknowledgement}
This work was supported by the MIT Quest For Intelligence.
\bibliographystyle{SageH}
\bibliography{bibliographies/amice_refs, bibliographies/refs}
% \FloatBarrier
\clearpage

\appendix
\section{Algebraic Kinematics}\label{A: alg kin}
An in depth review of algebraic kinematics and low order pairs can be found in \cite[Chapter 4]{wampler2011numerical}. We include a brief review in this appendix for completeness.

A mechanism composed of $N+1$ links is considered algebraic if each link is connected by one of the following five joints:
\begin{itemize}
        \item Revolute (R): a 1-DOF joint permitting revolution about an axis of symmetry. An example is a door handle.
        \item Prismatic (P): a 1-DOF joint permitting translation along an axis. An example is a linear rail.
        \item Cylindrical (C): a 2-DOF joint permitting both revolution about an axis of symmetry and independent translation along a given axis. An example is the rods of a Foosball table.
        \item Planar (E): A 3-DOF joint permitting translation and rotation in a two-dimensional plane. An example is hockey puck moving on the surface of the ice. 
        \item Spherical (S): A 3-DOF joint permitting free rotation between two links. An example is the human shoulder.
\end{itemize}
We recall from Section \ref{S: Rat Forward} that the pose of a point $A$ expressed in the reference frame $F$, written as a function of the robot configuration $q$ can be expressed as
\begin{align}%\label{E: gen forward kin}
    \begin{bmatrix}
        \leftidx{^F}R^{A}(q) & \leftidx{^F}p^{A}(q) \\
        0_{1 \times 3} & 1 \\
    \end{bmatrix}
    =
    \prod_{i \in \calI_{F, A}} \leftidx{^{P_{i}}}X^{C_i}(q_{i}) \ \leftidx{^{C_i}} X ^{P_{i+1}}
\end{align}
where $\leftidx{^{P_{i}}}X^{C_{i}}(q_{i})$ is a rigid transform describing the relative motion allowed by the $i${\textsuperscript{th}} joint. The matrices $\leftidx{^{P_{i}}}X^{C_{i}}(q_{i})$ are in general restriction of the following forms
\begin{align} \label{E: gen low order pair complete}
    \leftidx{^{P_{i}}}X^{C_{i}}(q_{i}) &=
    \begin{cases}
    \begin{bmatrix}
        \cos(\theta_{i}) & -\sin(\theta_{i}) & 0 & x_{i} \\
        \sin(\theta_{i}) &  \cos(\theta_{i}) & 0 & y_{i} \\
        0  & 0  & 1 & z_{i} \\
        0 & 0 & 0 & 1
    \end{bmatrix}
    &
    \text{if $i$\textsuperscript{th} joint is one of R, P, C, or E}
    \\
    \begin{bmatrix}
        U(\psi_{i}) & 0_{3 \times 1} \\ 0 _{1 \times 3} & 1
    \end{bmatrix}
    &
    \text{if $i$\textsuperscript{th} joint is S}
    \end{cases}
\end{align}
The specific restrictions for R, P, C, and E joints are given in Table \ref{Tab: gen low order pair complete}. The matrix $U$ is an element of $SO(3)$ parametrized using Euler angles $\{\phi_{i,x}, \phi_{i,y}, \phi_{i,z}\}$.

\begin{table}[H]
\centering
    \begin{tabular}{c | c | c}
        Joint & Restriction & Definition of $q_{i}$ \\
        \hline
        R & $x_{i} = y_{i} = z_{i} = 0$ & $q_{i} = \{\theta_{i}\}$ \\
        \hline
        P & $\theta_{i} = x_{i} = y_{i} = 0$ & $q_{i} = \{z_{i}\}$ \\
        \hline
        C & $x_{i} = y_{i} = 0$ & $q_{i} = \{\theta_{i}, z_{i}\}$ \\
        \hline
        E & $z_{i} = 0$ & $q_{i} = \{\theta_{i}, x_{i}, y_{i}\}$ \\
        \hline
        S & see equation \eqref{E: sphere joint} & $q_{i} = \{\phi_{i,x}, \phi_{i,y}, \phi_{i,z}\}$
    \end{tabular}
    \caption{parameterization of algebraic joints in terms of the matrix given in \eqref{E: gen low order pair complete}.}
    \label{Tab: gen low order pair complete}
\end{table}

We remark that the joints C, E, and S can be constructed by the composition of R and P joints.
\begin{itemize}
    \item A C joint is a composition of an R joint and a P joint:
    \begin{align}
        \begin{bmatrix}
        \cos(\theta_{i}) & -\sin(\theta_{i}) & 0 & 0 \\
        \sin(\theta_{i}) &  \cos(\theta_{i}) & 0 & 0 \\
        0  & 0  & 1 & z_{i} \\
        0 & 0 & 0 & 1
    \end{bmatrix}
    =
    \begin{bmatrix}
        \cos(\theta_{i}) & -\sin(\theta_{i}) & 0 & 0 \\
        \sin(\theta_{i}) &  \cos(\theta_{i}) & 0 & 0 \\
        0  & 0  & 1 & 0 \\
        0 & 0 & 0 & 1
    \end{bmatrix}
    \begin{bmatrix}
        1 & 0 & 0 & 0 \\
        0 &  1 & 0 & 0 \\
        0  & 0  & 1 & z_{i} \\
        0 & 0 & 0 & 1
    \end{bmatrix}
    \end{align}

    \item An E joint is the composition of one R joint and two P joints
    \begin{align}
        \begin{bmatrix}
        \cos(\theta_{i}) & -\sin(\theta_{i}) & 0 & x_{i} \\
        \sin(\theta_{i}) &  \cos(\theta_{i}) & 0 & y_{i} \\
        0  & 0  & 1 & 0 \\
        0 & 0 & 0 & 1
    \end{bmatrix}
    =
    \begin{bmatrix}
        1 & 0 & 0 & x_{i} \\
        0 &  1 & 0 & 0 \\
        0  & 0  & 1 & 0 \\
        0 & 0 & 0 & 1
    \end{bmatrix}
    \begin{bmatrix}
        1 & 0 & 0 & 0 \\
        0 &  1 & 0 & y_{i} \\
        0  & 0  & 1 & 0 \\
        0 & 0 & 0 & 1
    \end{bmatrix}
    \begin{bmatrix}
        \cos(\theta_{i}) & -\sin(\theta_{i}) & 0 & 0 \\
        \sin(\theta_{i}) &  \cos(\theta_{i}) & 0 & 0 \\
        0  & 0  & 1 & 0 \\
        0 & 0 & 0 & 1
    \end{bmatrix}
    \end{align}

    \item An S joint is the composition of three R joints expressed as Euler angles.
    \begin{align} \label{E: sphere joint}
    % \footnotesize{
    U(\psi_{i})
    =
    \begin{bmatrix}
        \cos(\psi_{i,x}) & -\sin(\psi_{i,x}) & 0 \\
        \sin(\psi_{i,x}) &  \cos(\psi_{i,x}) & 0 \\
        0 & 0 & 1
    \end{bmatrix}
    \begin{bmatrix}
        \cos(\psi_{i,y}) & 0 & -\sin(\psi_{i,y}) \\
         0 &  1 & 0 \\
        \sin(\psi_{i,y})  & 0  & \cos(\psi_{i,y}) \\
    \end{bmatrix}
    \begin{bmatrix}
        1 & 0 & 0  \\
        0 & \cos(\psi_{i,z}) & -\sin(\psi_{i,z}) \\
        0 & \sin(\psi_{i,z}) &  \cos(\psi_{i,z})   \\
    \end{bmatrix}
    % }
    \end{align}

Our approach presented for a robot composed of R and P joints can be extended to handle any algebraic mechanism by consider the other algebraic joints as compositions of R and P joints.
    
\end{itemize}

\section{Definition of Archimedean}\label{A: Archimedean}
In this section we formally define the Archimedean property that appears in Theorem \ref{T: Putinar} and Theorem \ref{T: Putinar Dual}. 
\begin{definition} \label{D: Archimedean}
A semialgebraic set $\calS_{g} = \{x \mid g_{i}(x) \geq 0, i \in [n]\}$ is Archimedean if there exists $N \in \setN$ and $\lambda_{i}(x) \in \bSigma$ such that:
\begin{align*}
    N - \sum_{i=1}^{n}x_{i}^2 = \lambda_{0}(x) + \sum_{i=1}^{n} \lambda_{i}(x)g_i(x)
\end{align*}
\end{definition}

\section{Semialgebraic Descriptions of Set Membership for Common Convex Bodies} \label{A: Semialgebraic Set Memebership}

\begin{table}[H]
\setlength\extrarowheight{-10pt}
\small{
\centering
    \begin{tabular}{p{1.6in} | p{0.65in} | p{3in}}
    \centering{Body} & 
    % \Centering{$x \in \calA$} & 
    \Centering{Variables} &
    \Centering{Description of $\calA(s)$ as a semi-algebraic set}
    % \\
    % \hline
    % \RaggedRight{H-rep Polytope defined as} $$\{x \mid G(s)x \leq r\}$$
    % &
    % % $$G(s)x \leq r(s)$$ &
    % $$\{x,s\}$$ &
    % \begin{gather*}
    % \sum_{i=1}^{m}\lambda_{i}(x,s) (r_{i} - g_{i}^{T}(s)x),
    % \\
    % \lambda_{i}(x,s) \in \bSigma
    % \end{gather*}
    \\
    \hline
    \begin{tabular}{p{1.6in}}
    \RaggedRight{V-rep Polytope with $m$ vertices $v_{i}$ at position 
    $
    \leftidx^{F}p^{v_{i}}(s) =
    \frac{\leftidx^{F}f^{v_{i}}(s)}
    {\leftidx^{F}g^{v_{i}}(s)}
    $
    }
    \end{tabular}
    % & 
    % \begin{gather*}
    % x = \sum_{i=1}^{m} \mu_{i}v_{i}(s), \\
    % \sum_{i=1} \mu_{i} = 1
    % \end{gather*}
    &
    \begin{tabular}{p{0.65in}}
    $\{s,x, \mu\}$
    \end{tabular}
    &
    \begin{tabular}{p{3.0in}}
    \begin{gather*}
        h_{1}(s, x, \mu) = \left(\prod_{i}\leftidx{^F}g^{v_{i}}\right) \left(x-
         \sum_{i=1}^{m} \mu_{i}\left(\frac{\leftidx^{F}f^{v_{i}}(s)}{\leftidx^{F}g^{v_{i}}(s)} \right)\right) 
        \\
         h_{2}(\mu) = 1-\sum_{i=1} \mu_{i}    
        \\
        \gamma_{i}(\mu_{i}) = \mu_{i}, ~ i \in [m]
    \end{gather*}
    \end{tabular}
    \\
    \hline
    \begin{tabular}{p{1.6in}}
    \RaggedRight{Sphere with center $o$ at position $\leftidx^{F}p^{o}(s) = \frac{\leftidx^{F}f^{o}(s)}{\leftidx^{F}g^{o}(s)}$ and radius $r$}
    \end{tabular}
    % & 
    % $$\norm{x - c(s)}^{2} \leq r^{2}$$ 
    &
    \begin{tabular}{p{0.65in}}
    $\{s,x\}$
    \end{tabular}
    &
    \begin{tabular}{p{3.0in}}
    \begin{multline*}
    \gamma_{1}(s,x) = 
    \left(\leftidx^{F}g^{o}(s)\right)^{2}
    \left(r^{2}- \norm{x - \frac{\leftidx^{F}f^{o}(s)}{\leftidx^{F}g^{o}(s)}}^{2}\right)
    \end{multline*}
    \end{tabular}
    \\
    \hline
    \begin{tabular}{p{1.6in}}
    \RaggedRight{Capsule, the convex hull of two spheres with centers 
$c_{1}$ and $c_{2}$ at positions $\leftidx^{F}p^{o_{i}}(s) = \frac{\leftidx^{F}f^{o_{i}}(s)}{\leftidx^{F}g^{o_{i}}(s)}$ and radii $r_{1}$, $r_{2}$
        }
    \end{tabular}
    &
    \begin{tabular}{p{0.6in}}
    $\{s,x,\mu\}$
    \end{tabular}
    &
    \begin{tabular}{p{3.0in}}
    \begin{gather*}
        \frac{\leftidx^{F}f^{o_{\mu}}}{\leftidx^{F}g^{o_{\mu}}} = \mu \frac{\leftidx^{F}f^{o_{1}}(s)}{\leftidx^{F}g^{o_{1}}(s)} + (1-\mu)\frac{\leftidx^{F}f^{o_{2}}(s)}{\leftidx^{F}g^{o_{2}}(s)}
        \\
        r_{\mu} = \mu r_{1} + (1-\mu) r_{2}
        \\
    \gamma_{1}(s, x,\mu) =\left(\leftidx^{F}g^{o_{\mu}}(s)\right)^{2}
    \left(
    r_{\mu}^{2}
    -  \norm{x - \frac{\leftidx^{F}f^{o_{\mu}}}{\leftidx^{F}g^{o_{\mu}}}}^{2}
    \right)
    \\
    \gamma_{2}(\mu) = \mu
    \\
        \gamma_{3}(\mu) = 1-\mu
    \end{gather*}
    \end{tabular}
    \\
    \hline
    \begin{tabular}{p{1.6in}}
    \RaggedRight{
    Cylinder, the convex hull of two circles with centers $o_{1}$ and $o_{2}$, at position $\leftidx^{F}p^{o_{i}}(s) = \frac{\leftidx^{F}f^{o_{i}}(s)}{\leftidx^{F}g^{o_{i}}(s)}$, lying in the plane normal to $\leftidx^{F}p^{o_{1}}(s) - \leftidx^{F}p^{o_{2}}(s)$, and with radii $r_{1}$ and $r_{2}$.
    }
    \end{tabular}
    &
    $\{s,x,v, \mu\}$
    &
    \begin{tabular}{p{3.0in}}
    \begin{gather*}
    \frac{\leftidx^{F}f^{o_{\mu}}(s)}{\leftidx^{F}g^{o_{\mu}}(s)} = \mu \frac{\leftidx^{F}f^{o_{1}}(s)}{\leftidx^{F}g^{o_{1}}(s)} + (1-\mu)\frac{\leftidx^{F}f^{o_{2}}(s)}{\leftidx^{F}g^{o_{2}}(s)}
        \\
    r_{\mu} = \mu r_{1} + (1-\mu) r_{2}
    \\
    h_{1}(v, s) = v^{T}\left(\frac{\leftidx^{F}f^{o_{1}}(s)}{\leftidx^{F}g^{o_{1}}(s)} -\frac{\leftidx^{F}f^{o_{2}}(s)}{\leftidx^{F}g^{o_{2}}(s)}\right)
    \\
    h_{2}(s, x, \mu, v) = x - \frac{\leftidx^{F}f^{o_{\mu}}(s)}{\leftidx^{F}g^{o_{\mu}}(s)} - v
    \\
    \gamma_{1}(v, \mu) = r_{\mu}^{2} - v^{T}v 
    \\
    \gamma_{2}(\mu) = \mu
    \\
    \gamma_{3}(\mu) = 1-\mu
    \end{gather*}
    \end{tabular}
    % \begin{gather*}
    % \end{gather*}
    \end{tabular}
    \caption{Parameterizations of the condition that $x$ lies in a convex body that moves rigidly as a function of $s$.} \label{Tab: shape conditions table poly}
    }
\end{table}

\section{Parametrized Hyperplane Separation Condition for the Cylinder}\label{A: cylinder matrix sos}
To derive the hyperplane separation condition for cylinder, we first attach a geometric frame $G$ to the cylinder, as shown in Fig.\ref{fig:cylinder}. The cylinder's geometric frame $G$'s origin coincides with the cylinder's center, with the $z$ axis of the $G$ frame along the cylinder axis. The height of the cylinder is $2h$, with the top/bottom circle radius being $r_1$ and $r_2$ respectively.
\begin{figure}
\centering
\includegraphics[width=0.5\textwidth]{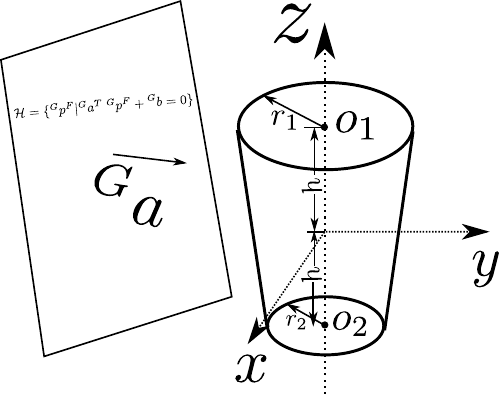}
\caption{Illustration of the cylinder on one side of the plane $\mathcal{H}$, with the plane normal being $\leftidx{^G}a$, expressed in the cylinders geometry frame $G$.}
\label{fig:cylinder}
\end{figure}

We first write the plane $\mathcal{H}$ with its parameters $\leftidx{^G}a(s), \leftidx{^G}b(s)$ in the cylinder's geometric frame $G$ and derive the conditions on $\leftidx{^G}a(s), \leftidx{^G}b(s)$. The cylinder is in the positive side of the plane if and only if both its top and bottom rim are on the positive side of the plane, namely
\begin{subequations}
\begin{align}
\left(\leftidx{^G}a(s)\right)^T \begin{bmatrix} r_1\cos\alpha\\
r_1\sin\alpha\\
h
\end{bmatrix} + \leftidx{^G}b(s) \ge 0\; \forall \alpha\\
\left(\leftidx{^G}a(s)\right)^T \begin{bmatrix} r_2\cos\alpha\\
r_2\sin\alpha\\
-h
\end{bmatrix} + \leftidx{^G}b(s) \ge 0\; \forall \alpha.
\end{align}
\end{subequations}
Taking the infimum of both sides with respect to $\alpha$ makes the above conditions equivalent to
\begin{subequations}
\begin{align}
\leftidx{^G}a_z(s) h + \leftidx{^G}b(s) \ge r_1 \norm{\begin{bmatrix} \leftidx{^G}a_x(s) & \leftidx{^G}a_y(s)\end{bmatrix}}\label{eq:cylinder_separation1}\\
-\leftidx{^G}a_z(s) h + \leftidx{^G}b(s) \ge r_2 \norm{\begin{bmatrix} \leftidx{^G}a_x(s) & \leftidx{^G}a_y(s)\end{bmatrix}}\label{eq:cylinder_separation2}
\end{align}
\label{eq:cylinder_separation}
\end{subequations}

Next, we use the Schur complement, to reformulate \eqref{eq:cylinder_separation1} and \eqref{eq:cylinder_separation2} the positive semidefinite matrix conditions. For example, \eqref{eq:cylinder_separation1} is equivalent to
\begin{subequations}
\begin{gather}
\begin{bmatrix}
\leftidx{^G}a_z(s)h+\leftidx{^G}b(s) & 0 & r_1\ \leftidx{^G}a_x(s) \\
0 & \leftidx{^G}{a}_z(s)h + \leftidx{^G}b(s) & r_1\ \leftidx{^G}a_y(s)\\
r_1\ \leftidx{^G}a_x(s) & r_1\ \leftidx{^G}a_y(s) &\leftidx{^G}{a}_z(s)h + \leftidx{^G}b(s)
\end{bmatrix}\succeq 0
\end{gather}

As explained in Section \ref{S: Hyperplane Cert}, this polynomial PSD condition can be reformulated as the condition
\begin{gather}
u^T\begin{bmatrix}
\leftidx{^G}a_z(s)h+\leftidx{^G}b(s) & 0 & r_1\ \leftidx{^G}a_x(s) \\
0 & \leftidx{^G}{a}_z(s)h + \leftidx{^G}b(s) & r_1\ \leftidx{^G}a_y(s)\\
r_1\ \leftidx{^G}a_x(s) & r_1\ \leftidx{^G}a_y(s) &\leftidx{^G}{a}_z(s)h + \leftidx{^G}b(s)
\end{bmatrix}u\geq 0 \;\forall u. \label{E: cylinder_matrix_sos}
\end{gather}
\end{subequations}

To avoid the trivial solution $\leftidx{^G}a(s) = 0, \leftidx{^G}b(s) = 0$ (which is not in fact a separating plane), we add the extra constraint $\leftidx{^G}a^T \left(\frac{\leftidx{^G}p^{o_1} + \leftidx{^G}p^{o_2}}{2}\right) + \leftidx{^G}b \geq 1$. Here $\frac{\leftidx{^G}p^{o_1} + \leftidx{^G}p^{o_2}}{2}$ is the position of the cylinder center expressed in the geometric frame $G$ which coincides with the frame's origin. Therefore $\frac{\leftidx{^G}p^{o_1} + \leftidx{^G}p^{o_2}}{2} = 0$, and so it is sufficient to introduce the constraint
\begin{align}
\leftidx{^G}b(s) \ge 1 \label{E: cylinder_b>1}
\end{align}
to exclude the trivial solution $\leftidx{^G}a=0, \leftidx{^G}b=0$.

In our optimization program, we express the separating plane in a frame $F$ (where the choice of frame $F$ is discussed in \ref{A: Frame Selection}), not in cylinder's geometric frame $G$. Hence we need to compute $\leftidx{^G}a(s), \leftidx{^G}b(s)$ from their corresponding terms $\leftidx{^F}a(s), \leftidx{^F}b(s)$ expressed in frame $F$ and the relative transform $\leftidx{^F}X^G$ between the two frames
\begin{subequations}
\begin{gather}
\leftidx{^G}a(s) = \leftidx{^G}R^F(s) \;\leftidx{^F}a(s)\\
\leftidx{^G}b(s) = \leftidx{^F}b(s) + \left(\leftidx{^F}a(s)\right)^T \;\leftidx{^F}p^G(s)
\end{gather}
\label{E: plane frame transform}
\end{subequations}

As described in Section \ref{S: Rat Forward}, both the position $\leftidx{^F}p^G(s)$ and orientation $\leftidx{^G}R^F(s)$ are rational functions of $s$. By replacing $\leftidx{^G}a(s), \leftidx^{G}b(s)$ in \eqref{E: cylinder_matrix_sos} and \eqref{E: cylinder_b>1} with \eqref{E: plane frame transform} and requiring the resulting numerator of the rational function to be non-negative, we derive that the plane separating cylinders can be enforced via a polynomial non-negativity condition which can be formulated as sums-of-squares condition.

\section{The Certification Programs are Necessary and Sufficient} \label{A: necessary and sufficient}
 In this section, we expand our discussion on the power of the certification programs presented in Sections \ref{S: Hyperplane Cert} and \ref{S: infeasible non-collision cert}. As remarked previously, Theorems \ref{T: dual psatz poly cert is always feasible} and \ref{T: hyperplane poly cert is always feasible} are necessary as programs \eqref{E: dual psatz poly cert} and \eqref{E: cert by hyperplane poly} are infinite dimensional. It is not immediately obvious that for every robot and every scene, there exists a finite degree where in each program must become feasible when $\calP$ truly contains no collisions.
 
A second subtlety applies specifically to \eqref{E: cert by hyperplane poly}. When generalizing \eqref{E: sep hyperplane generic}, we argued that it was beneficial to search for a parametric hyperplane as a function of our TC-space variable $s$ and asserted that a polynomial parameterization was a good choice. However, it is not obvious that a polynomial parameterization is sufficient, and perhaps we require a rational or even more complicated parameterization of the plane. 

These questions about the power of SOS programming arise in other domains. For example, SOS is commonly used to search for polynomial Lyapunov functions to prove the stability of polynomial dynamical systems \cite[]{majumdar2017funnel}. However, it is known that not every stable polynomial dynamical system admits a polynomial Lyapunov function \cite[]{ahmadi2011globally}, and therefore SOS programming is a sufficient, but not necessary tool for proving the stability of dynamical systems.

Fortunately, our certification programs from \ref{S: Hyperplane Cert} and \ref{S: infeasible non-collision cert} are indeed \emph{necessary and sufficient}, in the sense that there will always exist a finite degree such that the programs become feasible if $\calP$ contains no collision. The proof of this for the program \eqref{E: dual psatz poly cert} follows immediately from Theorem \ref{T: Putinar Dual}.

\begin{proof} 
\textbf{(of Theorem \ref{T: dual psatz poly cert is always feasible})}
    Our assumptions on $\calP, ~\calA$, and $\calB$ imply that $\calS_{\calP, \calA, \calB}$ is an Archimedean set. Therefore, the feasibility of \eqref{E: dual psatz poly cert} for sufficiently high degree $\rho$ follows immediately from ``effective" versions of Theorem \ref{T: Putinar Dual} such as those given in \cite[]{nie_complexity_2007, baldi2021moment} which give explicit degree bounds.  \qed
\end{proof}

Though the proof of Theorem \ref{T: hyperplane poly cert is always feasible} is more technically involved, the key idea is simple. In short, we construct a family of continuous functions which map each TC-space configuration $s \in \calP$ to a separating plane. We then argue that this family of continuous functions must contain hyperplanes which are parametrized as polynomials. Finally, we again appeal to ``effective" versions of Theorem \ref{T: Putinar} such as those given in \cite[]{nie_complexity_2007, baldi2021moment} to show that these polynomials can be found using SOS programming.

We proceed in steps, first establishing that the set of separating planes at a point $s$ in TC-free is open.

\begin{proposition} \label{P: sep set non empty and open}
Let $\Phi(s)$ denote the set of strictly separating hyperplanes at the point $s$ for bodies $\calA(s)$ and $\calB(s)$ and let $\calP$ be a non-empty, polytopic subset of TC-free. Then $s \in \calP$ implies that $\Phi(s)$ is a non-empty, open set.
\qed
\end{proposition}
\begin{proof}
     By definition, a hyperplane $\begin{bmatrix} a \\ b \end{bmatrix} \in \setR^{4}$ strictly separates $\calA(s)$ and $\calB(s)$ if and only if there exist
    positive constants $\varepsilon_{\calA}$ and $\varepsilon_{\calB}$ such that $a^{T}x + b \geq \varepsilon_{\calA}~ \forall x \in \calA(s) $and $a^{T}x + b \leq -\varepsilon_{\calB}~ \forall x \in \calB(s)$. Since the bodies $\calA(s)$ and $\calB(s)$ are strictly separating for every point $s \in \calP$, the Separating Hyperplane theorem guarantees the existence of such a vector and so $\Phi(s)$ is non-empty.

    Now consider $\begin{bmatrix} a \\ b \end{bmatrix} \in \Phi(s) \subseteq \setR^{4}$ and its $\delta$ neighborhood 
    \begin{align*}
        \calN(\delta) = 
        \left\{\begin{bmatrix} a \\ b \end{bmatrix} + \delta\begin{bmatrix} v_{a} \\ v_{b} \end{bmatrix} 
        \middle| \norm{\begin{bmatrix} v_{a} \\ v_{b} \end{bmatrix} } \leq 1\right\},
    \end{align*}
    with $\delta > 0$.
    
    We have that for all $x \in \calA$
    \begin{align*}
    (a^{T} + \delta v^{T}_{a})x + (b + \delta v_{b}) 
    \geq
    \varepsilon_{\calA} + \delta \min_{\norm{v} \leq 1,~ x \in \calA}  v^{T}_{a}x + v_{b}
    \end{align*}
    and similarly for all $x \in \calB$
    \begin{align*}
(a^{T} + \delta v^{T}_{a})x + (b + \delta v_{b}) 
    \leq
    -\varepsilon_{\calB} + \delta \max_{\norm{v} \leq 1,~ x \in \calB}  v^{T}_{a}x + v_{b}.
    \end{align*}

Letting $M_{l} = \displaystyle{\min_{\norm{v} = 1,~ x \in \calA}}  v^{T}_{a}x + v_{b}$ and $M_{u} = \displaystyle{\max_{\norm{v} = 1,~ x \in \calB}}  v^{T}_{a}x + v_{b}$, we have that all planes $\calN(\delta)$ are separating if
\begin{align*}
    0 < \delta < 
    \min \left\{
    \frac{\varepsilon_{\calA}}{\abs{M_{l}}},~
    \frac{\varepsilon_{\calB}}{\abs{M_{u}}}
    \right\}
\end{align*}
and so $\Phi(s)$ is open.
\qed
\end{proof}

\begin{proposition} \label{P: minimal neighborhood exists}
Define
\begin{align*}
        \calN(s, \delta) =
        \bigcap_{\norm{v} \leq 1} \Phi(s + \delta v)
    \end{align*}
    
    For all $s \in \calP$ there exists $\delta_{\min}(s) > 0$, not necessarily finite such that, $\calN(s, \delta)$ is non-empty and open for every $0 < \delta < \delta_{\min}$.
    
    % $\exists ~\delta_{\min} > 0$ such that $\calN(s, \delta_{\max}) \neq \emptyset$. Moreover, $\calN(s, \delta)$ is non-empty for all $\delta \in (0, \delta_{\min})$
\end{proposition}
\begin{proof}

Recall that the position of every point in $\calA(s)$ is a continuous function of $s$ and that the distance from a point to a set is a continuous function \cite[]{rudin1976principles}. Therefore, the distance of every point in $\calA(s)$  to every element of $\Phi(s)$ changes continuously.
% This implies that for every $\delta > 0$ and $\begin{bmatrix} a \\ b \end{bmatrix} \in \Phi(s)$ there exists $M_{\delta}(s)$ such that:
 For every $\delta > 0$ and $\begin{bmatrix} a \\ b \end{bmatrix} \in \Phi(s)$ we define:
\begin{align*}
M_{\delta}(s)
\coloneqq
\sup_{\norm{v} \leq 1}
\abs{
\inf_{x \in \calA(s)}
a^{T}x + b 
-
 \inf_{x \in \calA(s + \delta v)}
a^{T}x + b 
}
\end{align*}
% where both infimums are attained as they are taken over compact sets.

We have that
\begin{align*}
   \inf_{\norm{v} \leq 1} \inf_{\substack{x \in \calA(s + \delta v)}}
    a^{T}x + b 
    \geq
    \inf_{x \in \calA(s)}
    a^{T}x + b 
    - M_{\delta}(s)
    \geq 
    \varepsilon_{\calA} - M_{\delta}(s)
\end{align*}

 Moreover, if $\delta_{2} < \delta_{1}$, then $M_{\delta_{2}}(s) \leq M_{\delta_{1}}(s)$. By continuity and monotonicity, $M_{\delta}(s) \rightarrow 0$ as $\delta \rightarrow 0$ and so there exists $\delta$ sufficiently small such that $\varepsilon_{\calA} - M_{\delta}(s) > 0$. A similar argument shows that $\delta$ can be chosen sufficiently small such that the plane $\begin{bmatrix} a \\ b \end{bmatrix}$ continues to satisfy the separating plane conditions for $\calB$. Therefore  $\begin{bmatrix} a \\ b \end{bmatrix}  \in \Phi(s + \delta v)$ for all $v$ such that $\norm{v} \leq 1$ if $\delta$ is chosen sufficiently small. It is clear that choosing $\delta$ smaller continues to ensure that $\calN(s, \delta)$ is non-empty. Openness is immediate following a similar argument to Proposition \ref{P: sep set non empty and open}.
 \qed
\end{proof}

The above proposition enables us to establish that there exists an open family of continuous functions $f(s)$ such that their outputs are always separating hyperplanes.

\begin{proposition} \label{P: calF non-empty}
Let $\calF$ be the set of continuous functions mapping $$f: s \mapsto \begin{bmatrix} a \\ b \end{bmatrix}$$ such that $f(s) \in \Phi(s)$ for all $s \in \calP$. The set $\calF$ is non-empty and open under the pointwise metric $$d(f,g) = \sup_{s \in \calP} \norm{f(s) - g(s)}.$$
\end{proposition}
\begin{proof}

Suppose $\calF$ were empty. Then every function satisfying $f(s) \in \Phi(s) ~\forall s \in \calP$ is not a continuous function. Namely, for every $f$ there exists a point $s_{0}$ such that for all $\delta > 0$, $f(s_{0}) \in \Phi(s_{0})$ but $f(s_{0}) \notin \calN(s_{0}, \delta)$. This contradicts the openness of $\calN(s_{0}, \delta)$ for a sufficiently small $\delta$ from Proposition \ref{P: minimal neighborhood exists} and so $\calF$ is non-empty. Openness follows from the fact that if $\delta > 0$ is chosen sufficiently small, then for every continuous $g$ satisfying $d(f,g) < \delta$, then $g$ must also separate $\calA(s)$ and $\calB(s)$ for every $s \in \calP$.

\qed
\end{proof}

We are now ready to prove Theorem \ref{T: hyperplane poly cert is always feasible}

\begin{proof}
\textbf{(of Theorem \ref{T: hyperplane poly cert is always feasible})}
    By Proposition \ref{P: calF non-empty}, $\calF$ is a non-empty open subset of continuous functions defined on the compact domain $\calP$. The Stone-Weierstrass theorem \cite[]{rudin1976principles} states that the set of polynomial functions on a compact domain is dense in the set of continuous functions in that domain under the pointwise metric. Therefore, $\calF$ must contain a map $p: s \mapsto \begin{bmatrix} a(s) \\ b(s) \end{bmatrix}$ such that each component is a polynomial. This polynomial is of finite degree and is a strictly separating hyperplane and therefore by ``effective" versions of Theorem \ref{T: Putinar} such as \cite[]{nie_complexity_2007, baldi2021moment}, there exists a Putinar certificates of finite degree certifying that $p(s)$ is a separating hyperplane.
    \qed
\end{proof}

\section{Practical aspects}\label{A: Practical Aspects}
In this section, we discuss some practical aspects for essential for enabling Algorithm \ref{Alg: Bilinear Alternation} to realistic examples. These include the choice of reference frame in which to express the forward kinematics, the selection of a finite basis for the polynomials in our SOS programs, and which aspects of \ref{Alg: Bilinear Alternation} can be parallelized.

\subsection{Choosing the Reference Frame}  \label{A: Frame Selection}
The polynomial implications upon which the certification program \eqref{E: cert by hyperplane poly} and polytope growth program \eqref{E: polytope growth program} are based require choosing a coordinate frame between each collision pair $\calA$ and $\calB$. However, as the collision-free certificate between two different collision pairs can be computed independently of each other, we are free to choose a different coordinate frame to express the kinematics for each collision pair. This is important in light of \eqref{E: gen forward kin} and \eqref{E: rational forward kinematics gen} that indicate that the degree of the polynomials $\leftidx{^F}f^{\calA_{j}}$  and $\leftidx{^F}g^{\calA_{j}}$ are equal to two times the number of joints lying on the kinematic chain between frame $F$ and the frame for $\calA$. For example, the tangent-configuration-space polynomial in the variable $s$ describing the position of the end-effector of a 7-DOF robot is of total degree $14$ when written in the coordinate frame of the robot base. However, when written in the frame of the third link, the polynomial describing the position of the end effector is only of total degree $(7-3)\times 2=8$. This observation is also used in  \cite[]{trutman2020globally} to reduce the size of the optimization program. 

%Fortunately, the polynomials implications of \eqref{E: cert with ellipse} and \eqref{E: cert with polytope} are decoupled in the sense that the choice of reference frame $F$ for the forward kinematics can be chosen independently for each collision pair $(\calA, \calB)$. While choosing a reference frame $F$ near $\calA$ reduces the degree of the polynomial needed to express the forward kinematics of $\calA$, it comes at the cost of increasing the degree of the polynomial needed to express the kinematics of $\calB$. As we shall see in Section \ref{S: Basis Selection}, the size of the semidefinite variables in \eqref{E: cert with ellipse} and \eqref{E: cert with polytope} scale as the square of the degree of the polynomial used to express the forward kinematics. The complexity of solving an SDP is also quadratic in the size of the semidefinite variables and therefore quartic in the the polynomial degree \cite[]{jiang2020improved}. 

The size of the semidefinite variables in \eqref{E: cert by hyperplane poly} and \eqref{E: polytope growth program} scale as the square of the degree of the polynomial used to express the forward kinematics. Supposing there are $n$ links in the kinematics chain between $\calA$ and $\calB$, then choosing the $j$th link along the kinematics chain as the reference frame $F$ leads to scaling of order $j^{2} + (n-j)^{2}$. Choosing the reference frame in the middle of the chain minimizes this complexity to scaling of order $\frac{n^2}{2}$ and we therefore adopt this convention in our experiments.

\subsection{Basis Selection} \label{A: Basis Selection}
The condition that a polynomial can be written as a sum of squares can be equivalently formulated as an equality constraint between the coefficients of the polynomial and an associated semidefinite variable known as the Gram matrix \cite[]{parrilo2004sum}. Namely, a polynomial $p(s)$ is sums-of-squares if and only if $p(s) = z(s)^T X z(s), X\succeq 0 $ where $z(s)$ is a vector of monomials and $X$ is the Gram matrix. The number of rows in the positive semidefinite Gram matrix equals to the size of the vector $z(s)$. In general, a sums-of-squares polynomial in $k$ variables of total degree $2d$ requires a Gram matrix of size ${k +d} \choose {d}$ to represent which can quickly become prohibitively large. Fortunately, the polynomials in our programs contain substantially more structure which will allow us to select a small-sized vector of monomials $z(s)$, and hence drastically reduce the size of the Gram matrices and speed up the optimization problem.

\subsubsection{Polytopic collision geometry}
We begin with the separating plane condition for polytopic collision geometries. Note that from \eqref{E: rational forward kinematics gen} that while both the numerator and denominator of the forward kinematics are of total degree $2n$, with $n$ the number of links of the kinematics chain between frame $A$ and $F$, both polynomials are of \emph{coordinate} degree of at most two (i.e. the highest degree of $s_{i}$ in any term is $s_{i}^2$). We will refer to this basis as $\nu(s)$ which is a vector containing terms of the form $\prod_{i = 1}^{n} s_{i}^{\text{degree}(s_i)}$ with $\text{degree}(s_i) \in \{0,1,2\}$ for all $3^n$ possible permutations of the exponents $\text{degree}(s_i)$.

We recall that we parametrize our hyperplane using polynomial entries. If $a_{\calA, \calB}(s)= a^T_{\calA, \calB}\eta(s)$, $b_{\calA, \calB}(s) = b^T_{\calA, \calB}\eta(s)$ for some basis $\eta$ in the variable $s$. The position of $x(s) \in \calA(s)$ is expressed in basis $\nu(s)$, then the left hand side of \eqref{E: polytope separation psatz condition} can be expressed as a linear function of the basis $\gamma(s)$, where $\gamma(s)$ contains all the possible entries that appear in the outer product $\eta(s) \nu(s)^{T}$. 
\begin{example}
    Suppose
    $$\eta(s) = \begin{bmatrix} 1 & s_{1} & s_{2} \end{bmatrix}^{T}$$ 
    and 
    $$\nu(s) =\begin{bmatrix} 1 & s_{1} & s_{1}^{2} & s_{2} & s_{2}^{2} & s_{1}s_{2} & s_{1}^{2}s_{2} &  s_{1}s_{2}^{2} &  s_{1}^{2}s_{2}^{2}\end{bmatrix}^{T}$$.

    Then:
    \begin{multline*}
    \gamma(s) =
    \Big[
        1 \quad 
        s_{1} \quad s_{1}^{2} \quad s_{1}^{3} \quad
        s_{2} \quad s_{2}^{2} \quad s_{2}^{3} \quad
        s_{1}s_{2} \quad s_{1}^{2}s_{2} \quad s_{1}^{3}s_{2} \quad
        s_{1}s_{2}^{2} \quad s_{1}^{2}s_{2}^{2} \quad s_{1}^{3}s_{2}^{2} \quad
        s_{1}s_{2}^{3} \quad s_{1}^{2}s_{2}^{3}
        % 1 & 
        % s_{1} & s_{1}^{2} & s_{1}^{3} &
        % s_{2} & s_{2}^{2} & s_{2}^{3} \\
        % s_{1}s_{2} & s_{1}^{2}s_{2} & s_{1}^{3}s_{2} &
        % s_{1}s_{2}^{2} & s_{1}^{2}s_{2}^{2} & s_{1}^{3}s_{2}^{2} &
        % s_{1}s_{2}^{3} & s_{1}^{2}s_{2}^{3} & s_{1}^{3}s_{2}^{3}
    \Big]
    % \end{bmatrix}
    \end{multline*}
    Namely $\gamma(s)$ contains the monomials whose degree for each $s_i$ is at most 3, and only one of $s_i$ can have degree 3 (hence $s_1^3s_2^3$ is not included in $\gamma(s)$).
\end{example}

% Concretely, if we choose to make $a_{\calA, \calB}(s)$ and $b_{\calA, \calB}(s)$ linear functions of the indeterminates $s$, then $\eta(s) = l(s) = \begin{bmatrix} 1 & s_{1} & \dots & s_{n}\end{bmatrix}$. Then the left hand side of \eqref{E: polytope separation psatz condition} can be expressed as linear functions of the basis
% \begin{align} 
% \gamma(s) = \begin{bmatrix} \nu(s) & s_{1}\nu(s) & \dots & s_{n}\nu(s) \end{bmatrix} \label{E: chosen basis}
% \end{align} 

Similarly, we must select a basis $\rho(s)$ for our multiplier polynomials $\lambda_{ij}^{\calA, \calB}(s)$. The equality in \eqref{E: polytope separation psatz condition} determines the minimum necessary basis $\rho(s)$. If the polynomial $p(s)$ is expressed in basis $\gamma(s)$, then the minimal such basis is related to an object known in computational algebra as the Newton polytope of $\gamma$ denoted $\textbf{New}(\gamma(s))$ \cite[]{sturmfels1994newton}. Denoting the linear basis
\begin{align*}
    l(s) = \begin{bmatrix} 1 & s_{1} & s_{2} & \dots & s_{N} \end{bmatrix},
\end{align*}
then exact condition is that  $$\textbf{New}(\gamma(s)) = \textbf{New}(\eta(s)) + \textbf{New}(\nu(s)) \subseteq \textbf{New}(\rho(s)) + \textbf{New}(l(s))$$ where the sum in this case is the Minkowski sum.

By using affine polynomials for separating plane parameters $a_{\calA, \calB}(s), b_{\calA, \calB}(s)$, we know that $\eta(s)$ is the same as the linear basis $l(s)$, then we obtain the condition that $\textbf{New}(\rho(s)) = \textbf{New}(\nu(s))$ and since $\nu(s)$ is a dense, even degree basis we must take $\rho(s) = \nu(s)$. A sums-of-squares polynomial in the basis of $\nu(s)$ has Gram matrix with $2^{n}$ rows. Choosing $\eta(s)$ as the constant basis would in fact result in the same condition, and therefore searching for separating planes which are linear functions of the tangent-configuration-space variable does not increase the size of the semidefinite variables. As the complexity of \eqref{E: cert by hyperplane poly} and \eqref{E: polytope growth program} are dominated by the size of these semidefinite variables, separating planes which are linear functions changes do not substantially affect the solve time but can dramatically increase the size of the regions which we can certify.

Because of this, we choose to parametrize all of our hyperplanes throughout our experiments as linear functions of the TC-space variables. We stress that in general, the choice of a linearly parametrized hyperplane, and the selection of $\rho(s)$ to be the minimum size to match the degree of the left hand side of \eqref{E: polytope separation psatz condition} may not be sufficient to prove that a region $\calP$ is collision-free, even if $\calP$ truly is collision-free. Indeed due of many complexity-theoretic results, we expect that in general $\eta(s)$ and $\rho(s)$ may need to have exponentially high degree for some robots, scenes, and polytopes $\calP$ \cite[]{stengle1996complexity}. However, in practice we have observed that the choices in this section are sufficient to certify many regions of interest, while keeping the optimization problem size tractable for state-of-art numerical solvers.

\begin{remark}
Attempting to certifying that the end-effector of a 7-DOF robot will not collide with the base using program \eqref{E: cert by hyperplane poly} using linearly parametrized hyperplanes and choosing to express conditions \eqref{E: polytope separation psatz condition} in the world frame with na\"ively chosen bases would result in semidefinite variables of size ${7+7 \choose 7} = 3432$. Choosing to express the same conditions according to the discussion in Section \ref{A: Frame Selection} and choosing the basis $\gamma(s)$ described in this section results in semidefinite matrices of rows at most $2^{\lceil7/2\rceil} = 2^{4} = 16$. The division by 2 comes from choosing the middle link as the expressed frame, hence halving the kinematic chain length.
\end{remark}

\subsubsection{Non-polytopic collision geometry} 
In this section, we use the sphere as a running example for explaining how we choose the monomial bases for certifying separation of the non-polytopic geometries; the monomial bases for capsules and cylinders can be derived in a similar manner.

As mentioned in \eqref{E: shur complement implication}, we need to impose
\begin{align}
\label{E: sphere_matrix_sos}
s \in \calP \implies
    \begin{bmatrix}
        \left((a(s))^T \ \leftidx{^F}f^{o}(s)+b(s) \ \leftidx{^F}g^{o}(s)\right)I_{3} & ra(s) \ \leftidx{^F}g^{o}(s) \\ r(a(s))^T \ \leftidx{^F}g^{o}(s) & (a(s))^T \ \leftidx{^F}f^{o}(s)+b(s) \ \leftidx{^F}g^{o}(s)
    \end{bmatrix}
    \succeq 0.
\end{align}
By the definition of positive semidefinite matrix
\footnote{A matrix $X$ is positive semidefinite if and only if $\forall\bar{u}, \begin{bmatrix}\bar{u}\\1\end{bmatrix}^TX\begin{bmatrix}\bar{u}\\1\end{bmatrix}\ge 0$}
, we know that the $4\times 4$ matrix in the right of $\implies$ in \eqref{E: sphere_matrix_sos} is positive semidefinite if and only if
\begin{align}
\forall \bar{u}\in\mathbb{R}^3,\underbrace{ \begin{bmatrix}\bar{u}\\1\end{bmatrix}^T \begin{bmatrix}
        \left((a(s))^T \ \leftidx{^F}f^{o}(s)+b(s) \ \leftidx{^F}g^{o}(s)\right)I_{3} & ra(s) \ \leftidx{^F}g^{o}(s) \\ r(a(s))^T \ \leftidx{^F}g^{o}(s) & (a(s))^T \ \leftidx{^F}f^{o}(s)+b(s) \ \leftidx{^F}g^{o}(s)
    \end{bmatrix}\begin{bmatrix}\bar{u}\\1\end{bmatrix}}_{\sigma(\bar{u}, s)} \geq 0. \label{E: sphere_matrix_sos_sigma}
\end{align}

We impose the following sufficient condition for \eqref{E: sphere_matrix_sos}, where $\calP = \{ s | c_j^T(s)\le d_j, j=1,\hdots,m\}$
\begin{subequations}
\begin{align}
\sigma(\bar{u}, s) = \lambda_0(\bar{u}, s) + \sum_{j=1}^{m}\lambda_j(\bar{u}, s) (d_j - c_j^Ts)\\
\text{for } j=0, \hdots, m, \lambda_j(\bar{u}, s) \ge 0 \;\forall \bar{u}, s.
\end{align}
\label{E: sphere_psatz}
\end{subequations}

Now we analyze the degree of the polynomial $\sigma(\bar{u}, s)$ defined in \eqref{E: sphere_matrix_sos_sigma}. As mentioned in the previous subsection, each monomial in $\leftidx^{F}f^o(s), \leftidx{^F}g^o(s)$ are of the form $\prod_{i=1}^n s_i^{\text{degree}(s_i)}, \text{degree}(s_i)\in\{0, 1, 2\}$. Combining this with 
 the choice of a separating plane $a(s), b(s)$ being affine functions of $s$, we derive that each monomial in $\sigma(\bar{u}, s)$ is of the form $\bar{u}_j^{\text{degree}(\bar{u}_j)} \prod_{i=1}^n s_i^{\text{degree}(s_i)}, \text{ where 
 } \text{degree}(\bar{u}_j)\in\{0, 1, 2\}$, $ \text{degree}(s_i)\in\{0, 1, 2, 3\}$, and at most one of $\text{degree}(s_i)$ can be 3. As an example, $\bar{u}_1^2 s_1^3s_2s_3^2$ is a valid monomial in $\sigma(\bar{u}, s)$ but $\bar{u}_1\bar{u}_2$ is not (because $\sigma(\bar{u}, s)$ doesn't contain the cross product between $\bar{u}_j, \bar{u}_k, j\neq k$). Similarly, $s_1^3s_2^3$ is not in the basis because at most one of $s_i$ can have degree 3. Given these properties on the monomials in $\sigma(\bar{u}, s)$- specifically there being no cross-product term $\bar{u}_j\bar{u}_k, j\neq k$ in $\sigma(\bar{u}, s)$- we can write the positive polynomials $\lambda_j(\bar{u}, s)$ as the summation of three SOS polynomials
 \begin{subequations}
 \begin{align}
 \lambda_j(\bar{u}, s) = \sum_{k=1}^3\lambda_{j, k}(\bar{u}_k, s)\\ \lambda_{j, k}(\bar{u}_k, s)\in\bSigma.
 \end{align}
 \end{subequations}
 For each monomial in the SOS polynomial $\lambda_{j, k}(\bar{u}_k, s)$, the degree of $\bar{u}_k \text{ and } s_i$ for $i=1,\hdots, n$ is either 0, 1, or 2. Hence the number of rows in the Gram matrix in $\lambda_{j, k}(\bar{u}_k, s)$ is of size $2^{n+1}$. By choosing the reference frame according to the convention from Appendix \ref{A: Frame Selection}, $n$ is no larger than $\ceil{N/2}$ where $N$ is the number of joints in the robot.
 \begin{remark}
     For a 6-DOF UR3erobot whose collision geometries are approximated by cylinders, to certify the collision-avoidance between the robot and objects in the world (or self-collision), the largest positive semidefinite matrix in our optimization problem has rows at most $2^{\lceil 6 / 2\rceil + 1} = 2^4 = 16$, where the division by 2 comes from choosing the middle link as the expressed link, hence halving the kinematic chain length to $\lceil6 / 2 \rceil$.
 \end{remark}

\subsection{Parallelization} \label{A: Parallelization}
While it is attractive from a theoretical standpoint to write \eqref{E: cert by hyperplane poly} as a single, large program it is worth noting that it can in fact be viewed as $K$ individual SOS programs, where $K$ is the number of collision pairs in the environment. Indeed, certifying whether pairs $(\calA_{1}, \calA_{2})$ are collision-free for all $s$ in the polytope $\calP$ can be done completely independently of the certification of another pair $(\calA_{1},  \calA_{3})$ as the constraint are not coupled between any pairs. Similarly, the search for the largest inscribed ellipsoid can be done independently of the search for the separating hyperplanes.

Solving the certification program \eqref{E: cert by hyperplane poly} as $K$ individual SOS programs has several advantages. First, as written \eqref{E: cert by hyperplane poly} has $2(m+1)K\sum_{i} \abs{\calA_{i}}$ semidefinite variables of various sizes, where $m$ is the number of inequalities in $\calP$ and $\abs{\calA_{i}}$ denotes the number of inequalities required to express that body $\calA_{i}$ is on a particular side of the plane (see Table \ref{Tab: shape conditions table poly}). In the example from Section \ref{S: Pinball} this corresponds to $18,720$ semidefinite variables. This can be prohibitively large to store in memory as a single program as the size of these semidefinite variables grow. Solving for the separating plane for each pair of collision bodies independently also enables us to determine which collision bodies cannot be certified as collision-free and allows us to terminate our search as soon as a single pair cannot be certified. Finally, decomposing the problems into subproblems enables us to increase computation speed by leveraging parallel processing.

The program \eqref{E: max inscribed ellipse in polytope} can also be solved completely independently of the certification program \eqref{E: cert by hyperplane poly} and is in general a much smaller SDP than any individual certification program. Therefore, lines $3$ and $4$ of Algorithm \ref{Alg: Bilinear Alternation} can be solved in parallel.

We note that \eqref{E: polytope growth program} cannot be similarly decomposed as on this step the variables $c_{i}^T$ and $d_{i}$ affect all of the constraints. However, this program is substantially smaller as we have fixed $2mK\sum_{i} \abs{\calA_{i}}$ of the semidefinite variables as constants and replaced them with $2m$ linear variables representing the polytope. This program is much more amenable to being solved as a single program. 

\section{Seeding Algorithm \ref{Alg: Bilinear Alternation}} \label{A: Seeding}
    Algorithm \ref{Alg: Bilinear Alternation} must be initialized with a polytope $\calP_{0}$ for which \eqref{E: cert by hyperplane poly} is feasible. In principle, the alternation proposed in Section \ref{S: Bilinear Alternation} can be seeded with an arbitrarily small polytope around a collision-free seed point. This seed polytope is then allowed to grow using Algorithm \ref{Alg: Bilinear Alternation}. However, this may require running several dozens of iterations of Algorithm \ref{Alg: Bilinear Alternation} for each seed point which can become prohibitive as the number of degrees of freedom in our robot or the complexity of the scene grows. It is therefore advantageous to seed with as large a region as can be initially certified.

Here we discuss an extension of the \Iris algorithm in \cite[]{deits2015computing} which uses nonlinear optimization to rapidly generate large regions in TC-space. These regions are not guaranteed to be collision-free and therefore they must still be passed to Algorithm \ref{Alg: Bilinear Alternation} to be certified, but do provide good initial guesses. In this section, we will assume that the reader is familiar with \Iris and will only discuss the modification required to use it to grow TC-space regions. Detailed pseudocode is available in Appendix \ref{S: iris pseudocode}.

\Iris grows regions in a given space by alternating between two subproblems: {\sc{SeparatingHyperplanes}} and {\sc{InscribedEllipsoid}}. The {\sc{InscribedEllipsoid}} is exactly the program described in \cite[Section 8.4.2]{boyd2004convex} and we do not need to modify it. The subproblem {\sc{SeparatingHyperplanes}} finds a set of hyperplanes which separate the ellipse generated by {\sc{InscribedEllipsoid}} from all of the obstacles. This subproblem is solved by calling two subroutines: {\sc{ClosestPointOnObstacle}} and {\sc{TangentPlane}}. The former finds the closest point on a given obstacle to the ellipse, while the latter places a plane at the point found in {\sc{ClosestPointOnObstacle}} that is tangent to the ellipsoid.

The original work of \cite[]{deits2015computing} assumes convex obstacles which enables {\sc{ClosestPointOnObstacle}} to be solved as a convex program and for the output of {\sc{TangentPlane}} to be globally separating plane between the obstacle and the ellipsoid of the previous step. Due to the non-convexity of the TC-space obstacles in our problem formulation, finding the closest point on an obstacle exactly becomes a computationally difficult problem to solve exactly \cite[]{ferrier2000computation}. Additionally, placing a tangent plane at the nearest point will be only a locally separating plane, not a globally separating one.

To address the former difficulty, we formulate {\sc{ClosestPointOnObstacle}} as a nonlinear program. Let the current ellipse be given as $\calE = \{Qs + s_{0}\mid \norm{s}_2 \leq 1 \}$ and suppose we have the constraint that $s \in \calP = \{s \mid Cs \leq d\}$. Let $\calA$ and $\calB$ be two collision pairs and ${}^{\calA}p_{\calA}, {}^{\calB}p_{\calB}$ be some point in bodies $\calA$ and $\calB$ expressed in some frame attached to $\calA$ and $\calB$. Also, let ${}^{W}X^{\calA}(s)$ and ${}^{W}X^{\calB}(s)$ denote the rigid transforms from the reference frames $\calA$ and $\calB$ to the world frame respectively. We remind the reader that this notation is drawn from \cite[]{tedrakeManip}. The closest point on the obstacle subject to being contained in $\calP$ can be found by solving the program
\begin{subequations}
\begin{gather} 
\min_{s, {}^{\calA}p_{\calA}, {}^{\calB}p_{\calB}} (s - s_{0})^TQ^TQ(s-s_{0}) \subjectto \\
{}^{W}X^{\calA}(s){}^{\calA}p_{\calA} = {}^{W}X^{\calB}(s){}^{\calB}p_{\calB} \label{E: same point constraint}\\
Cs \leq d
\end{gather}\label{E: closest point}
\end{subequations}
This program searches for the nearest configuration in the metric of the ellipse such that two points in the collision pair come into contact. We find a locally optimal solution $(s^{\star},  {}^{\calA}p_{\calA}^{\star}, {}^{\calB}p_{\calB}^{\star})$ to the program using a fast, general-purpose nonlinear solver such as {\sc{SNOPT}} \cite[]{gill2005snopt}. The tangent plane to the ellipse $\calE$ at the point $s^{\star}$ is computed by calling {\sc{TangentPlane}}, then appended to the inequalities of $\calP$ to form $\calP'$. This routine is looped until \eqref{E: closest point} is infeasible at which point {\sc{InscribedEllipse}} is called again.

Once a region $\calP  =\{s \mid Cs \leq d\}$ is found by Algorithm \ref{A: SNOPT IRIS}, it will typically contain some minor violations of the non-collision constraint. To find an initial, feasible polytope $\calP_{0}$ to use in Algorithm \ref{Alg: Bilinear Alternation}, we search for a minimal uniform contraction $\delta$ of $\calP$ such that $\calP_{\delta} = \{s \mid Cs \leq d - \delta*1\}$ is collision-free. This can be found by bisecting over the variable $\delta \in [0, \delta_{\max}]$ and solving repeated instances of \eqref{E: cert by hyperplane poly}.

Seeding Algorithm \ref{Alg: Bilinear Alternation} with a $\calP_{0}$ as above can dramatically reduce the number of alternations required to obtain a fairly large region and is frequently faster than seeding Algorithm \ref{Alg: Bilinear Alternation} with an arbitrarily small polytope.

\section{Supplementary Algorithms} \label{S: iris pseudocode}

We present a pseudocode for the algorithm presented in Appendix \ref{A: Seeding}. A mature implementation of this algorithm can be found in \href{https://github.com/RobotLocomotion/drake/blob/2f75971b66ca59dc2c1dee4acd78952474936a79/geometry/optimization/iris.cc}{Drake}\footnote{\url{https://github.com/RobotLocomotion/drake/blob/2f75971b66ca59dc2c1dee4acd78952474936a79/geometry/optimization/iris.cc\#L440}}.

\begin{algorithm}
\caption{
Given an initial tangent-configuration-space point $s_{0}$ and a list of obstacles $\calO$, return a polytopic region $\calP = \{s \mid Cs \leq d\}$ and inscribed ellipsoid $\calE_{\calP} = \{s \mid Qs + s_{0}\}$ which contains a substantial portion of the free TC-space (but is not guaranteed to contain no collisions).
}\label{A: SNOPT IRIS}
\SetAlgoLined
 \LinesNumbered
  \SetKwRepeat{Do}{do}{while}
 $(C, d) \gets $ robot joint limits
 \\
 $\calP_{0} \gets \{s \mid Cs \leq d\}$
 \\
  $\calE_{\calP_{0}} \gets $ {\sc{InscribedEllipsoid}}$(\calP_{0})$
 \\
 $j \gets$ number of rows of $C$
 \\
 \Do{$\left(\textbf{vol}(\calE_{i}) - \textbf{vol}(\calE_{i-1}) \right)/ \textbf{vol}(\calE_{i-1})  \geq$ tolerance}{
 \Do{{\sc{FindClosestCollision$(\calP_{i}, \calE_{\calP_{i}})$}} is feasible}
 {
 $(s^{\star},  {}^{\calA}p_{\calA}^{\star}, {}^{\calB}p_{\calB}^{\star}) \gets$ {\sc{FindClosestCollision}}$(\calP_{i}, \calE_{\calP_{i}})$
 \\
 $(c_{j+1}^T, d_{j+1}) \gets$ {\sc{TangentHyperplane}}$(s^{\star}, \calE_{\calP_{i}})$
 \\
 $C \gets \textbf{vstack}(C, c_{j+1}^T)$
 \\
 $d \gets \textbf{vstack}(d, d_{j+1})$
 \\
 $\calP_{i} \gets \{s \mid Cs \leq d\}$
 \\
 $j \gets j +1$
 }
  $\calE_{\calP_{i}} \gets $ {\sc{InscribedEllipsoid}}$(\calP_{i})$
  \\
 $i \gets i+1$
 }
 \Return $(\calP_{i}, \calE_{\calP_{i}})$
\end{algorithm}

\end{document}